\documentclass{amsart}

\usepackage[
  colorlinks,
  linkcolor=blue,
  citecolor=blue,
  urlcolor=blue
]{hyperref}
\usepackage{cleveref}

\usepackage[utf8]{inputenc}

\usepackage[T1]{fontenc}
\usepackage{amsmath,amssymb,amsthm}

\usepackage{graphicx}
\graphicspath{{./figs/}}

\usepackage{subcaption}
\usepackage{tabularx}
\newcolumntype{Y}{>{\centering\arraybackslash}X}

\usepackage{verbatim}
\usepackage[utf8]{inputenc}
\usepackage{lmodern}
\usepackage{tikz-cd}
\usepackage{hyperref}
\usepackage{graphicx}
\usepackage[english]{babel}
\usepackage{color}
\usepackage[margin=1in]{geometry}
\usepackage{cite}
\usepackage{mathtools}
\usepackage{faktor}
\usepackage{enumitem}

\crefname{figure}{Fig.}{Figs.}
\crefname{section}{Sec.}{Secs.}
\crefname{subsection}{Sec.}{Secs.}
\crefname{equation}{Eqn.}{Eqns.}

\Crefname{figure}{Fig.}{Figs.}
\Crefname{section}{Sec.}{Secs.}
\Crefname{subsection}{Sec.}{Secs.}
\Crefname{equation}{Eqn.}{Eqns.}
\theoremstyle{definition}

\newtheorem{theorem}{Theorem}
\newtheorem{prop}{Proposition}[section]
\newtheorem{lemma}[prop]{Lemma}
\newtheorem{corollary}[prop]{Corollary}
\newtheorem{defn}[prop]{Definition}
\newtheorem{example}[prop]{Example}
\newtheorem{remark}[prop]{Remark}

\newcommand{\define}[1]{\narrowbf{#1}}

\usepackage{bbm}
\newcommand{\1}{\mathbbm{1}}

\newcommand{\R}{\mathbb{R}}

\newcommand{\cX}{\mathcal{X}}
\newcommand{\cY}{\mathcal{Y}}
\newcommand{\cZ}{\mathcal{Z}}

\newcommand {\mm}[1] {\ifmmode{#1}\else{\mbox{\(#1\)}}\fi}

\newcommand{\Rspace}        {\mm{{\mathbb R}}}

\newcommand{\Acal}        {\mm{{\mathcal A}}}
\newcommand{\Bcal}        {\mm{{\mathcal B}}}

\newcommand{\Pcal}        {\mm{{\mathcal P}}}

\newcommand{\Xcal}        {\mm{{\mathcal X}}}
\newcommand{\Ycal}        {\mm{{\mathcal Y}}}
\newcommand{\Zcal}        {\mm{{\mathcal Z}}}

\newcommand{\supp}        {\mm{\mathrm{supp}}}

\usepackage{wrapfig}
\usepackage[font=footnotesize]{caption}

\newcommand{\GW}{\mathsf{GW}} 
\newcommand{\W}{\mathsf{W}} 
\newcommand{\coup}{\mathcal{C}} 
\DeclareMathOperator{\dis}{dis} 

\newcommand{\narrowbf}[1]{\textbf{\scalebox{0.9}[1]{#1}}}

\newcommand{\para}[1]{\vspace{2mm}\noindent{\narrowbf{#1}}}

\title{Metrics for Parametric Families of Networks}

\author{Mario G{\'o}mez}
\address{}
\email{}

\author{Guanqun Ma}
\address{}
\email{}

\author{Tom Needham}
\address{}
\email{}

\author{Bei Wang}
\address{}
\email{}

\begin{document}

\begin{abstract}
We introduce a general framework for analyzing data modeled as parameterized families of networks. Building on a Gromov–Wasserstein variant of optimal transport, we define a family of parameterized Gromov-Wasserstein distances for comparing such parametric data, including time-varying metric spaces induced by collective motion, temporally evolving weighted social networks, and random graph models. We establish foundational properties of these distances, showing that they subsume several existing metrics in the literature, and derive theoretical approximation guarantees. In particular, we develop computationally tractable lower bounds and relate them to graph statistics commonly used in random graph theory. Furthermore, we prove that our distances can be consistently approximated in random graph and random metric space settings via empirical estimates from generative models. Finally, we demonstrate the practical utility of our framework through a series of numerical experiments.
\end{abstract} 

\maketitle

\section{Introduction}
\label{sec:introduction}

\para{Motivating examples.}~Consider the problem of mathematically modeling the collective spatial motion of multiple agents over time, such as a group of animals~\cite{parrish1997animal,huang2008modeling,benkert2008reporting,ulmer2019topological}, a population of cells~\cite{nguyen2024quantifying,bonilla2020tracking}, or a fleet of vehicles~\cite{jeung2008discovery}. It is desirable to adopt a representation that is invariant under ambient isometries when only the \emph{intrinsic} features of the motion are of interest, rather than the \emph{extrinsic} position of the group in ambient space. A natural choice in this setting is to record the pairwise distances between agents over time, which leads to the notion of a \emph{time-varying metric space}: a one-parameter family of metrics defined on a common underlying set. A substantial body of work has developed methods for analyzing data in the form of time-varying metric spaces~\cite{xian2022capturing,kim2020persistent,kim2021spatiotemporal,munch2013applications}.

While time-varying metric spaces offer a concrete and well-studied class of examples, they exemplify a broader and increasingly common scenario in data science and network analysis: the need to analyze \emph{parameterized families} of complex structures. Additional examples include:
\begin{itemize}[leftmargin=*]
\item \narrowbf{Time-varying graphs.} Applications in social network modeling~\cite{skarding2021foundations,greene2010tracking} and neuroscience~\cite{yoo2016topological,billings2021simplicial} often involve weighted graphs whose edge weights evolve over time, giving rise to data in the form of a family of graphs parameterized by real numbers, representing the time parameter.
\item \narrowbf{Heat kernels.} The heat kernel on a Riemannian manifold~$M$ describes the diffusion of heat across the manifold and is widely used in geometry processing~\cite{OvsjanikovMerigotMemoli2010,raviv2010volumetric,memoli2011spectral}. In this setting, the data naturally take the form of a family of maps $M \times M \to \Rspace$ parameterized by the positive real numbers.
\item \narrowbf{Random graph models.} Generative random graph models~\cite{van2014random}, such as the Erd\H{o}s--R\'enyi model~\cite{Erdos:1959:pmd}, can be viewed as drawing graphs from a (typically unknown) distribution over the space of all graphs on a fixed node set. Here, the data consist of a collection of graphs parameterized by this underlying state space according to the (unknown) distribution.
\end{itemize}
In practice, a dataset may consist of an \emph{ensemble} of parameterized objects, such as a collection of time-varying fMRI brain connectivity graphs, in which case a metric is needed to compare elements within the ensemble.

To this end, we introduce the notion of a \emph{parameterized measure network}, a flexible model for representing parameterized families of complex structures that encompasses the examples described above. We then define a novel and highly general family of distances on the space of parameterized measure networks, establish their metric, analytical, and statistical properties, and demonstrate their practical utility through a robust numerical implementation.

\para{Gromov-Wasserstein distance.} 
To describe our contributions in more detail, we now briefly review a key component of our approach, the Gromov-Wasserstein framework from optimal transport theory~\cite{Memoli2007,memoli2011}. Recall that a \emph{metric measure space} (\emph{mm-space})~\cite{Gromov2007} is a triple $\Xcal = (X,\mu_X,d_X)$\footnote{The original definition uses the tuple $(X, d_X, \mu_X)$; for convenience, we adopt the order $(X, \mu_X, d_X)$ in this paper.}, where $X$ is a Polish space (i.e.,~a separable completely metrizable topological space), $d_X: X \times X \to \Rspace_{\geq 0}$ is a metric on $X$, and $\mu_X$ is a Borel probability measure on $X$. In other words, it is a metric space equipped with a probability measure.  

Given another mm-space $\Ycal = (Y, \mu_Y, d_Y)$, the associated order-$p$ \emph{Gromov-Wasserstein (GW) distance} between mm-spaces $\Xcal$ and $\mathcal{Y}$~\cite{Gromov2007,memoli2011}, for $1 \leq p < \infty$, is 
\begin{equation}
\label{eqn:GW_distance_intro}
    \mathsf{GW}_p(\Xcal,\mathcal{Y}) \coloneqq \frac{1}{2} \inf_\pi \left(\int_{(X \times Y)^2} \int_{(X \times Y)^2} |d_X(x,x') - d_Y(y,y')|^p \pi(dx \otimes dy) \pi(dx' \otimes dy')\right)^{1/p},
\end{equation}
where the infimum is over \emph{measure couplings} (or joint measures) $\pi$ on $X \times Y$ whose marginals agree with $\mu_X$ and $\mu_Y$, respectively (see \Cref{sec:Wasserstein} for more details). 

\para{A general framework of comparing parameterized measure networks.}
Our proposed framework substantially generalizes the GW distance, extending it to a family of metrics designed to compare general parameterized measure networks.
As a concrete illustration, we next show how the GW framework described in \Cref{eqn:GW_distance_intro} can be adapted to the setting of time-varying metric spaces. 

Let $\Xcal = (X,\mu_X,(d_X^t)_{t \in [0,1]})$, where $(X,\mu_X)$ is a Polish probability space as in the setting of metric measure spaces, but the metric structure is replaced by a one-parameter family $(d_X^t)_{t \in [0,1]}$ of metrics on $X$, assumed to vary continuously in $t$. Given another such structure $\Ycal = (Y,\mu_Y,(d_Y^t)_{t \in [0,1]})$, one can define a notion of GW distance between $\Xcal$ and $\Ycal$ by
\begin{equation}
\label{eqn:GW_distance_intro_time_varying}
    \mathsf{GW}_p(\Xcal,\mathcal{Y}) \coloneqq \frac{1}{2} \inf_\pi \int_{0}^1 \left(\int_{(X \times Y)^2} \int_{(X \times Y)^2} |d_X^t(x,x') - d_Y^t(y,y')|^p \pi(dx \otimes dy) \pi(dx' \otimes dy')\right)^{1/p} \nu(dt),
\end{equation}
where $\nu$ is a Lebesgue measure on $[0,1]$. 

The structure of the metric in \cref{eqn:GW_distance_intro_time_varying} is quite natural, and related ideas have appeared in previous work~\cite{munch2013applications,sturm2018super,xian2022capturing}. In contrast, our general framework encompasses a broader class of novel metrics that have not yet been explored in the literature. In particular, it includes a variant that incorporates an additional optimal transport–based alignment step for comparing parameterized measure networks defined over different parameter spaces—an approach especially relevant for applications involving, for example, random graph models.  

\para{Contributions.} 
We now provide a more detailed account of our contributions.

\begin{enumerate}[leftmargin=*,itemsep=1ex]

\item \narrowbf{New model for parameterized data.} 
We introduce the notion of a \emph{parameterized measure network} (\emph{pm-net}) (\Cref{def:parametrized-networks}), which unifies all of the data types described above. Analogous to a metric measure space, a pm-net is defined as a tuple
$$
\Xcal = (X,\mu_X,\Omega_X,\nu_X,\omega_X),
$$
where $(X,\mu_X)$ is a Polish probability space, $(\Omega_X,\nu_X)$ is a measured \emph{parameter space}, and $\omega_X = (\omega_X^t)_{t \in \Omega_X}$ is a parameterized family of kernels $\omega_X^t : X \times X \to \Rspace$.
Examples of pm-nets, along with their connection to the motivating examples discussed earlier, are given in \Cref{sec:examples_of_pmnets}.

\item \narrowbf{A general family of Gromov-Wasserstein-type distances.} We define a new family of GW-type distances, called \emph{parameterized Gromov–Wasserstein distances} (\Cref{def:multiscale_GW}). Members of this family are denoted $\mathsf{GW}_\mathsf{C}$, where the subscript $\mathsf{C}$ specifies a chosen \emph{cost structure}—that is, a rule for quantifying the geometric distortion induced by a given probability coupling. The metric properties of $\mathsf{GW}_\mathsf{C}$, which depend on the choice of $\mathsf{C}$, are established in \Cref{thm:metric_structure} and \Cref{thm:isomorphism_of_pm_nets}. The first of these results is stated at a high level of generality and provides a category-theoretic interpretation of GW-type distances, which may be of independent interest in optimal transport theory (see \Cref{rem:category_theory}). The second result is more specialized: its proof shares key ideas with standard arguments for existing GW variants (see, e.g.,~\cite{memoli2011,ChowdhuryMemoli2019,bauer2024z}), but requires substantial work to extend them to the parameterized setting. In particular, the proof relies on technical lemmas concerning the (lower semi-)continuity of a certain distortion functional (\Cref{lem:continuity_lemma} and \Cref{lem:continuity_lemma_extended}), as well as on a novel and somewhat subtle equivalence relation for parameterized networks (\Cref{def:equivalence_of_measure_networks}).

\item \narrowbf{Generalizations of existing metrics.}~Certain instances of the parameterized GW framework recover metrics that have previously appeared in the literature. In particular, we show that for specific choices of $\mathsf{C}$ under suitable technical assumptions, our metric coincides with:
\begin{itemize}[leftmargin=*]
\item the temporal alignment–based GW distance of~\cite{cohen2021aligning} (\Cref{ex:parameterized_by_R});
\item certain instances of the $Z$-GW distances of~\cite{bauer2024z} (\Cref{prop:Z-GW});
\item metrics for time-varying metric spaces from~\cite{sturm2018super} (\Cref{ex:Sturm}) and~\cite{kim2021spatiotemporal,kim2020persistent} (\Cref{ex:Kim-Memoli});
\item GW-type distances for heat kernels from~\cite{memoli2011spectral} (\Cref{ex:Memoli-HK}) and~\cite{spectralGW} (\Cref{ex:Chowdhury-HK}).
\end{itemize}
    
\item \narrowbf{Lower bounds and stability of invariants.}~It is shown in \Cref{thm:wasserstein_estimate} that (for a particularly useful choice of $\mathsf{C}$), the distance $\mathsf{GW}_\mathsf{C}$ can be lower bounded by a Wasserstein distance (see \Cref{sec:Wasserstein}) defined over the (classical) GW space. It follows that the parameterized GW distance $\mathsf{GW}_\mathsf{C}$ is lower bounded by a polynomial-time computable pseudometric, defined also in terms of Wasserstein distances (see \Cref{cor:global_distribution_stability} and \Cref{rem:computation}). On one hand, these lower bounds give computationally tractable estimates of the parameterized GW distance. Alternatively, we interpret these lower bounds as proofs that certain invariants of random graph and random metric space models are stable. For example, \Cref{cor:global_distribution_stability_graphs} shows that the distribution of total edges in a random graph is a GW-stable invariant of the model.

\item \narrowbf{Sampling convergence for random graphs and metric spaces.}~In the random graph model setting, one does not typically have access to the full parameter space $(\Omega_X,\nu_X)$, but rather samples from it. In \Cref{sec:approximation_by_samples}, we study approximation of parameterized GW distances from random samples of the parameter space; in particular, \Cref{thm:approx_random_metrics} shows that estimates from random samples converge to the true distance as the number of samples goes to infinity. 

\item {\narrowbf{Numerical experiments.}} 
We describe our implementation of parametrized GW distances in \Cref{sec:implementation}. We adapt components of the Python Optimal Transport (POT) library~\cite{python-ot} (which implements algorithms from \cite{gw_averaging, ot-structured-data}) to approximate $\GW_{\mathsf{C}}$ using gradient descent for several choices of cost structure $\mathsf{C}$. We provide explicit formulas for the gradient of $\GW_{\mathsf{C}}$ and other auxiliary quantities. 

We perform a number of numerical experiments, which illustrate the intuition that parametrized GW distances integrate information from all parameters $t \in \Omega$. We give qualitative examples which show that parameterized GW distances are able to pick up subtle structures of data at multiple scales in \cref{sec:pandas} and \cref{sec:nested-cycles}. In \cref{sec:random-graphs}, we show that the parameterized GW distance serves as a meaningful invariant for comparing two samples of random graphs. Finally, we incorporate our metrics into a supervised learning framework: rather than fixing the measure $\nu$ on the parameter set $\Omega$, we allow it to vary as a mechanism for feature selection. We evaluate this idea by clustering dynamic metric spaces (parameterized by time) that differ only at specific time intervals (\cref{sec:feature-selection}). 
\end{enumerate}

\para{Outline.}~We begin in \Cref{sec:preliminary} with a review of essential concepts from measure theory and optimal transport. In \Cref{sec:parametrized-GW}, we formally define parameterized measure networks and introduce a family of GW distances for comparing them, establishing their core metric properties. \Cref{sec:results} presents our main theoretical results on the equivalence and estimation of these distances, including their connections to existing metrics, lower bounds, and sampling convergence guarantees. Finally, \Cref{sec:experiments} details our computational framework for estimating the proposed distances and demonstrates their practical utility through several numerical examples.  

\section{Preliminary Concepts and Notations}
\label{sec:preliminary} 

\subsection{Basic Terminology from Measure Theory}
\label{sec:measure}
In this subsection, we review basic terminology from measure theory, which experts may safely skip. This also serves to standardize our notation for the rest of the paper.

A \emph{measurable space} is a pair $(X,\Acal)$, where $X$ is a set and $\Acal$ is a $\sigma$-algebra on $X$ (i.e., a collection of subsets of $X$, called \emph{measurable sets}, that is closed under complements, countable unions and intersections).
Given a measurable space $(X, \Acal)$, a measure $\mu_X$ is a function $\mu_X : \Acal \to [0,\infty]$ that assigns a non-negative value to each measurable set such that it is countably additive and $\mu_X(\emptyset)=0$.
A \emph{measure space} is a measurable space together with a measure, denoted as a triple $(X, \Acal, \mu_X)$; sometimes the $\sigma$-algebra $\Acal$ is omitted from the notation when it is clear from the context---we generally assume that $X$ is endowed with a topology and that it is endowed with the Borel $\sigma$-algebra, generated by open sets. 
Let $(X, \Acal, \mu_X)$ be a measure space; a property is said to hold \emph{$\mu_X$-almost everywhere} (often written as $\mu_X$-a.e.) on $X$ if the set of points where the property does not hold has measure zero. 
Given a pair of measure spaces $(X, \Acal, \mu_X)$ and $(Y, \Bcal, \mu_Y)$, a \emph{product measure} is a measure
$\mu_X \otimes \mu_Y : \mathcal{A}\otimes\mathcal{B} \to [0,\infty]$
satisfying
$(\mu_X\otimes\mu_Y)(A\times B) = \mu_X(A)\,\mu_Y(B)$, $\forall A\in\mathcal{A},\; B\in\mathcal{B}$.
For a measure space $(X,\mu_X)$ and $p\in[1,\infty]$, we use $\|\cdot\|_{L^p(X;\mu_X)}$ to denote the standard $L^p$-norm. We use $L^p(X;\mu_X)$ to denote the set of all measurable functions $f:X \to \mathbb{R}$ with finite $L^p$-norm. In particular $L^\infty(X;\mu_X)$ is the set of essentially bounded functions.

A \emph{Polish space} is a topological space $(X,\tau)$ that is \emph{separable} (i.e., it contains a countable dense subset) and \emph{completely metrizable} (i.e., there exists a metric $d_X$ on $X$ such that $d_X$ induces the topology $\tau$ and $(X,d_X)$ is complete). Here, the metric is not part of the structure; we only require that one exists.
In contrast, a \emph{Polish metric space} is a metric space $(X,d_X)$ that is both \emph{complete} (i.e., every Cauchy sequence converges) and \emph{separable}. In this case, the metric $d_X$ is part of the structure.
Every Polish metric space gives rise to a Polish space by forgetting the metric, and every Polish space admits some compatible metric that turns it into a Polish metric space. A \emph{Polish probability space} is a measure space $(X,\mu_X)$ where $X$ is assumed to be a Polish space and $\mu_X$ is a Borel probability measure, i.e., $\mu_X(X) = 1$. We use $\Pcal(X)$ to denote the set of probability measures on $X$. Given a $\mu_X \in \Pcal(X)$, the \emph{support} of $\mu_X$, denoted $\supp(\mu_X)$, is the set of $x \in X$ such that every open neighborhood of $x$ has positive measure.

Let $(X,\mu_X)$ be a Polish probability space. Given another Polish space $Y$ and a (Borel) measurable map $f:X \to Y$, the \emph{pushforward} of $\mu_X$ to $Y$, denoted $f_\# \mu_X$, is the measure on $Y$ defined by $(f_\# \mu_X)(A) = \mu_X(f^{-1}(A))$ (for $A$ being any Borel set). 
If $\mu_Y$ is a probability measure on $Y$, the map $f$ is \emph{measure-preserving} if $f_\# \mu_X = \mu_Y$. 

More generally, given two Polish probability spaces $(X,\mu_X)$ and $(Y,\mu_Y)$, a \emph{coupling} of $\mu_X$ and $\mu_Y$ is a measure $\pi \in \Pcal(X \times Y)$ whose left and right marginals are $\mu_X$ and $\mu_Y$, respectively; that is, using $p_X:X \times Y \to X$ and $p_Y:X \times Y \to Y$ to denote the coordinate projections, $(p_X)_\# \pi = \mu_X$ and $(p_Y)_\# \pi = \mu_Y$. We use $\mathcal{C}(\mu_X,\mu_Y)$ to denote the collection of all couplings of $\mu_X$ and $\mu_Y$.

\subsection{Wasserstein Distances}
\label{sec:Wasserstein}

The notion of measure coupling leads naturally to the Kantorovich formulation of transport distance between measures, the core object of study in optimal transport theory~\cite{Villani2009}.  
Let $(X,d_X)$ be a Polish metric space and let $\mu_X, \mu'_X \in \Pcal(X)$ be any two probability measures on $X$. For $p \in [1,\infty]$, the order-$p$ \define{Wasserstein distance}~\cite[Definition 6.1]{Villani2009} between $\mu_X$ and $\mu'_X$ is 
\begin{equation}
\label{eqn:Wasserstein_Distance}
    \mathsf{W}^{d_X}_p(\mu_X,\mu'_X) \coloneqq \inf_{\pi \in \mathcal{C}(\mu_X,\mu'_X)} \|d_X\|_{L^p(X \times X; \pi)} \underset{p < \infty}{=} \inf_{\pi \in \mathcal{C}(\mu_X,\mu'_X)} \left(\int_{X \times X} d_X(x,x')^p \pi(dx \otimes dx')\right)^{1/p}.
\end{equation}
Here, and throughout the rest of the paper, we use the notation $\underset{p < \infty}{=}$ to indicate that the integral formulation is valid for $p<\infty$, whereas the $L^p$-norm definition holds for any $p \in [1,\infty]$. 

\subsection{Measure Networks and Gromov-Wasserstein Distances}
\label{sec:GW} 

The Wasserstein distance described above is able to compare measures defined over the same (Polish) metric space, whereas this paper is primarily interested in comparing distributions defined over \emph{distinct} spaces. This is handled with the Gromov-Wasserstein (GW) framework \cref{eqn:GW_distance_intro}, whose purview is extended beyond metric measure spaces, following the work of Chowdhury and M\'{e}moli~\cite{ChowdhuryMemoli2019}. 

In order to handle kernel structures which are more general than metrics, Chowdhury and M\'{e}moli introduced the following concept.

\begin{defn}[Measure Network~{\cite[Definition 2.1]{ChowdhuryMemoli2019}}]
A \define{measure network} is a triple $\Xcal = (X,\mu_X,\omega_X)$ such that $X$ is a compact Polish space, $\mu_X$ is a fully supported Borel probability measure, and $\omega_X$ is a bounded measurable function on $X \times X$. 
In other words, $(X,\mu_X)$ is a Polish probability space and $\omega_X:X \times X \to \R$ is a kernel belonging to $L^\infty(X \times X, \mu_X \otimes \mu_X)$. 
\end{defn}

\begin{example}
\label{ex:measure_network_examples}
If $\omega_X$ is a distance metric (inducing the given Polish space a topology $\tau$ on $X$), then $\Xcal=(X,\mu_X,\omega_X)$ is a metric measure space (mm-space). However, the measure network formalism allows much more general structures than an mm-space. Of particular interest is the case where $X$ is a finite set of nodes for a graph structure, and $\omega_X$ is a kernel encoding node interactions. For example, $\omega_X$ could be an adjacency function (possibly including edge weights), that is, $\omega_X: X \times X \to \Rspace_{\geq 0}$, where $\omega_X(u,v) = w_e$ if nodes $u$ and $v$ are connected by an edge $e$ with a weight $w_e$; otherwise $\omega_X(u,v) = 0$. 
$\omega_X$ could also be a graph Laplacian or a graph heat kernel (see \Cref{ex:graph_heat_kernels}). 
\end{example}

In \cite{ChowdhuryMemoli2019}, the GW distance \eqref{eqn:GW_distance_intro} was extended to a pseudometric which is able to compare general measure networks. A similar idea was considered by Sturm~\cite{Sturm2023}, where some different regularity and symmetry assumptions on the kernels were imposed. To streamline notation for the rest of the paper, we define a preliminary concept (which goes back at least to M\'{e}moli~\cite{Memoli2007}). 

\begin{defn}[Distortion]
\label{def:p-distortion}
    Fix $1 \leq p < \infty$. Let $(X,\mu_X)$ and $(Y,\mu_Y)$ be probability spaces, and let $f_X:X \times X \to \R$ and $f_Y:Y \times Y \to \R$ be essentially bounded, measurable functions. Given a coupling $\pi \in \coup(\mu_X, \mu_Y)$, define the \define{$p$-distortion} as 
    \begin{equation*}
        \dis_p(\pi, f_X, f_Y) \coloneqq \left(\int_{(X \times Y)^2} |f_X(x,x') - f_Y(y,y')|^p \, \pi(dx \otimes dy) \, \pi(dx' \otimes dy')\right)^{1/p}.
    \end{equation*}
    For $p = \infty$, define
    \begin{equation*}
        \dis_\infty(\pi, f_X, f_Y) \coloneqq \sup_{(x,y),(x',y') \in \supp(\pi)} |f_X(x,x') - f_Y(y,y')|,
    \end{equation*}
    where $\supp(\pi)$ denotes the \emph{support} of $\pi$. In other words, we have the general definition
    \[
    \dis_p(\pi, f_X, f_Y) \coloneqq \| f_X \circ (p_X, p_X) - f_Y \circ (p_Y, p_Y) \|_{L^p\left( (X \times Y)^2; \pi \otimes \pi \right)},
    \]
    where $p_X$ and $p_Y$ are the coordinate projections from $X \times Y$ to $X$ and $Y$, respectively.
\end{defn}

For $p \in [1,\infty]$, the associated \define{Gromov-Wasserstein (GW) distance} between measure networks $\Xcal$ and $\Ycal$ is given by 
\begin{equation}\label{eqn:ChowdhuryMemoli2019_distance}
\begin{split}
    \mathsf{GW}_p(\Xcal,\Ycal) &\coloneqq \frac{1}{2} \inf_{\pi \in \mathcal{C}(\mu_X,\mu_Y)} \dis_p(\pi,\omega_X,\omega_Y) \\
    & \underset{p < \infty}{=} \frac{1}{2} \inf_{\pi \in \mathcal{C}(\mu_X,\mu_Y)} \left(\int_{(X \times Y)^2} |\omega_X(x,x') - \omega_Y(y,y')|^p \pi(dx \otimes dy) \pi(dx' \otimes dy') \right)^{1/p}.
\end{split}
\end{equation}
It was shown in \cite{ChowdhuryMemoli2019} that $\mathsf{GW}_p$ defines a pseudometric on the space of measure networks, with $\mathsf{GW}_p(\Xcal,\Ycal) = 0$ if and only if $\Xcal$ and $\Ycal$ are \emph{weakly isomorphic}, defined as follows~\cite[Definition 2.4]{MemoliNeedham2024}.

\begin{defn}[Weakly Isomorphic]
\label{def:weakly_isomorphic}
Let $\Xcal$ and $\Ycal$ be measure networks. A measure network $\Zcal$ is called a \define{stabilization} of $\Xcal$ and $\Ycal$ if there exist maps $\varphi_X:Z \to X$ and $\varphi_Y:Z \to Y$ such that (i) $\varphi_X$ and $\varphi_Y$ are measure-preserving, and (ii) $\omega_Z(z,z') := \omega_X(\varphi_X(z),\varphi_X(z')) =\omega_Y(\varphi_Y(z),\varphi_Y(z'))$ holds for $(\mu_Z \otimes \mu_Z)$-almost every $(z, z') \in Z \times Z$.
If a stabilization exists, we say that $\Xcal$ and $\Ycal$ are \define{weakly isomorphic}. 
\end{defn}
  
\section{Parameterized Gromov-Wasserstein Distances}
\label{sec:parametrized-GW}

\subsection{Parameterized Measure Networks}
\label{sec:parametrized-measure-networks}

Motivated by the examples described in~\cref{sec:introduction}, we now introduce a general framework for encoding objects consisting of a parameterized family of kernels over a fixed set. 

\subsubsection{Main Definition.} 
The primary objects of interest in this paper are defined as follows.

\begin{defn}[Parameterized Measure Network]
\label{def:parametrized-networks}
A \define{parameterized measure network} (abbreviated as \define{pm-net}) is a 5-tuple of the form $\Xcal = (X,\mu_X,\Omega_X,\nu_X,\omega_X)$, where:
\begin{itemize}[leftmargin=*]
\item $X$ is a Polish space endowed with a Borel probability measure $\mu_X \in \Pcal(X)$, referred to as the \emph{underlying measure space},
\item $\Omega_X$ is a compact Polish space endowed with a Borel probability measure $\nu_X$, referred to as the \emph{parameter space}, and
\item $\omega_X$ is a function
\begin{align*}
\omega_X:\Omega_X &\to L^\infty(X\times X; \mu_X \otimes \mu_X) \\
t &\mapsto \omega_X^t
\end{align*}
which is continuous with respect to the $L^\infty$ norm, referred to as a \emph{parameterized network kernel}.
\end{itemize}
\end{defn}

\begin{remark}
Various technical conditions in \Cref{def:parametrized-networks} could be relaxed—for example, one could assume each $\omega_X^t$ lies in $L^p$ rather than $L^\infty$, or weaken the compactness requirement on $\Omega_X$—but we impose these conditions for convenience. They provide a sufficiently flexible framework while sparing us from having to verify an excess of intricate technical details in the proofs.    
\end{remark}

\subsubsection{Examples of Parameterized Measure Networks.}
\label{sec:examples_of_pmnets} 
We now present several examples of parameterized measure networks, starting with the simple observation that this framework generalizes the standard measure network as a special case.

\begin{example}[Measure Networks]
\label{ex:measure_networks}
Let $(X,\mu_X,\omega_X)$ be a \emph{measure network}, in the sense of \cite{ChowdhuryMemoli2019} (see \Cref{sec:GW}). This gives a trivial example of a parameterized measure network: let $(\Omega_X,\nu_X)$ be a space consisting of a single point, $\Omega_X = \{t\}$, and define $\Xcal = (X,\mu_X,\Omega_X,\nu_X,\omega_X)$, where (by a slight abuse of notation) $\omega_X^t = \omega_X$. 
\end{example}

The next few examples take (a subset of) $\R$ as the parameter space. In these cases, a parameter $t$ intuitively represents a notion of ``time'' or ``scale''. These examples formalize concepts which were described informally in the introduction.

\begin{example}[Time-Varying Metric Spaces]
\label{ex:time_varying}
    A \emph{time-varying metric space} or \emph{dynamic metric space} consists of a (finite) set $X$ endowed with a collection $(d_X^t)_{t \in \Omega_X}$ of (pseudo-)metrics, parameterized by some compact subset $\Omega_X$ of $\R$. Such objects were studied as models for flocking behavior in \cite{sumpter2010collective} and metrics on the space of these objects were studied in \cite{kim2020persistent,kim2021spatiotemporal,kim2020analysis,xian2022capturing}, using constructions similar to Gromov-Hausdorff and Gromov-Wasserstein distances, as well as ideas from topological data analysis. Setting  $\omega_X^t = d_X^t$ and imposing the mild assumption that $t \mapsto \omega_X^t$ is $L^\infty$-continuous, one obtains a representation of a dynamic metric space as a parameterized measure network for any choices of  $\nu_X \in \Pcal(\Omega_X)$ and $\mu_X \in \Pcal(X)$.
\end{example}

\begin{example}[Time-Varying Networks]
\label{ex:time_varying_networks}
Similar to the above, one can consider data consisting of a weighted graph whose edge weights vary in time $t$, defined over some compact subset of real numbers (e.g., ~\cite{skarding2021foundations,greene2010tracking,yoo2016topological,kim2020analysis,billings2021simplicial}). Taking $\omega_X^t$ to be a graph kernel for each $t$ (e.g., the weighted adjacency matrix), and choosing necessary distributions, leads to a pm-network representation of this \emph{time-varying network} structure. 
\end{example}

\begin{example}[Riemannian Heat Kernels]
\label{ex:riemannian_heat_kernels}
Let $X$ be a (say, compact) Riemannian manifold with associated Laplacian operator $\Delta$. A solution $\omega: (0,\infty) \times X \times X \to \R$ of the \emph{heat equation}
\[
\frac{\partial}{\partial t} \omega (t,x,x') = \Delta_x \omega (t,x,x'),
\]
which limits to the delta distribution at $x$ as $t \to 0^+$, is called a \emph{heat kernel} for $X$. Let $\Omega_X$ be a compact subset of $\Rspace_{>0}$, endowed with some measure $\nu_X$, let $\mu_X$ denote the normalized Riemannian volume measure on $X$, and let $\omega_X^t(x,x') = \omega(t,x,x')$ for a heat kernel $\omega$. This data then defines a parameterized measure network.
\end{example}

\begin{example}[Graph Heat Kernels]\label{ex:graph_heat_kernels}
The heat kernels described in \Cref{ex:riemannian_heat_kernels} have a  discrete counterpart in graph theory. Let $G$ be a finite graph with node set $X$. The natural discrete version of the Riemannian Laplacian operator is the \emph{graph Laplacian}. We formulate it in matrix notation as follows. Choosing an enumeration of $X = (x_1,\ldots,x_n)$, let $A$ denote the associated $n\times n$ adjacency matrix, and $D$ the $n \times n$ degree matrix (the diagonal matrix whose $i$-th diagonal entry is the degree of node $x_i$ in $G$). Then the \emph{graph Laplacian} is the matrix $\Delta = D - A$. The \emph{graph heat equation} is typically written as 
\[
\frac{d}{dt} B(t) = - \Delta B(t), 
\]
and a solution $B$, understood to be a time-varying $n\times n$ matrix, is given by the \emph{graph heat kernel}
\[
B(t) = \exp(-t \Delta).
\]
One thus obtains a parameterized measure network by taking $\mu_X$ to be some measure over $X$, $\nu_X$ to be some measure over a compact parameter space $\Omega_X \subset \Rspace_{>0}$, and by setting 
\[
\omega_X^t(x_i,x_j) = B(t)_{i,j},
\]
where $B(t)_{i,j}$ denotes the $(i,j)$-entry of the matrix $B(t)$. 
\end{example}

The following examples have a different flavor than those above, in that the parameter space is treated as a state space, so that the parameterized network kernel is naturally considered as a random variable (valued in a function space). 

\begin{example}[Random Graphs]
\label{ex:random_graphs}
    Let $X$ be a finite set and let $\Omega_X$ be the set of all (say, simple) graphs with node set $X$. This is a finite set, which we endow with the discrete topology. A distribution $\nu_X$ over $\Omega_X$ can then be understood as a \emph{random graph model}; for example, the \emph{Erd\H{o}s-R\'enyi model}~\cite{Erdos:1959:pmd}, the \emph{stochastic block model}~\cite{holland1983stochastic}, or the \emph{Watts-Strogatz model}~\cite{watts1998collective} (see \cite{van2014random} for a survey on the topic). We construct a pm-net $\Xcal$ by choosing a measure $\mu_X$ on $X$ (e.g., the uniform measure), and taking $\omega_X^t$ to be the adjacency kernel associated to the graph $t \in \Omega_X$ (i.e., $\omega_X^t(x,x') = 1$ if the nodes $x,x' \in X$ are connected by an edge in the graph at time $t$, and $\omega_X^t(x,x') = 0$ otherwise). 

    We make the observation here that one typically does not have access to the full distribution $\nu_X$ in practice. Rather, the standard examples described above are generative models; one generally has access to iid samples $t_1,\ldots,t_N$ from $\Omega_X$, and hence to samples $\omega_X^{t_1},\ldots,\omega_X^{t_N}$ of the the kernel defining the pm-net structure. This motivates the statistical questions that we consider in \Cref{sec:approximation_by_samples}. 
\end{example}

\begin{example}[Random Metric Spaces]\label{ex:random_metric_spaces} This example is similar to the random graph example above, extending it to consider random metric structures. Towards this end, let $(X,\mu_X)$ be a compact Polish probability space and let $\Omega_X$ be a compact collection of metrics inducing the given topology on $X$, endowed with the subspace topology coming from the inclusion $\Omega_X \subset L^\infty(X \times X; \mu_X \otimes \mu_X)$. A distribution $\nu_X$ on $\Omega_X$ defines a \emph{random metric space model} over $X$. There is an associated pm-net $\Xcal$ given by taking $\omega_X^t$ to be equal to the metric $t \in \Omega_X$. 

As in \Cref{ex:random_graphs}, one typically only has access to samples of metrics distributed according to $\nu_X$, and statistical inferences come into play (see~\Cref{sec:approximation_by_samples}). In practice, these samples could arise from noisy measurements of some metric structure; to give a concrete example, the nuclear magnetic resonance problem in structural biology represents molecular conformation via many noisy measurements of pairwise distances between its atoms~\cite{crippen1988distance}. 
\end{example}

We conclude our list of examples by describing a situation where the parameter space consists of a set of modalities for representing a given dataset.

\begin{example}[Graph Representations]\label{ex:graph_representations}
    Given a graph $G$ with a node set $X$, there are many choices of kernels $X \times X \to \R$ for representing the structure of $G$, including: the adjacency kernel, the graph Laplacian, the graph heat kernel for various choices of $t$, the shortest path distance function, or other kernels induced by node features (if they exist). One can consider a multimodal representation of $G$ by setting $\Omega_X$ to be a (say, finite) set of representation modalities $t$ and defining $\omega_X^t$ to be the kernel for $G$ under the given modality. For any choices of distributions $\nu_X$ and $\mu_X$, this data determines a parameterized measure network. 
\end{example}

\subsection{Distances Between Parameterized Measure Networks} 
\label{sec:distances}

Our next objective is to define a suitable notion of distance between parameterized measure networks. Rather than tailoring a distinct distance for each specific class of objects, we adopt a unified and general framework that accommodates a wide range of structures, including those introduced in \Cref{ex:measure_networks} through \Cref{ex:graph_representations}. This generality is essential, as the objects we seek to analyze are inherently diverse. By formulating a family of distances at this level of abstraction, we are able to derive theoretical guarantees applicable across multiple settings, specializing only when necessary to address particular cases.

\subsubsection{Classes of Parameterized Measure Networks.} Throughout the rest of the section, let $\mathfrak{N}$ denote some fixed but arbitrary \define{class of parameterized measure networks}. The family of distances introduced below will be defined with respect to the class $\mathfrak{N}$. 

\begin{example}[Classes of Parameterized Measure Networks]
\label{ex:classes_of_pmnets}
    We will return frequently to the following important classes $\mathfrak{N}$ of pm-nets:
    \begin{enumerate}
        \item the class containing \emph{all} pm-nets, which we denote as $\mathfrak{N}_\mathrm{all}$;
        \item for a fixed parameter space $(\Omega,\nu)$, the class of pm-nets of the form $\Xcal = (X,\mu_X,\Omega,\nu,\omega_X)$, which we denote as $\mathfrak{N}_{\nu}$;
        \item one of the more specific classes consisting of the objects described in  \Cref{ex:measure_networks}--\Cref{ex:random_metric_spaces}, for which we do not currently  introduce any specialized notation.
    \end{enumerate}
\end{example}

\subsubsection{General Family of Distances.} Our general family of distances on a fixed class $\mathfrak{N}$ is defined as follows.

\begin{defn}[Parameterized Gromov-Wasserstein Distance]
	\label{def:multiscale_GW}
    A \define{cost structure} on $\mathfrak{N}$ is an assignment $\mathsf{C}$ taking a pair of pm-nets $\Xcal,\Ycal \in \mathfrak{N}$ to a function 
    \[
    \mathsf{C}_{\Xcal,\Ycal}:\mathcal{C}(\mu_X,\mu_Y) \to \R_{\geq 0}.
    \]
    Given a cost structure $\mathsf{C}$, the \define{parameterized Gromov-Wasserstein distance induced by $\mathsf{C}$} is 
    \begin{equation}\label{eqn:GW_C}
    \begin{split}
        \GW_\mathsf{C}:\mathfrak{N} \times \mathfrak{N} &\to \R_{\geq 0} \\
        (\Xcal,\Ycal) &\mapsto \GW_\mathsf{C}(\Xcal,\Ycal) \coloneqq \inf_{\pi \in \mathcal{C}(\mu_X,\mu_Y)} \mathsf{C}_{\Xcal,\Ycal}(\pi). 
    \end{split}
    \end{equation}
\end{defn}

\begin{remark}
    We refer to $\GW_\mathsf{C}$ as a “distance” only in a colloquial sense, as we do not assert that it satisfies the axioms of a metric in general. However, we show below that under suitable assumptions on the cost function $\mathsf{C}$, $\GW_\mathsf{C}$ does indeed exhibit metric properties. 
\end{remark}

\begin{remark}
    An optimization problem similar to \Cref{eqn:GW_C} was studied in a recent paper by Sebbouh, Cuturi, and Peyr\'{e}~\cite{sebbouh2024structured}, with a view toward generalizing duality properties in optimal transport-type problems. There, the additional assumption that $\mathsf{C}_{\Xcal,\Ycal}$ is always a concave function was imposed for analytical purposes, but we make no such restriction.
\end{remark} 

\subsubsection{Continuity of Distortion.} Before providing examples of interesting cost structures $\mathsf{C}$, we establish a useful property of the distortion function, introduced in \Cref{def:p-distortion}.

\begin{lemma}
\label{lem:continuity_lemma}
     Let $\Xcal$ and $\Ycal$ be pm-nets and fix $1 \leq p \leq \infty$. For all $s \in \Omega_X$ and $t \in \Omega_Y$, the quantity $\dis_p(\pi, \omega^s_X, \omega^t_Y)$ is well-defined; in particular, it is finite. Moreover, if $p < \infty$, the function
    \begin{align*}
        \dis_p:\coup(\mu_X, \mu_Y) \times \Omega_X \times \Omega_Y &\to \R \\
        (\pi,s,t) &\mapsto \dis_p(\pi, \omega_X^s, \omega_Y^t)
    \end{align*}
    is continuous. If $p = \infty$, $\dis_\infty$ is lower semicontinuous.
\end{lemma}

\begin{proof}
    The finiteness claim is straightforward, due to our regularity assumptions in \Cref{def:parametrized-networks}:
        \[
        \mathrm{dis}_p(\pi,\omega_X^s,\omega_Y^t) \leq \|\omega_X^s\|_{L^p(\mu_X \otimes \mu_X)} + \|\omega_Y^t\|_{L^p(\mu_Y \otimes \mu_Y)} \leq \|\omega_X^s\|_{L^\infty(\mu_X \otimes \mu_X)} + \|\omega_Y^t\|_{L^\infty(\mu_Y \otimes \mu_Y)} < \infty.
        \]

    We proceed with the continuity claim. 
    Suppose $p < \infty$. Let $\epsilon > 0$ and fix $\pi_0 \in \coup(\mu_X, \mu_Y)$, $s_0 \in \Omega_X$, and $t_0 \in \Omega_Y$. By \cite[Lemma 2.3]{ChowdhuryMemoli2019}, for any fixed $s$ and $t$, the function
    	\begin{align*}
		\dis_p(\bullet, \omega_X^s, \omega_Y^t): \coup(\mu_X, \mu_Y) &\to \R \\
		\pi &\mapsto \dis_p(\pi, \omega_X^s, \omega_Y^t)
	\end{align*}
	is continuous. Thus there exists $V \subset \coup(\mu_X, \mu_Y)$ neighborhood of $\pi_0$ such that for all $\pi \in V$,
	\begin{equation*}
		|\dis_p(\pi, \omega_X^{s_0}, \omega_Y^{t_0}) - \dis_p(\pi_0, \omega_X^{s_0}, \omega_Y^{t_0})| < \epsilon.
	\end{equation*}
	Similarly, since $\omega_X:\Omega_X \to L^\infty(X \times X, \mu_X \otimes \mu_X)$ is continuous, there exists $U_X \subset \Omega_X$ neighborhood of $s_0$ such that
	\begin{equation*}
		\|\omega_X^{s} - \omega_X^{s_0}\|_{L^p(X^2, \mu_X \otimes \mu_X)} \leq \|\omega_X^{s} - \omega_X^{s_0}\|_{L^\infty(X^2, \mu_X \otimes \mu_X)} < \epsilon
	\end{equation*}
	for all $s \in U_X$. We define $U_Y$ analogously.
    
    Let $p_X$ and $p_Y$ be the standard projections from $X \times Y$. For convenience, write $\overline{\omega}_X^s = \omega_X^s \circ (p_X, p_X)$ and $\overline{\omega}_Y^t = \omega_Y^t \circ (p_Y, p_Y)$. By the triangle inequality of the $L_p$ norm,
	\begin{align*}
		\dis_p(\pi, \omega_X^{s_0}, \omega_Y^{t_0})
		&= \| \overline{\omega}_X^{s_0} - \overline{\omega}_Y^{t_0} \|_{L^p\left( (X \times Y)^2; \pi \otimes \pi \right)} \\
		&\leq
		\| \overline{\omega}_X^{s_0} - \overline{\omega}_X^{s} \|_{L^p\left( (X \times Y)^2; \pi \otimes \pi \right)}
		+
		\| \overline{\omega}_X^s - \overline{\omega}_Y^t \|_{L^p\left( (X \times Y)^2; \pi \otimes \pi \right)}
		+
		\| \overline{\omega}_Y^{t} - \overline{\omega}_Y^{t_0} \|_{L^p\left( (X \times Y)^2; \pi \otimes \pi \right)} \\
		&=
		\| \omega_X^{s_0} - \omega_X^{s} \|_{L^p(X^2; \mu_X \otimes \mu_X)}
		+
		\| \overline{\omega}_X^s - \overline{\omega}_Y^t \|_{L^p\left( (X \times Y)^2; \pi \otimes \pi \right)}
		+
		\| \omega_Y^{t} - \omega_Y^{t_0} \|_{L^p(Y^2; \mu_Y \otimes \mu_Y )} \\
		&< \dis_p(\pi, \omega_X^{s}, \omega_Y^{t}) + 2\epsilon.
	\end{align*}
	With a symmetric argument we obtain $|\dis_p(\pi, \omega_X^{s}, \omega_Y^{t}) - \dis_p(\pi, \omega_X^{s_0}, \omega_Y^{t_0})| < 2\epsilon$. Then for all $\pi \in V$, $s \in U_X$, and $t \in U_Y$,
	\begin{multline*}
		|\dis_p(\pi, \omega_X^{s}, \omega_Y^{t}) - \dis_p(\pi_0, \omega_X^{s_0}, \omega_Y^{t_0})| \\
		\leq
		|\dis_p(\pi, \omega_X^{s}, \omega_Y^{t}) - \dis_p(\pi, \omega_X^{s_0}, \omega_Y^{t_0})|
		+ |\dis_p(\pi, \omega_X^{s_0}, \omega_Y^{t_0}) - \dis_p(\pi_0, \omega_X^{s_0}, \omega_Y^{t_0})|
		< 3\epsilon.
	\end{multline*}
	This proves that $\dis_p$ is continuous for $p<\infty$. Since $\dis_\infty$ is the supremum over $p \geq 1$ of the family of continuous functions $\dis_p$, it is lower semicontinuous. 
\end{proof}

\subsubsection{Examples of Parameterized Gromov-Wasserstein Distances.} The following examples of cost structures $\mathsf{C}$, along with their associated parameterized Gromov–Wasserstein distances, will be referenced frequently throughout the paper.

\begin{example}[Main Example: Fixed Parameter Space]\label{ex:fixed_parameter_space}
    Fix a parameter space $(\Omega,\nu)$ and let $\mathfrak{N}_\nu$ be the class of pm-nets with this parameter space, that is, of the form $\Xcal = (X,\mu_X,\Omega,\nu,\omega_X)$; see \Cref{ex:classes_of_pmnets}. A natural cost structure is given by 
    \begin{equation}\label{eqn:fixed_parameters_cost}
    \mathsf{C}_{\Xcal,\Ycal}(\pi) = \frac{1}{2}\|\mathrm{dis}_p(\pi,\omega_X,\omega_Y)\|_{L^q(\Omega;\nu)} \underset{q < \infty}{=} \frac{1}{2} \left(\int_\Omega \mathrm{dis}_p(\pi,\omega_X^t,\omega_Y^t)^q \nu(dt)\right)^{1/q}.
    \end{equation}
     We record the full expression for $\GW_\mathsf{C}$ in this case, for later reference:
    \begin{equation}\label{eqn:GW_fixed_space}
    \begin{split}
        &\GW_\mathsf{C}(\Xcal,\Ycal) = \frac{1}{2}\inf_{\pi \in \mathcal{C}(\mu_X,\mu_Y)} \|\mathrm{dis}_p(\pi,\omega_X,\omega_Y)\|_{L^q(\Omega;\nu)} \\
        &\qquad \qquad \underset{p,q < \infty}{=} \frac{1}{2}\inf_{\pi \in \mathcal{C}(\mu_X,\mu_Y)} \left(\int_\Omega \left(\int_{(X \times Y)^2} |\omega_X(x,x') - \omega_Y(y,y')|^p \pi(dx \times dy) \pi(dx' \times dy')\right)^{q/p} \nu(dt)\right)^{1/q}.
    \end{split}
    \end{equation}
    Observe that the cost structure is well-defined (i.e., finite). 
    Indeed, \Cref{lem:continuity_lemma} implies that, for fixed $\pi \in \coup(\mu_X,\mu_Y)$, the map $\Omega \to \R: t \mapsto \mathrm{dis}_p(\pi,\omega_X^t,\omega_Y^t)$
        is continuous. By compactness of $\Omega$, it is therefore $q$-integrable. 
\end{example}

    \begin{remark}\label{rem:recovers_GW_distance}
        Suppose that $(\Omega,\nu)$ is a one-point space, $\Omega = \{t\}$. Let $\Xcal = (X,\mu_X,\omega_X)$ be a measure network; as was observed in \Cref{ex:measure_networks}, $\cX$ is naturally  represented as an element of the class $\mathfrak{N}_\nu$, which we denote in this remark as $\overline{\Xcal} = (X,\mu_X,\Omega,\nu,\omega_X)$. With the cost structure $\mathsf{C}$ from \Cref{ex:fixed_parameter_space}, the parameterized GW distance is equivalent to the standard GW distance for measure networks, as was described in \Cref{sec:GW}. Indeed,
        \begin{align*}
\GW_\mathsf{C}(\overline{\Xcal},\overline{\Ycal}) &= \frac{1}{2} \inf_{\pi \in \mathcal{C}(\mu_X,\mu_Y)} \|\mathrm{dis}_p(\pi,\omega_X,\omega_Y)\|_{L^q(\Omega;\nu)} \\
&= \frac{1}{2} \inf_{\pi \in \mathcal{C}(\mu_X,\mu_Y)} \mathrm{dis}_p(\pi,\omega_X^t,\omega_Y^t) = \GW_p(\Xcal,\Ycal).
        \end{align*}
    \end{remark}

    \begin{remark}\label{rem:ZGW_distance}
        For any parameter space, $\GW_\mathsf{C}$ is related to a  $Z$-Gromov-Wasserstein distance, in the sense of \cite{bauer2024z}. This is explained in detail below, in \cref{sec:ZGW}.
    \end{remark}

\begin{example}[Main Example: General Parameter Spaces]\label{ex:general_parameter_spaces}
    Let $\mathfrak{N}_\mathrm{all}$ denote the class of all pm-nets, and let $\Xcal,\Ycal \in \mathfrak{N}_\mathrm{all}$. For $p,q \in [1,\infty]$, we have the cost structure
    \begin{equation}\label{eqn:cost_structure_different_classes}
    \begin{split}
    \mathsf{C}_{\Xcal,\Ycal}(\pi) &= \frac{1}{2} \inf_{\xi \in \mathcal{C}(\nu_X,\nu_Y)} \|\mathrm{dis}_p(\pi,\omega_X,\omega_Y)\|_{L^q(\Omega_X \times \Omega_Y,\xi)} \\
    &\underset{q < \infty}{=} \frac{1}{2} \inf_{\xi \in \mathcal{C}(\nu_X,\nu_Y)} \left(\int_{\Omega_X \times \Omega_Y} \mathrm{dis}_p(\pi,\omega_X^t,\omega_Y^s)^q \xi(dt \otimes ds) \right)^{1/q},
    \end{split}
    \end{equation}
    in which case $\GW_\mathsf{C}$ becomes
    \begin{equation}\label{eqn:GW_general_parameters}
        \GW_\mathsf{C}(\Xcal,\Ycal) = \frac{1}{2} \inf_{\pi \in \mathcal{C}(\mu_X,\mu_Y)} \inf_{\xi \in \mathcal{C}(\nu_X,\nu_Y)} \|\mathrm{dis}_p(\pi,\omega_X,\omega_Y)\|_{L^q(\Omega_X \times \Omega_Y,\xi)}.
    \end{equation}
    By arguments similar to the above, \Cref{lem:continuity_lemma} implies that the cost structure is well-defined (i.e.,~finite).
\end{example}

Although our analysis and experiments in the remainder of the paper primarily focus on the cost structures presented in \Cref{ex:fixed_parameter_space} and \Cref{ex:general_parameter_spaces}, a wide range of alternative cost structures can be constructed to suit specific applications.
We largely leave detailed exploration of these possibilities for future work, but describe one such potential idea below. We also note that a discrete analogue of this construction was previously introduced in \cite{cohen2021aligning}, motivated by the problem of aligning time series valued in different spaces. 

\begin{example}[Optimization Over Reparameterizations]\label{ex:parameterized_by_R}
Consider the class $\mathfrak{N}$ consisting of pm-nets whose parameter space is $\Omega_X = [0,1]$, endowed with Lebesgue measure $\nu_X$; for example, this class contains versions of the pm-nets described in \Cref{ex:riemannian_heat_kernels}, \Cref{ex:graph_heat_kernels}, and \Cref{ex:time_varying}, when the parameter space is restricted to the interval (extending definitions to the parameter $t=0$, as necessary). Let $\mathrm{Diff}_+([0,1])$ denote the group of orientation-preserving diffeomorphisms of $[0,1]$. The cost structure
\[
    \mathsf{C}_{\Xcal,\Ycal}(\pi) = \inf_{\alpha \in \mathrm{Diff}_+([0,1])} \left(\int_0^1 \mathrm{dis}_p(\pi,\omega_X^t,\omega_Y^{\alpha(t)})^q \nu(dt)\right)^{1/q} 
\]
leads to a distance $\GW_{\mathsf{C}}$ involving a ``temporal alignment'' of the pm-nets. In particular, if $\Xcal$ and $\Ycal$ only differ by an orientation-preserving reparameterization of their parameter spaces, then one has $\mathsf{GW}_\mathsf{C}(\Xcal,\Ycal) = 0$ under this cost structure. The associated optimization problem has connections to ideas from the field of \emph{statistical shape analysis}~\cite{srivastava2016functional,bauer2024elastic}, which we plan to explore in future work.
\end{example}

\subsection{General metric properties} 
\label{sec:metric-properties}

We now formally study metric-like structures arising from parameterized GW distances. Predictably, these depend on properties of the cost structure $\mathsf{C}$. For the rest of this subsection, fix a class $\mathfrak{N}$ of pm-nets and a cost structure $\mathsf{C}$ on $\mathfrak{N}$. 

\subsubsection{Preliminary Concepts.} To state the main theorem, we need some additional terminology and notation. Recall the \emph{Gluing Lemma} \cite[Lemma 1.4]{Sturm2023}, a standard result in optimal transport theory which says that, given couplings $\pi_{XY} \in \mathcal{C}(\mu_X,\mu_Y)$ and $\pi_{YZ} \in \mathcal{C}(\mu_Y,\mu_Z)$, there exists a unique probability measure $\tilde{\pi}$ on $X \times Y \times Z$ whose $(X,Y)$- and $(Y,Z)$-marginals are $\pi_{XY}$ and $\pi_{YZ}$, respectively. We use the notation $\pi_{XY} \bullet \pi_{YZ}$ for the $(X\times Z)$-marginal of $\tilde{\pi}$; that is, letting $p_A:X \times Y \times Z \to A$ denote the coordinate projection for $A \in \{X,Y,Z\}$, we define
\[
\pi_{XY} \bullet \pi_{YZ} \coloneqq (p_X \times p_Z)_\# \tilde{\pi} \in \mathcal{C}(\mu_X,\mu_Z). 
\]

We now introduce several useful properties of a cost structure.

\begin{defn}[Properties of Cost Structures]
	\label{def:props_cost_structures}
    Let $\mathsf{C}$ be a cost structure on a class of pm-nets $\mathfrak{N}$. 
    \begin{enumerate}
        \item The cost structure $\mathsf{C}$ \define{respects gluing} if, for any $\pi_{XY} \in \mathcal{C}(\mu_X,\mu_Y)$ and $\pi_{YZ} \in \mathcal{C}(\mu_Y,\mu_Z)$, it holds that
    \[
    \mathsf{C}_{\Xcal,\cZ}(\pi_{XY} \bullet \pi_{YZ}) \leq \mathsf{C}_{\Xcal,\Ycal}(\pi_{XY}) + \mathsf{C}_{\Ycal,\cZ}(\pi_{YZ}).
    \]
    \item Recall that a map $\varphi:X \to Y$ that is measure-preserving with respect to $\mu_X$ and $\mu_Y$ induces a coupling via $(\mathrm{id}_X \times \varphi)_\# \mu_X \in \mathcal{C}(\mu_X,\mu_Y)$. 
Given a cost structure, we abuse notation and write
\[
\mathsf{C}_{\Xcal,\Ycal}(\varphi)
\coloneqq
\mathsf{C}_{\Xcal,\Ycal} \left( \left(\mathrm{id}_X \times \varphi\right)_\# \mu_X \right).
\]
The cost structure \define{respects identities} if $\mathsf{C}_{\Xcal,\Xcal}(\mathrm{id}_X) = 0$ for all $\Xcal \in \mathfrak{N}$.
    \item If $\mathsf{C}$ respects gluing and respects identities, we call it a \define{lax homomorphism}. 
    \item Given a coupling $\pi \in \mathcal{C}(\mu_X,\mu_Y)$, there is a corresponding \define{adjoint coupling} $\pi^\ast \in \mathcal{C}(\mu_Y,\mu_X)$, given by $\pi^\ast = \mathrm{swap}_\# \pi$, where $\mathrm{swap}:X \times Y \to Y \times X$ is the map $\mathrm{swap}(x,y) = (y,x)$. We say that the cost structure $\mathsf{C}$ is \define{symmetric} if 
    \[
    \mathsf{C}_{\Xcal,\Ycal}(\pi) = \mathsf{C}_{\Ycal,\Xcal}(\pi^\ast) \qquad \forall \; \pi \in \mathcal{C}(\mu_X,\mu_Y).
    \]
    \end{enumerate}
\end{defn}

\begin{remark}[Categorical Interpretation]\label{rem:category_theory}
The \emph{lax homomorphism} terminology is inspired by its connection to category theory: $\mathsf{C}$ can be viewed as a \emph{lax 2-functor} between a certain 2-category of pm-nets and a 2-category constructed via the monoidal structure of the non-negative real numbers (see, e.g.,~\cite{johnson20202,mcfaddin2023interleaving} for general background on 2-categories). Intuitively, the idea is that the cost structure translates between the gluing composition of couplings and  addition on $\R$, both considered as algebraic operations (hence it behaves like a homomorphism), but it only preserves the structure up to an inequality (this is the ``lax-ness'' of the homomorphism). As fully developing this perspective would require substantial effort and is not essential for the remainder of the paper, we defer its detailed treatment to future work.   
\end{remark}

\subsubsection{Parameterized Gromov-Wasserstein Distances as Pseudometrics.} 
We are now ready to state our result. Its proof is straightforward, owing to the careful design of our definitions. However, demonstrating that this general theorem applies to specific examples of interest requires additional work, which we present later.

\begin{theorem}\label{thm:metric_structure}
    If $\mathsf{C}$ is a symmetric lax homomorphism, then the parameterized Gromov-Wasserstein distance $\GW_\mathsf{C}$ defines a pseudometric on $\mathfrak{N}$. 
\end{theorem}

\begin{proof}
    Non-negativity and symmetry of $\GW_\mathsf{C}$ are obvious and $\mathsf{C}_{\Xcal,\Xcal}(\mathrm{id}_X) = 0$ implies that $\mathsf{GW}_\mathsf{C}(\Xcal,\Xcal) = 0$. Finally, triangle inequality follows immediately from the assumption that $\mathsf{C}$ respects gluings.
\end{proof}

\subsection{Specialized metric properties}

In this subsection, we examine special cases of the parameterized GW framework to which \Cref{thm:metric_structure} can be applied.

\subsubsection{Standard Examples of Parameterized Measure Networks and Cost Structures} Throughout this subsection, we consider the following pairs $(\mathfrak{N},\mathsf{C})$ of classes of pm-nets and cost structures, which we refer to as the \emph{standard examples}:
\begin{enumerate}
    \item $\mathfrak{N} = \mathfrak{N}_\nu$ is the class of pm-nets over a fixed parameter space $(\Omega,\nu)$, and $\mathsf{C}$ is the cost structure from \Cref{ex:fixed_parameter_space}, for some choices of $p,q \in [1,\infty]$. 
        \item $\mathfrak{N} = \mathfrak{N}_\mathrm{all}$ is the class of pm-nets whose parameter spaces are endowed with probability measures (i.e., $\Xcal \in \mathfrak{N}$ has $\nu_X(\Omega_X) = 1$), and $\mathsf{C}$ is the cost structure from \Cref{ex:general_parameter_spaces}, for some choice of $p,q \in [1,\infty]$.
\end{enumerate}

\begin{prop}
	\label{prop:optimal_couplings}
	For the standard examples $(\mathfrak{N},\mathsf{C})$, optimal couplings exist. That is, the infimum of the associated parameterized GW distance $\GW_\mathsf{C}$ is realized.
\end{prop}

The proof uses the following extension of \Cref{lem:continuity_lemma}.

\begin{lemma}\label{lem:continuity_lemma_extended}
    Let $\Xcal$ and $\Ycal$ be pm-nets, and let $p,q \in [1,\infty]$. The function
	\begin{align*}
		G_{p,q}:\coup(\nu_X, \nu_Y) \times \coup(\mu_X, \mu_Y) &\to \R \\
		(\xi, \pi) &\mapsto \left( \int_{\Omega_X \times \Omega_Y} \dis_p(\pi, \omega_X^s, \omega_Y^t)^q \, \xi(ds \otimes dt) \right)^{1/q}
	\end{align*}
	is lower semicontinuous for $1 \leq p, q \leq \infty$ and continuous if $p, q < \infty$.
\end{lemma}

\begin{proof}
    Suppose $p < \infty$ and fix $\epsilon > 0$, $\xi_0 \in \coup(\nu_X, \nu_Y)$ and $\pi_0 \in \coup(\mu_X, \mu_Y)$. The function
	\begin{align*}
		G_{p,q}(\bullet, \pi_0):\coup(\nu_X, \nu_Y) &\to \R \\
		\xi &\mapsto \left( \int_{\Omega_X \times \Omega_Y} \dis_p(\pi_0, \omega_X^s, \omega_Y^t)^q \, \xi(ds \otimes dt) \right)^{1/q}
	\end{align*}
    is continuous in the topology of weak convergence because the cost function $(s,t) \mapsto \dis_p(\pi_0, \omega_X^s, \omega_Y^t)^q$ is continuous, by \Cref{lem:continuity_lemma}, and bounded, by compactness of $\Omega_X \times \Omega_Y$. Hence, there exists $V_1 \subset \coup(\Omega_X, \Omega_Y)$ neighborhood of $\xi_0$ such that for all $\xi \in V_1$, $|G_{p,q}(\xi, \pi_0) - G_{p,q}(\xi_0, \pi_0)| < \epsilon$.
	
    For the second component of $G_{p,q}$, we claim that there exists $V_2 \subset \coup(\mu_X, \mu_Y)$, a neighborhood of $\pi_0$, such that
	\begin{equation}
		\label{eq:bound_pi_0}
		|\dis_p(\pi, \omega_X^s, \omega_Y^t) - \dis_p(\pi_0, \omega_X^{s}, \omega_Y^{t})| < 2\epsilon
	\end{equation}
	for all $\pi \in V_2$ and any $s \in \Omega_X$ and $t \in \Omega_Y$. Since $\pi \mapsto \dis_p(\pi, \omega_X^s, \omega_Y^t)$ is continuous by \cite[Lemma 2.3]{ChowdhuryMemoli2019}, we can find $V_{s,t} \subset \coup(\mu_X, \mu_Y)$ and $U_{s,t} \subset \Omega_X \times \Omega_Y$ neighborhoods of $\pi_0$ and $(s,t) \in \Omega_X \times \Omega_Y$, respectively, such that for all $\pi \in V_{s,t}$ and $(s', t') \in U_{s,t}$,
	\begin{equation}
		\label{eq:bound_pi_prime}
		|\dis_p(\pi, \omega_X^{s'}, \omega_Y^{t'}) - \dis_p(\pi_0, \omega_X^{s}, \omega_Y^{t})| < \epsilon.
	\end{equation}
	By compactness of $\Omega_X$ and $\Omega_Y$, there exists a finite cover $\{ U_{s_i, t_i} \}_{1 \leq i \leq n}$ of $\Omega_X \times \Omega_Y$. We claim that the neighborhood of $\pi_0$ defined by $V_2 = \bigcap_{1 \leq i \leq n} V_{s_i, t_i}$ is the desired set. Fix $\pi \in V_2$. For any $(s,t) \in \Omega_X \times \Omega_Y$, there exists $1 \leq k \leq n$ such that $(s,t) \in U_{s_k, t_k}$. Since $\pi \in V_2 \subset V_{s_k, t_k}$ and $(s,t) \in U_{s_k, t_k}$, inequality \Cref{eq:bound_pi_prime} gives $|\dis_p(\pi, \omega_X^s, \omega_Y^t) - \dis_p(\pi_0, \omega_X^{s_k}, \omega_Y^{t_k})| < \epsilon$. By the same reasoning, $\pi_0 \in V_2 \subset V_{s_k, t_k}$ and $(s,t) \in U_{s_k, t_k}$ imply $|\dis_p(\pi_0, \omega_X^{s}, \omega_Y^{t}) - \dis_p(\pi_0, \omega_X^{s_k}, \omega_Y^{t_k})| < \epsilon$. Hence,
	\begin{multline*}
		|\dis_p(\pi, \omega_X^s, \omega_Y^t) - \dis_p(\pi_0, \omega_X^{s}, \omega_Y^{t})| \\
		< |\dis_p(\pi, \omega_X^s, \omega_Y^t) - \dis_p(\pi_0, \omega_X^{s_k}, \omega_Y^{t_k})| + |\dis_p(\pi_0, \omega_X^{s}, \omega_Y^{t}) - \dis_p(\pi_0, \omega_X^{s_k}, \omega_Y^{t_k})|
		< 2\epsilon.
	\end{multline*}
	\indent Now we finish the proof of the continuity of $G_{p,q}$. If $\xi \in V_1$ and $\pi \in V_2$, the reverse triangle inequality of the $L^q$ norm gives
	\begin{align*}
		|G_{p,q}(\xi, \pi) - G_{p,q}(\xi_0, \pi_0)|
		&\leq |G_{p,q}(\xi, \pi) - G_{p,q}(\xi, \pi_0)| + |G_{p,q}(\xi, \pi_0) - G_{p,q}(\xi_0, \pi_0)| \\
		&< \| \dis_p(\pi, \omega_X^s, \omega_Y^t) - \dis_p(\pi_0, \omega_X^{s}, \omega_Y^{t}) \|_{L^q(\Omega_X \times \Omega_Y, \xi)} + \epsilon \\
		&< \| 2\epsilon \|_{L^q(\Omega_X \times \Omega_Y, \xi)} + \epsilon = 3\epsilon.
	\end{align*}

    Finally, we establish lower semicontinuity in the case that $p$ or $q$ is $\infty$. For fixed $p<\infty$, standard properties of $L^q$-norms imply
    \[
    G_{p,\infty}(\xi,\pi) = \sup_{1 \leq q < \infty} G_{p,q}(\xi,\pi),
    \]
    so that $G_{p,\infty}$ is a supremum of a family of continuous functions, hence lower semicontinuous. Now fix $q \in [1,\infty]$ and $(\xi,\pi)$. Observe that, by H\"{o}lder's inequality, $\dis_p(\pi,\omega_X^s,\omega_Y^t)$ is an increasing function of $p \in [1,\infty)$, with limit equal to $\dis_\infty(\pi,\omega_X^s,\omega_Y^t)$. By the Monotone Convergence Theorem, 
    \[
    G_{\infty,q}(\xi,\pi) = \|\dis_\infty(\pi,\omega_X^\bullet, \omega_Y^\bullet)\|_{L^q(\Omega_X \otimes \Omega_Y; \xi)} = \sup_{1 \leq p < \infty} \|\dis_p(\pi,\omega_X^\bullet, \omega_Y^\bullet)\|_{L^q(\Omega_X \otimes \Omega_Y; \xi)} = \sup_{1 \leq p < \infty} G_{p,q}(\xi,\pi),
    \]
    so that lower semicontinuity follows.
\end{proof}

We will also use the following general result.

\begin{lemma}\label{lem:general_optimal_couplings}
    Suppose that, for all $\Xcal,\Ycal \in \mathfrak{N}$, $\mathsf{C}_{\Xcal,\Ycal}:\mathcal{C}(\mu_X,\mu_Y) \to \R_{\geq 0}$ is lower semicontinuous in the topology of weak convergence. Then the infimum in the definition of $\GW_\mathsf{C}$ is always realized.
\end{lemma}

\begin{proof}
    By Prokhorov's theorem, the space $\mathcal{C}(\mu_X,\mu_Y)$ is compact (see \cite[Lemma 1.2]{Sturm2023} for details). The lower semicontinuous function $\mathsf{C}_{\Xcal,\Ycal}$ must therefore achieve its minimum over this set.
\end{proof}

\begin{proof}[Proof of \Cref{prop:optimal_couplings}]
    By \Cref{lem:general_optimal_couplings}, it suffices to show that $\mathsf{C}$ is lower semicontinuous. 
    The conclusion follows immediately from \Cref{lem:continuity_lemma_extended} in the setting of $\mathfrak{N}_\mathrm{all}$, whereas $\mathfrak{N}_\nu$ requires a bit more work. Indeed, in the latter case, the associated cost structure $\mathsf{C}$ is obtained by restricting the first coordinate of $G_{p,q}$ (as defined in \Cref{lem:continuity_lemma_extended}) to be the identity coupling in $\mathcal{C}(\nu,\nu)$. Lower semicontinuity of $G_{p,q}$ then implies lower semicontinuity of its restriction.
\end{proof}

\subsubsection{Metric Structure for the Standard Examples.} Finally, we show that the standard examples fall under the purview of \Cref{thm:metric_structure}, so that the associated distances $\GW_\mathsf{C}$ define pseudometrics. Moreover, in these settings, we can completely characterize the distance zero equivalence classes, using the following definition.

\begin{defn}[Isomorphisms for Parameterized Measure Networks]
\label{def:equivalence_of_measure_networks}
    Let $\Xcal$ and $\cZ$ be pm-nets. A pair of maps $(\Phi,\varphi)$, with $\Phi:\Omega_Z \to \Omega_X$ and $\varphi:Z \to X$, is called \define{structure-preserving} if 
    \begin{enumerate}
        \item\label{item:equiv1} $\Phi$ and $\varphi$ are both measure-preserving maps, and
        \item\label{item:equiv2} the pair $(\Phi,\varphi)$ preserves parameterized network kernels, in the sense that 
        \[
        \omega_Z^t(z,z') = \omega_X^{\Phi(t)}(\varphi(z),\varphi(z'))
        \]
        holds almost everywhere, with respect to $\nu_Z \otimes \mu_Z \otimes \mu_Z$. 
    \end{enumerate}

    Let $\Ycal$ be another pm-net. We say that $\cZ$ is a \define{stabilization} of $\Xcal$ and $\Ycal$ if there exist structure-preserving maps $\Phi_A:\Omega_Z \to \Omega_A$ and $\varphi_A:Z \to A$ for $A \in \{X,Y\}$. If a stabilization exists, we say that $\Xcal$ and $\Ycal$ are \define{isomorphic}. In the case that $\Xcal$ and $\Ycal$ have the same parameter space $(\Omega,\nu)$, we say that $\Xcal$ and $\Ycal$ are \define{fixed-parameter isomorphic} if there is a stabilization $\mathcal{Z}$ with parameter space $(\Omega,\nu)$ such that the maps $\Phi_A$ are identity maps.
\end{defn}

\begin{remark}[Weak Isomorphism for Measure Networks]\label{rem:weak_isomorphism_measure_networks}
    Let $\Xcal = (X,\mu_X,\omega_X)$ and $\Ycal = (Y,\mu_Y,\omega_Y)$ be measure networks and let $\overline{\Xcal}$ and $\overline{\Ycal}$ denote their representations as pm-nets parameterized over the one point space $(\Omega,\nu)$ (cf.~\Cref{rem:recovers_GW_distance}). Then $\Xcal$ and $\Ycal$ are weakly isomorphic (\Cref{def:weakly_isomorphic}) if and only if their pm-net representations $\overline{\Xcal}$ and $\overline{\Ycal}$ are fixed-parameter isomorphic. 
\end{remark}

In light of \Cref{rem:weak_isomorphism_measure_networks}, the following generalizes \cite[Theorems 2.3 and 2.4]{ChowdhuryMemoli2019}, which characterize the pseudometric structure of $\GW_p$ on the space of measure networks.

\begin{theorem}
	\label{thm:isomorphism_of_pm_nets}
    For the standard examples $(\mathfrak{N},\mathsf{C})$, the associated parameterized GW distance $\GW_\mathsf{C}$ defines a pseudometric. Moreover, $\mathsf{GW}_\mathsf{C}(\Xcal,\Ycal) = 0$ if and only if: 
    \begin{enumerate}
        \item $\Xcal$ and $\Ycal$ are isomorphic, in the case $\mathfrak{N} = \mathfrak{N}_\mathrm{all}$;
        \item $\Xcal$ and $\Ycal$ are fixed-parameter isomorphic, in the case $\mathfrak{N} = \mathfrak{N}_\nu$.
    \end{enumerate}
\end{theorem}

\begin{proof}
    By \Cref{thm:metric_structure}, to show that $\GW_\mathsf{C}$ is a pseudometric, we need to show that $\mathsf{C}$ is symmetric, respects identities, and respects gluings. The first two properties are obvious, so we focus on the last.

    Let $\cX, \cY, \cZ \in \mathfrak{N}$. For ease of notation, let $\overline{\omega}_A^t := \omega_A^t \circ (p_A, p_A)$ for $t \in \Omega_A$ and $A \in \{X, Y, Z\}$, with $p_A$ denoting projection onto the $A$ factor of any product of spaces including $A$. Let $\pi_{XY} \in \coup(\mu_X, \mu_Y)$, $\pi_{YZ} \in \coup(\mu_Y, \mu_Z)$, $\pi_{XY} \bullet \pi_{YZ} \in \coup(\mu_X, \mu_Y)$ and $\tilde{\pi}$ as in  \Cref{def:props_cost_structures}. Then for all $r \in \Omega_X$, $s \in \Omega_Y$ and $t \in \Omega_Z$, the triangle inequality of the $L^p$ norm gives
    \begin{align}
    \label{eq:triangle_inequality_1}
    	\dis_p(\pi_{XY} \bullet \pi_{YZ}, \omega_X^r, \omega_Z^t)
    	&= \| \overline{\omega}_X^r - \overline{\omega}_Z^t \|_{L^p(X \times Z; \pi_{XY} \bullet \pi_{YZ})} \\
    	&= \| \overline{\omega}_X^r - \overline{\omega}_Z^t \|_{L^p(X \times Y \times Z; \tilde{\pi})} \nonumber\\
    	&\leq \| \overline{\omega}_{X}^r - \overline{\omega}_{Y}^s \|_{L^p(X \times Y \times Z; \tilde{\pi})} + \| \overline{\omega}_{Y}^s - \overline{\omega}_{Z}^t \|_{L^p(X \times Y \times Z; \tilde{\pi})} \nonumber\\
    	&= \| \overline{\omega}_{X}^r - \overline{\omega}_{Y}^s \|_{L^p(X \times Y; \pi_{XY})} + \| \overline{\omega}_{Y}^s - \overline{\omega}_{Z}^t \|_{L^p(Y \times Z; \pi_{YZ})} \nonumber\\
    	&= \dis_p(\pi_{XY}, \omega_X^r, \omega_Y^s) + \dis_p(\pi_{YZ}, \omega_Y^s, \omega_Z^t). \nonumber
    \end{align}
    Analogously, given $\xi_{XY} \in \coup(\nu_X, \nu_Y)$ and $\xi_{YZ} \in \coup(\nu_Y, \nu_Z)$, we get
    \begin{multline}
    	\label{eq:triangle_inequality}
        \| \dis_p(\pi_{XY} \bullet \pi_{YZ}, \omega_X, \omega_Z) \|_{L^q(\Omega_X \times \Omega_Z; \xi_{XY} \bullet \xi_{YZ})} \\
        \leq \| \dis_p(\pi_{XY}, \omega_X, \omega_Y) \|_{L^q(\Omega_X \times \Omega_Y; \xi_{XY})} + \| \dis_p(\pi_{YZ}, \omega_Y, \omega_Z) \|_{L^q(\Omega_Y \times \Omega_Z; \xi_{YZ})}.
    \end{multline}
    Recall that in the case of $\mathfrak{N}_\nu$, we have a common parameter space $(\Omega, \nu)$. If we set $r=s=t$ and $\xi_{XY} = \xi_{YZ} = (\mathrm{id}_{\Omega} \times \mathrm{id}_{\Omega})_{\#} \nu$, \Cref{eq:triangle_inequality} becomes
    \begin{align*}
    	\mathsf{C}_{\cX\cZ}(\pi_{XY} \bullet \pi_{YZ})
    	&= \| \dis_p(\pi_{XY} \bullet \pi_{YZ}, \omega_X^t, \omega_Z^t) \|_{L^q(\Omega; \nu)} \\
    	&\leq \| \dis_p(\pi_{XY}, \omega_X^t, \omega_Y^t) \|_{L^q(\Omega; \nu)} + \| \dis_p(\pi_{YZ}, \omega_Y^t, \omega_Z^t) \|_{L^q(\Omega; \nu)} \\
    	&= \mathsf{C}_{\cX\cY}(\pi_{XY}) + \mathsf{C}_{\cY\cZ}(\pi_{YZ}).
    \end{align*}
    In the case of $\mathfrak{N}_\mathrm{all}$, we infimize the right side of  \Cref{eq:triangle_inequality} over $\xi_{XY} \in \coup(\nu_X, \nu_Y)$ and $\xi_{YZ} \in \coup(\nu_Y, \nu_Z)$ to obtain
    \begin{align*}
    	\mathsf{C}_{\cX\cZ}(\pi_{XY} \bullet \pi_{YZ})
    	&= \inf_{\xi \in \coup(\nu_X, \nu_Z)} \| \dis_p(\pi_{XY} \bullet \pi_{YZ}, \omega_X, \omega_Z) \|_{L^q(\Omega_X \times \Omega_Z; \xi)} \\
    	&\leq \| \dis_p(\pi_{XY} \bullet \pi_{YZ}, \omega_X, \omega_Z) \|_{L^q(\Omega_X \times \Omega_Z; \xi_{XY} \bullet \xi_{YZ})} \\
    	&\leq \mathsf{C}_{\cX\cY}(\pi_{XY}) + \mathsf{C}_{\cY\cZ}(\pi_{YZ}).
    \end{align*}
    Hence, both cost structures in the standard examples respect gluings.
    
    It remains to characterize the distance zero conditions. First consider the case of $\mathfrak{N}_\mathrm{all}$. For $\Xcal,\mathcal{Z} \in \mathfrak{N}_\mathrm{all}$, suppose that there exist a structure-preserving pair of maps $\Phi:\Omega_Z \to \Omega_X$ and $\varphi:Z \to X$. We use these to construct couplings as pushforwards: 
    \[
    \xi = (\mathrm{id}_{\Omega_Z} \times \Phi)_\# \nu_Z \quad \mbox{and} \quad \pi = (\mathrm{id}_Z \times \varphi)_\# \mu_Z.
    \]
    It is then straightforward to verify, via the change-of-variables formula, that
    \[
    \mathsf{C}_{\Xcal,\mathcal{Z}}(\pi) \leq \|\mathrm{dis}_p(\pi,\omega_Z,\omega_X)\|_{L^q(\Omega_Z \times \Omega_X; \nu_Z \otimes \nu_X)} = 0,
    \]
    so that $\GW_\mathsf{C}(\mathcal{Z},\Xcal)=0$. Thus, if $\Xcal$ and $\Ycal$ are isomorphic pm-nets, it follows by the triangle inequality of $\mathsf{GW}_\mathsf{C}$ that $\GW_\mathsf{C}(\Xcal,\Ycal) = 0$. Conversely, suppose that $\mathsf{GW}_\mathsf{C}(\Xcal,\Ycal) = 0$. By \Cref{prop:optimal_couplings}, there exist couplings $\xi \in \coup(\nu_X, \nu_Y)$ and $\pi \in \coup(\mu_X, \mu_Y)$ such that
	\begin{equation}
		\label{eq:zero_distortion}
		\| \dis_p(\pi, \omega_X^s, \omega_Y^t) \|_{L^q(\Omega_X \times \Omega_Y; \xi)} = 0.
	\end{equation}
	Define $Z := X \times Y$, $\mu_Z := \pi$, $\Omega_Z := \Omega_X \times \Omega_Y$, and $\nu_Z := \xi$. Let $\varphi_A:Z \to A$ and $\Phi_A:\Omega_Z \to \Omega_A$ be the standard projections for $A \in \{X, Y\}$, and define $\omega_Z^t$ as the pullback $\varphi_X^*\omega_X^{\Phi(t)}$ for every $t \in \Omega_Z$. We claim that $\cZ := (Z, \mu_Z, \Omega_Z, \nu_Z, \omega_Z)$ is a stabilization of $\cX$ and $\cY$. The maps $\varphi_A$ and $\Phi_A$ are measure-preserving for both $A \in \{X, Y\}$ because of the marginal constraints of $\pi$ and $\xi$, respectively. In addition, $(\Phi_X, \varphi_X)$ satisfies condition (\Cref{item:equiv2}) of \Cref{def:equivalence_of_measure_networks} for every $(t,z,z') \in \Omega_Z \times Z \times Z$ by definition of $\omega_Z$. Hence, $(\Phi_X, \varphi_X)$ is structure-preserving. The fact that $(\Phi_Y, \varphi_Y)$ is structure-preserving follows from \Cref{eq:zero_distortion}, as
	\begin{align*}
		\int_{\Omega_Z} &\left(\int_{Z^2} |\varphi_X^* \omega_X^{\Phi_X(t)}(z,z') - \varphi_Y^* \omega_Y^{\Phi_Y(t)}(z,z')|^p \mu_Z(dz) \mu_Z(dz') \right)^{q/p} \nu_Z(dt) \\
		&= \int_{\Omega_X \times \Omega_Y} \left(\int_{(X \times Y)^2} |\omega_X^{r}(x,x') - \omega_Y^{s}(y,y')|^p \pi(dx \otimes dy) \pi(dx' \otimes dy') \right)^{q/p} \xi(dr \otimes ds) \\
		&= 0.
	\end{align*}
	The above implies $\omega_Z^t(z,z') = \varphi_X^* \omega_X^{\Phi_X(t)}(z,z') = \varphi_Y^* \omega_Y^{\Phi_Y(t)}(z,z')$ for $(\nu_Z \otimes \mu_Z \otimes \mu_Z)$-almost every $(t,z,z')$. Hence, $\cZ$ is a stabilization of $\cX$ and $\cY$, so $\cX$ and $\cY$ are isomorphic.
    
    Now consider the case of $\mathfrak{N}_\nu$. If $\Xcal$ and $\Ycal$ are fixed-parameter isomorphic, then constructions similar to those that were used above can be used to show that $\GW_\mathsf{C}(\Xcal,\Ycal) = 0$. On the other hand, suppose that $\mathsf{GW}_\mathsf{C}(\Xcal,\Ycal) = 0$. The proof that $\Xcal$ and $\Ycal$ are fixed-parameter isomorphic is similar to the proof above, except that we have $\Omega_X = \Omega_Y = \Omega$ and $\nu_X = \nu_Y = \nu$, so we define $\Omega_Z = \Omega$, $\nu_Z = \nu$, and $\Phi_X = \Phi_Y = \mathrm{id}_{\Omega}$. \Cref{prop:optimal_couplings} yields the existence of a coupling $\pi$ such that $\| \dis_p(\pi, \omega_X^t, \omega_Y^t) \|_{L^q(\Omega; \nu)} = 0$, and the above equation becomes
	\begin{align*}
		\int_{\Omega} &\left(\int_{Z^2} |\varphi_X^* \omega_X^{\Phi_X(t)}(z,z') - \varphi_Y^* \omega_Y^{\Phi_Y(t)}(z,z')|^p \mu_Z(dz) \mu_Z(dz') \right)^{q/p} \nu(dt) \\
		&= \int_{\Omega} \left(\int_{(X \times Y)^2} |\omega_X^{t}(x,x') - \omega_Y^{t}(y,y')|^p \pi(dx \otimes dy) \pi(dx' \otimes dy') \right)^{q/p} \nu(dt) = 0,
	\end{align*}
    so that $\mathcal{Z}$ gives the desired stabilization.
\end{proof}

\section{Comparisons and Approximations for Parametric Gromov-Wasserstein Distances}
\label{sec:results}

This section compares our parameterized GW distances for general pm-nets with existing distances proposed in the literature for specific classes of pm-nets. We also examine several approximation strategies for parameterized GW distances.

\subsection{Realization as a Z-Gromov-Wasserstein Distance}
\label{sec:ZGW}

In this subsection, we fix a parameter space $(\Omega,\nu)$ and consider the class of pm-nets $\mathfrak{N}_\nu$ from \Cref{ex:classes_of_pmnets}. We also fix the cost structure $\mathsf{C}$ from \Cref{ex:fixed_parameter_space}, for some choices of $p,q \in [1,\infty]$.
As noted in \Cref{rem:ZGW_distance}, the associated parameterized GW distance is related to the $Z$-Gromov–Wasserstein distance~\cite{bauer2024z}, and we now make this connection precise.

\subsubsection{$Z$-Gromov-Wasserstein Distances.} 
\label{sec:ZGW-distances}

We take a brief detour to recall the notion of $Z$-GW distance. Fix a complete metric space $(Z,d_Z)$ and consider a structure of the form $\Xcal = (X,\mu_X,\omega_X)$, where $(X,\mu_X)$ is a (say, compact) Polish measure space and $\omega_X:X \times X \to Z$ is a $Z$-valued kernel, which we assume to be bounded (more generally, \cite{bauer2024z} allows $p$-integrable kernels, for a given $p$). Such a structure is referred to in \cite{bauer2024z} as a \define{$Z$-network}. For $p \in [1,\infty]$, the \define{$p$-Gromov Wasserstein distance} between two $Z$-networks $\Xcal$ and $\Ycal$ is
\[
\mathsf{GW}_p^Z(\Xcal,\Ycal) \coloneqq \frac{1}{2} \inf_{\pi \in \mathcal{C}(\mu_X,\mu_Y)} \|d_Z \circ (\omega_X, \omega_Y) \|_{L^p((X \times Y)^2; \pi \otimes \pi)},
\]
where the function in the norm is given by
\begin{align*}
    d_Z \circ (\omega_X, \omega_Y): X \times Y \times X \times Y &\to \R \\
    (x,y,x',y') &\mapsto d_Z\big(\omega_X(x,x),\omega_Y(y,y')\big).
\end{align*}

For context, note that when $(Z, d_Z)$ is the real line equipped with its standard metric, the $Z$-GW framework reduces to the classical GW distance on measure networks. 

\subsubsection{Parameterized GW Distances as $Z$-GW Distances.} 
\label{sec:PGW-ZGW}

We now return to the setting of interest, to understand the parameterized GW distance on $\mathfrak{N}_\nu$ as a $Z$-GW distance. To this end, set $(Z,d_Z)$ to be $(L^q(\Omega;\nu),\|\cdot\|_{L^q(\Omega;\nu)})$, where we are abusing notation and using $\|\cdot\|_{L^q(\Omega;\nu)}$ as a placeholder for its induced metric. Let $\Xcal = (X,\mu_X,\Omega,\nu,\omega_X) \in \mathfrak{N}_\nu$. This can be understood as a $Z$-network $\overline{\Xcal} = (\overline{X},\overline{\mu}_X,\overline{\omega}_X)$, where $\overline{X} = X$, $\overline{\mu}_X = \mu_X$, and $\overline{\omega}_X$ is the $Z$-valued kernel defined, for all $x,x' \in X$, $t \in \Omega_X$, $\omega^t_X(x,x') \in \R$ by 
\[
\overline{\omega}_X(x,x') \coloneqq (t \mapsto \omega^t_X(x,x')) \in L^q(\Omega;\nu).
\]

This leads to the following comparison result for parameterized GW distance and $Z$-GW distance, where notation is the same as above:

\begin{prop}\label{prop:Z-GW}
    For any $\Xcal,\Ycal \in \mathfrak{N}_\nu$, if $p \square q$ then $\mathsf{GW}_\mathsf{C}(\Xcal,\Ycal) \square \mathsf{GW}_p^Z(\overline{\Xcal},\overline{\Ycal})$, where $\square \in \{\leq, \geq, =\}$. 
\end{prop}

\begin{proof}
    Assuming $p \leq q$, we have 
\begin{align*}
    2\cdot \mathsf{GW}_\mathsf{C}(\Xcal,\Ycal) &= \inf_{\pi \in \mathcal{C}(\mu_X,\mu_Y)} \|\mathrm{dis}_p(\pi,\omega_X,\omega_Y)\|_{L^q(\Omega,\nu)} \\
    &= \inf_{\pi \in \mathcal{C}(\mu_X,\mu_Y)} \left\|
    \|\omega_X \circ (p_X,p_X) - \omega_Y \circ (p_Y,p_Y)\|_{L^p((X \times Y)^2; \pi \otimes \pi)}
    \right\|_{L^q(\Omega,\nu)}\\
    &\leq \inf_{\pi \in \mathcal{C}(\mu_X,\mu_Y)} \left\|
    \|\omega_X \circ (p_X,p_X) - \omega_Y \circ (p_Y,p_Y)\|_{L^q(\Omega,\nu)} 
    \right\|_{L^p((X \times Y)^2; \pi \otimes \pi)} \\
    &= \inf_{\pi \in \mathcal{C}(\mu_X,\mu_Y)} \left\|
    d_Z \circ (\overline{\omega}_X,\overline{\omega}_Y) 
    \right\|_{L^p((X \times Y)^2; \pi \otimes \pi)} \\
    &= 2\cdot \mathsf{GW}_p^Z(\overline{\Xcal},\overline{\Ycal}),
\end{align*}
    where we have applied a generalized version of Minkowski's inequality \cite[Proposition 1.3]{bahouri2011fourier}. The case $p \geq q$ follows by a similar argument and these together imply the $p=q$ case. 
\end{proof} 

\subsubsection{Followup Remarks.} We conclude this subsection with some remarks on the connection described above.
\begin{enumerate}[leftmargin=*]
    \item Theoretical properties of $\mathsf{GW}^Z_p$ are derived for general choices of $(Z,d_Z)$ in~\cite{bauer2024z}. These properties therefore apply to $\mathsf{GW}_\mathsf{C}$ on $\mathfrak{N}_\nu$ in the case where $p=q$: the induced metric space (after modding out by fixed-parameter isomorphism; see \Cref{thm:isomorphism_of_pm_nets}) is complete, contractible and geodesic, for example.
    \item One could instead define a cost structure $\mathsf{C}$ on $\mathfrak{N}_\nu$ by integrating in a different order; that is, for $p,q < \infty$,  
    \[
    \mathsf{C}_{\Xcal,\Ycal}(\pi) \coloneqq \frac{1}{2}\left(\int_{(X \times Y)^2} \left( \int_\Omega |\omega_X(x,x') - \omega_Y(y,y')|^q \pi(dx \times dy) \pi(dx' \times dy')\right)^{p/q} \nu(dt)\right)^{1/p}
    \]
    In this case, one has $\mathsf{GW}_\mathsf{C} = \mathsf{GW}^Z_p$ (with $(Z,d_Z)$ defined as above) for all choices of $p,q$. The rationale for our particular choice of cost structure on $\mathfrak{N}_\nu$ is that it more naturally generalizes to the cost structure on the collection of all pm-nets $\mathfrak{N}_\mathrm{all}$ defined in \Cref{ex:general_parameter_spaces}.
    \item The parameterized GW distance defined on $\mathfrak{N}_\mathrm{all}$ via \Cref{ex:general_parameter_spaces} is not realized as an instance of a $Z$-GW distance (in any obvious way). 
\end{enumerate}

\subsection{Metrics for Spaces Parameterized by Real Numbers}
\label{sec:real_numbers}

The pm-nets described in \Cref{ex:time_varying} through \Cref{ex:graph_heat_kernels} are all defined over a parameter space $(\Omega,\nu)$ consisting of a compact set of real numbers, endowed with some measure. In this subsection, we compare our approach to existing metrics in the literature, which were designed to compare specialized pm-net structures. 

\subsubsection{Time-Varying Metric Spaces.} Fix a parameter space $(\Omega,\nu)$ consisting of a compact interval of real numbers endowed with some probability measure (e.g., $\Omega = [0,1]$, $\nu$ is Lebesgue measure) and consider the class of pm-nets $\mathfrak{N}_\mathrm{tvm}$ consisting of \emph{time-varying metric measure spaces} (\Cref{ex:time_varying}); that is, an element of $\mathfrak{N}_\mathrm{tvm}$ is a pm-net of the form $\Xcal = (X,\mu_X,\Omega,\nu,d_X)$, where, for each $t \in \Omega$, $d_X^t$ is a metric inducing the given topology on $X$. Observe that this is a proper subclass of $\mathfrak{N}_\nu$, due to the constraint that the kernel is a metric for each parameter value. We now provide examples of distances between objects of $\mathfrak{N}_\mathrm{tvm}$ which have been previously introduced in the literature, and compare them to our approach.

\begin{example}[Sturm's Distance]
\label{ex:Sturm}
In~\cite{sturm2018super}, Sturm introduced a distance on a more general class of objects of the form $\Xcal = \big(X,(\mu_X^t)_{t \in \Omega},(d_X^t)_{t \in \Omega}\big)$, where 
\begin{itemize}[leftmargin=*]
\item $X$ is a Polish space; 
\item Each $d_X^t$ is a metric generating the topology of $X$; Sturm also assumed that each $d_X^t$ is geodesic, but this property is not important for our discussion; 
\item $\mu_X^t$ is a time-varying family of Borel measures which are absolutely continuous with respect to some reference probability measure $\mu_X$.
\end{itemize}
We refer to such a structure as a \emph{generalized time-varying metric measure space}. Clearly, this concepts restricts to our notion of time-varying metric space by imposing the constraint that $\mu_X^t = \mu_X$ for all $t$.

Given two generalized time-varying metric measure spaces $\Xcal$ and $\Ycal$, Sturm defines a distance between them as 
\begin{equation}\label{eqn:sturm_distance}
\mathsf{D}_{\mathrm{St}}(\Xcal,\Ycal) \coloneqq \inf_{d_{X \sqcup Y}, \pi} \left( \int_\Omega \int_{X \times Y} d_{X \sqcup Y}^t(x,y)^2 \pi(dx \otimes dy) \nu(dt)\right)^{1/2} + \int_\Omega \int_{X \times Y} |f_X^t(x) - f_Y^t(y)| \pi(dx \otimes dy) \nu(dt),
\end{equation}
where the infimum is over 
\begin{itemize}[leftmargin=*]
    \item $\pi \in \mathcal{C}(\mu_X,\mu_Y)$, i.e., couplings of the reference measures, 
    \item $(d_{X \sqcup Y}^t)_{t \in \Omega}$ is a parameterized family of \emph{metric couplings}, or metrics $d_{X \sqcup Y}^t$ on the disjoint union which restrict to $d_X^t$ and $d_Y^t$, respectively, on the appropriate subsets, for almost every $t \in \Omega$, and
    \item $(f_X^t)_{t \in \Omega}$ is chosen so that $\mu_X^t = e^{f_X^t} \mu_X$ for all $t$, and similarly for $(f_Y^t)_{t \in \Omega}$.
\end{itemize}
Restricting to the subspace $\mathfrak{N}_{\mathrm{tvm}}$, the second term of \Cref{eqn:sturm_distance} vanishes and \Cref{eqn:sturm_distance} further simplifies to 
\[
\mathsf{D}_{\mathrm{St}}(\Xcal,\Ycal) = \inf_{d_{X \sqcup Y}, \pi} \left( \int_\Omega \int_{X \times Y} d_{X \sqcup Y}^t(x,y)^2 \pi(dx \otimes dy) \nu(dt)\right)^{1/2} = \left\|\inf_{d_{X \sqcup Y}} \mathsf{W}_2^{d_{X \sqcup Y}^t}(\mu_X,\mu_Y)\right\|_{L^2(\Omega; \nu)},
\]
that is, an integrated version of Sturm's well-known \emph{$L^2$-transportation distance on the space of metric measure spaces}~\cite{sturm2006geometry} (we abuse notation here, and consider $\mu_X$ and $\mu_Y$ as measures on the disjoint union, in the obvious way). The transportation distance is known to upper bound (classical) GW distance, and the proof in~\cite{memoli2011} can be extended to show 
\[
\GW_\mathsf{C}(\Xcal,\Ycal) \leq \mathsf{D}_{\mathrm{St}}(\Xcal,\Ycal),
\]
where the cost structure $\mathsf{C}$ comes from \Cref{ex:fixed_parameter_space}, with $p=q=2$.
\end{example}

\begin{example}[Integrated Gromov-Hausdorff Distance]
A simple metric on $\mathfrak{N}_{\mathrm{tvm}}$, which has primarily been used in the topological data analysis (TDA) literature~\cite{munch2013applications,xian2022capturing}, is the \define{integrated Gromov-Hausdorff distance}, defined as 
\[
\mathsf{IGH}(\Xcal,\Ycal) \coloneqq \int_{\Omega} \mathsf{GH}\big((X,d_X^t),(Y,d_Y^t) \big) \nu(dt),
\]
where $\mathsf{GH}$ is the standard Gromov-Hausdorff distance between compact metric spaces. Arguments presented in \cite{memoli2011} can be adapted to show that 
\[
\GW_\mathsf{C}(\Xcal,\Ycal) \leq \mathsf{IGH}(\Xcal,\Ycal),
\]
where $\mathsf{C}$ is as in \Cref{ex:fixed_parameter_space} with $p=\infty$ and $q=1$. 
\end{example}

\begin{example}[Slack Interleaving Distance]\label{ex:Kim-Memoli}
An alternative metric on $\mathfrak{N}_\mathrm{tvm}$, based on constructions used in TDA, was in introduced in~\cite{kim2020persistent}. We review the details here. Given continuous functions $f,g:\R\to \R_{\geq 0}$ and $\lambda \geq 0$, the \define{$\lambda$-slack interleaving distance} between the functions is defined by
\[
d_{\lambda}(f_1,f_2)\coloneqq\inf\left\{\varepsilon\in[0,\infty] \mid \forall t\in\R, \min_{s\in[t-\varepsilon,t+\varepsilon]} f_i(s)\leq f_j(t)+\lambda\varepsilon,\,i,j=1,2\right\}.
\]
For $p \in [1,\infty]$, the \define{$(p,\lambda)$-Gromov-Wasserstein distance} between $\Xcal,\Ycal \in \mathfrak{N}_\mathrm{tvm}$ is defined as
\[
\inf_{\pi\in\mathcal{C}(\mu_X,\mu_Y)}\big\|d_\lambda \circ (d_X \times d_Y)\big\|_{L^p(\pi\otimes \pi)}.
\]
Taking the cost structure
\[
\mathsf{C}_{\Xcal,\Ycal}(\pi) = \big\|d_\lambda \circ (d_X \times d_Y)\big\|_{L^p(\pi\otimes \pi)}, 
\]
we see that the $(p,\lambda)$-GW distance is an instance of a parameterized GW distance. We note that it was shown in~\cite[Proposition 3.13]{bauer2024z} that the $(p,\lambda)$-GW distance can also be realized as a $Z$-GW distance.
\end{example}

\subsubsection{Heat Kernels.} We consider classes of pm-nets induced by heat kernels. For the rest of this subsection, we temporarily abuse terminology and drop the assumptions that the parameter space in a pm-net is compact, and that the measure on this space is a probability measure. This is only for the sake of avoiding technical details; in this setting, the basic definitions introduced in the paper related to parameterized GW distances are still valid, but certain theoretical properties are no longer guaranteed. As we are only concerned here with estimates of the parameterized GW distances, this issue is not a concern.

First, consider the class $\mathfrak{N}_\mathrm{HK}$ of pm-nets of the form $\Xcal = (X,\mu_X,\Omega,\nu,\omega_X)$, where
\begin{itemize}[leftmargin=*]
    \item $\Omega = \Rspace_{>0}\coloneqq(0,+\infty)$ (with its usual topology, hence a non-compact space) and $\nu$ is Lebesgue measure (not a probability measure),
    \item $X$ is a compact Riemannian manifold endowed with normalized Riemannian volume $\mu_X$, and
    \item $\omega_X^t:X \times X \to \R$ is the normalized heat kernel of $X$ (with respect to the given Riemannian metric), where \emph{normalized} means that $\lim_{t \to \infty} \omega^t(x,x') = 1$ for all $x,x' \in X$. 
\end{itemize}

\begin{example}[Spectral Gromov-Wasserstein distance]
	\label{ex:Memoli-HK}
    M{\'e}moli defined a distance $\mathsf{GW}_p^\mathrm{spec}$ to compare two Riemannian manifolds using their heat kernels~\cite{memoli2009spectral,memoli2011spectral}; in our terminology, this is a distance on the class $\mathfrak{N}_\mathrm{HK}$ defined above. Here, we recall the definition of  M\'{e}moli's distance and rewrite it as a parametrized GW distance. To recall the original definition, we first define an auxiliary function $c(t) \coloneqq e^{-(t + t^{-1})}$. Given $\Xcal,\Ycal \in \mathfrak{N}_\mathrm{HK}$, we define a cost
    \begin{equation*}
    	\Gamma_{X, Y, t}^\text{spec}(x,y, x',y') \coloneqq | \omega_X^t(x,x') - \omega_Y^t(y,y')|
    \end{equation*}
    and \[
    \displaystyle \mathsf{GW}_p^\mathrm{spec}(X, Y) \coloneqq \inf_{\pi \in \coup(\operatorname{Vol}_X, \operatorname{Vol}_Y)} \sup_{t > 0} c^2(t) \cdot \| \Gamma_{X, Y, t}^\text{spec}\|_{L^p((X \times Y)^2; \pi \otimes \pi)},\] 
    where $p \in [1,\infty]$.
    
    Given $p, q \in [1, \infty]$, we define a new cost structure $\mathsf{C}$ by
    \begin{equation*}
    	\mathsf{C}_{\cX, \cY}(\pi) \coloneqq \left\| c^2(t)  \cdot \| \Gamma_{X, Y, t}^\text{spec}\|_{L^p((X \times Y)^2, \pi \otimes \pi)} \right\|_{L^q(\Omega, \nu)}
    \end{equation*}
    for some $\pi \in \coup(\operatorname{Vol}_X, \operatorname{Vol}_Y)$. 
    Then $\displaystyle \mathsf{GW}_p^\mathrm{spec}(X, Y) = \GW_{\mathsf{C}}(\cX, \cY)$ when $q = \infty$. We note that a similar interpretation of $\mathsf{GW}_p^\mathrm{spec}$ as a $Z$-GW distance was provided in~\cite[Proposition 3.12]{bauer2024z}, but that this result only gives an upper bound due to issues similar to those that arose in the proof of \Cref{prop:Z-GW}.
\end{example}

Next we consider the class $\mathfrak{N}_\mathrm{GHK}$ of \emph{graph heat kernels} $\Xcal = (X,\mu_X,\Omega,\nu,\omega_X)$, where $(\Omega,\nu)$ is as above, $X$ is the set of nodes of a graph, endowed with some probability measure $\mu_X$, and $\omega_X^t$ is the graph heat kernel at scale $t$. 

\begin{example}[Graph Heat Kernels]
\label{ex:Chowdhury-HK}
Chowdhury and Needham~\cite{spectralGW} studied GW distances between graph heat kernels at a fixed scale parameter $t$. Their approach contrasts with that of the present paper, which aims to incorporate information across all scales when comparing pm-nets. In \cite{spectralGW}, the scale parameter was treated as a tunable hyperparameter. From a high-level perspective, this can be viewed as computing $\mathsf{GW}_\mathsf{C}$ while allowing additional flexibility in the choice of the measure $\nu$; for example, permitting it to collapse to a Dirac measure at a single scale $t$. This viewpoint aligns with the feature selection approach described in \cref{sec:feature-selection}.
\end{example}

\subsection{Approximation by Wasserstein Distance}

In this subsection, we derive a lower bound on the parameterized GW distance by a certain Wasserstein distance. 

\subsubsection{Wasserstein Distance Over the GW Space.} We begin by setting up necessary preliminary concepts. Let $\mathfrak{M}$ denote the class of measure networks (by \Cref{ex:measure_networks}, $\mathfrak{M}$ is equivalent to the class $\mathfrak{N}_\nu$, where $(\Omega,\nu)$ is a one-point space, but we introduce this additional notation for bookkeeping purposes). For the rest of this subsection, we let $\sim$ denote the weak isomorphism equivalence relation on $\mathfrak{M}$ (\Cref{def:weakly_isomorphic}); for a measure network $\Xcal$, we let $[\Xcal]$ denote its equivalence class, and we let $\faktor{\mathfrak{M}}{\sim}$ denote the set of equivalence classes, that is, measure networks, considered up to weak isomorphism. 

Given $p \in [1, \infty]$, let $\GW_p$ be the $p$-GW distance on $\mathfrak{M}$. Considering elements of $\mathfrak{M}$ up to weak isomorphism, $\GW_p$ induces a complete and separable metric on $\faktor{\mathfrak{M}}{\sim}$: completeness follows essentially from~\cite[Theorem 5.8]{Sturm2023}, which enforces an additional symmetry condition on the kernels that is not intrinsically necessary for the proof; separability is a standard argument, but a precise reference is \cite[Proposition 4.8]{bauer2024z}. We abuse notation and continue to denote the induced metric by $\GW_p$. Therefore $(\faktor{\mathfrak{M}}{\sim},\GW_p)$ is a Polish metric space, and measure theory on it is sufficiently well-behaved to consider the Wasserstein distance between distributions defined over it. Throughout the rest of this subsection, we use $\W_q^{\mathsf{GW}_p}$ to denote the $q$-Wasserstein distance on $(\faktor{\mathfrak{M}}{\sim},\GW_p)$ (see \Cref{sec:Wasserstein}). 

\subsubsection{Lower Bound on Parameterized GW Distance.}
\label{subsec:lower_bound} 
Now consider an arbitrary pm-net $\cX \in \mathfrak{N}_\mathrm{all}$. For any $t \in \Omega_X$,  the triplet $\cX_t \coloneqq (X, \mu_X, \omega_X^t)$ is a measure network. We define the map
\begin{align*}
    m_\Xcal: \Omega_X &\to \faktor{\mathfrak{M}}{\sim} \\
    t &\mapsto [\cX_t],
\end{align*}
and the associated pushforward measure $\overline{\nu}_X \coloneqq (m_\Xcal)_\# \nu_X$ on $\faktor{\mathfrak{M}}{\sim}$. For this to be well-defined, we require $m_\Xcal$ to be measurable, which is verified by the following \Cref{lem:mX_continuous}.

\begin{lemma}
\label{lem:mX_continuous}
    The map $m_\Xcal$ is continuous. 
\end{lemma}

\begin{proof}
    We factor $m_\Xcal$ as $m_\Xcal = q \circ \hat{m}_\Xcal$, where $\hat{m}_\Xcal:\Omega_X \to \mathfrak{M}$ is the map $t \mapsto \Xcal_t$ and $q:\mathfrak{M} \to \faktor{\mathfrak{M}}{\sim}$ is the quotient map. It suffices to prove that $\hat{m}_\Xcal$ is continuous, with respect to the topology induced by the pseudometric $\GW_p$. Let us metrize $\Omega_X$ via some choice $d_{\Omega_X,p}$; the particular metric is irrelevant, and we only assume that it induces the  given Polish topology on $\Omega_X$. Let $\epsilon > 0$. As parameterized network kernels are assumed to be $L^\infty$ continuous, there exists $\delta > 0$ such that $d_{\Omega_X,p}(s,t) < \delta$ implies $\|\omega_X^s - \omega_X^t\|_{L^\infty(X \times X; \mu_X \otimes \mu_X)} < \epsilon$. We have 
    \[
    \GW_p(\cX_s,\cX_t) \leq \GW_\infty(\cX_s,\cX_t) \leq \|\omega_X^s - \omega_X^t\|_{L^\infty(X \times X; \mu_X \otimes \mu_X)}.
    \]
    Here, the first inequality follows from \cite[Theorem 5.1 (h)]{memoli2011}; that result proves the inequality in the case of metric measure spaces, but the proof is based on properties of $L^p$-norms and applies to more general measure networks. The second inequality follows by considering the $\infty$-distortion of the identity coupling of $\mu_X$ with itself.
\end{proof}
We are now able to state the main result of this subsection.
\begin{theorem}\label{thm:wasserstein_estimate}
	Let $p, q \in [1, \infty]$ and let $\mathsf{C}$ be the cost structure from  \Cref{ex:general_parameter_spaces}. For $\Xcal, \Ycal \in \mathfrak{N}_\mathrm{all}$, we have
	\begin{equation*}
		\W_q^{\mathsf{GW}_p}(\overline{\nu}_X, \overline{\nu}_Y) \leq \GW_{\mathsf{C}}(\Xcal ,\Ycal).
	\end{equation*}
\end{theorem}

The proof will use the following general measure theory result.

\begin{lemma}\label{lem:pushforward_lemma}
Let $(X,\mu_X)$ and $(Y,\mu_Y)$ be Polish probability spaces, $X'$ and $Y'$ Polish spaces, and $f:X \to X'$ and $g:Y \to Y'$ measurable maps. Then
\[
\mathcal{C}(f_\# \mu_X, g_\# \mu_Y) = \{(f \times g)_\# \pi \mid \pi \in \mathcal{C}(\mu_X,\mu_Y)\},
\]
where $f \times g :X \times Y \to X' \times Y'$ is the product map $(x,y) \mapsto (f(x),g(y))$. 
\end{lemma}

\begin{proof}
One inclusion is straightforward: it is easy to show that a measure of the form $(f \times g)_\# \pi$ is a coupling of $f_\# \mu_X$ and $g_\# \mu_Y$, for $\pi \in \mathcal{C}(\mu_X,\mu_Y)$.  

To prove the  remaining inclusion, let $\{\mu_X(\cdot \mid x')\}_{x' \in X'}$ denote the disintegration of $\mu_X$ with respect to $f$. That is, for each $x' \in X'$, $\mu_X(\cdot \mid x')$ is a Borel probability measure on $X$ satisfying 
\[
\mu_X(A) = \int_{X'} \mu_X(A\mid x') f_\# \mu_X(dx') \quad \mbox{and} \quad f_\# \mu_X(\cdot \mid x') = \delta_{x'},
\]
where $\delta_{x'}$ denotes the Dirac measure on $X'$ (see, e.g., \cite[Section 5.3]{ambrosio2005gradient}). Likewise, let $\{\mu_Y(\cdot \mid y')\}_{y' \in Y'}$ denote the disintegration of $\mu_Y$ with respect to $g$.

Given $\xi \in \mathcal{C}(f_\# \mu_X, g_\# \mu_Y)$, define a measure $\pi$ on $X \times Y$  for a product Borel set $A \times B$ as 
\[
\pi(A \times B) = \int_{X' \times Y'} \mu_X(A \mid x') \mu_Y(B \mid y') \xi(dx' \otimes dy').
\]
We claim that $\pi \in \mathcal{C}(\mu_X,\mu_Y)$ and that $(f \times g)_\# \pi = \xi$. The first point follows by checking marginals:
\begin{align*}
\pi(A \times Y) &= \int_{X' \times Y'} \mu_X(A \mid x') \mu_Y(Y \mid y') \xi(dx' \otimes dy') \\ 
&= \int_{X' \times Y'} \mu_X(A \mid x') \xi(dx' \otimes dy') \\ 
&= \int_{X'} \mu_X(A \mid x') f_\# \mu_X(dx') = \mu_X(A),
\end{align*}
where we have used that $\mu_Y(\cdot\mid y')$ is a probability measure and the marginal condition on $\xi$. The fact that $\pi(X \times B) = \mu_Y(B)$ follows similarly.

Finally, we prove that $(f \times g)_\# \pi = \xi$. We will use the following identity, which holds for any $x' \in X'$, and which follows by the property that $f_\# \mu_X(\cdot \mid x') = \delta_{x'}$:
\begin{equation}\label{eqn:indicator_equality}
\int_{f^{-1}(A')} \mu_X(dx \mid x') = 1_{A'}(x'),
\end{equation}
where $1_{A'}$ denotes the indicator function for a Borel set $A'$. Proceeding with the proof, let $A' \times B'$ be a product Borel set in $X' \times Y'$. Then
\begin{align}
    (f \times g)_\# \pi (A' \times B') &= \int_{A' \times B'} (f \times g)_\# \pi(dx' \otimes dy') \nonumber \\
    &= \int_{f^{-1}(A') \times g^{-1}(B')} \pi (dx \otimes dy) \nonumber \\
    &= \int_{X' \times Y'} \int_{g^{-1}(B')} \int_{f^{-1}(A')} \mu_X(dx\mid x') \mu_Y(dy \mid y') \xi(dx' \otimes dy')  \label{eqn:pushforward_lemma_1} \\
    &= \int_{X' \times Y'} \int_{g^{-1}(B')} 1_{A'}(x') \mu_Y(dy \mid y') \xi(dx' \otimes dy')  \label{eqn:pushforward_lemma_2} \\
    &= \int_{X' \times Y'}  1_{A'}(x') 1_{B'}(y') \xi(dx' \otimes dy') \label{eqn:pushforward_lemma_3}  \\
    &= \int_{A' \times B'} \xi(dx' \otimes dy') = \xi(A' \times B'), \nonumber
\end{align}
where \Cref{eqn:pushforward_lemma_1} follows from the definition of $\pi$ and Fubini's Theorem, and \Cref{eqn:pushforward_lemma_2} and \Cref{eqn:pushforward_lemma_3} both follow by applying \Cref{eqn:indicator_equality} to the iterated integrals. 
\end{proof}

\begin{proof}[Proof of \Cref{thm:wasserstein_estimate}]
	First consider the $q < \infty$ case ($p \in [1,\infty]$ is arbitrary).
    For any $\pi \in \coup(\mu_X, \mu_Y)$, we have
	\begin{align}
		\W_q^{\mathsf{GW}_p}(\overline{\nu}_X, \overline{\nu}_Y)
		&= \inf_{\overline{\xi} \in \coup(\overline{\nu}_X, \overline{\nu}_Y)} \left( \int_{\mathfrak{M}\times \mathfrak{M}} \GW_p([\Xcal], [\Ycal])^q \, \overline{\xi}(d[\Xcal] \otimes d[\Ycal]) \right)^{1/q} \nonumber \\
        &= \inf_{\xi \in \coup(\nu_X, \nu_Y)} \left( \int_{\mathfrak{M}\times \mathfrak{M}} \GW_p([\Xcal], [\Ycal])^q \, (m_\Xcal \times m_\Ycal)_\# \xi (d[\Xcal] \otimes d[\Ycal]) \right)^{1/q} \label{eqn:wasserstein_bound_1} \\
        &= \inf_{\xi \in \coup(\nu_X, \nu_Y)} \left( \int_{\Omega_X \times \Omega_Y} \GW_p(\Xcal_s, \Ycal_t)^q \, \xi (ds \otimes dt) \right)^{1/q} \label{eqn:wasserstein_bound_2} \\
		&\leq \inf_{\xi \in \coup(\nu_X, \nu_Y)} \left( \int_{\Omega_X \times \Omega_Y} \frac{1}{2} \dis_p(\pi, \omega_X^s, \omega_Y^t)^q \, \xi(ds \otimes dt) \right)^{1/q} \nonumber,
	\end{align}
	where \Cref{eqn:wasserstein_bound_1} follows from \Cref{lem:pushforward_lemma} and \Cref{eqn:wasserstein_bound_2} follows from a change of variables, and the fact that $\mathsf{GW}_p$ is invariant over equivalence classes. The result then follows by infimizing over $\pi \in \coup(\mu_X, \mu_Y)$ on the right-hand side. 
    
    Finally, consider the $q=\infty$ case. The work above shows that for any $\pi \in \coup(\mu_X, \mu_Y)$ and any $\xi \in \mathcal{C}(\nu_X,\nu_Y)$,
    \[
    \W_\infty^{\mathsf{GW}_p}(\overline{\nu}_X, \overline{\nu}_Y) = \lim_{q \to \infty}  \W_q^{\mathsf{GW}_p}(\overline{\nu}_X, \overline{\nu}_Y) \leq \lim_{q \to \infty} \frac{1}{2} \|\mathrm{dis}_p(\pi,\omega_X,\omega_Y)\|_{L^q(\xi)} = \frac{1}{2} \|\mathrm{dis}_p(\pi,\omega_X,\omega_Y)\|_{L^\infty(\xi)},
    \]
    where the first equality follows from \cite[Proposition 3]{givens1984class}. The result once again follows by infimizing the right-hand side. 
\end{proof}

\subsubsection{Lower Bound by Weight Distribution}
\label{sec:weight-distribution}

We now utilize \Cref{thm:wasserstein_estimate} to give further lower bounds on parameterized GW distances. These are given in terms of a new invariant of pm-nets, defined below.

\begin{defn}[Weight Distribution]
\label{defn:weight-distribution}
Consider the function (see \cite{memoli2011,ChowdhuryMemoli2019}) 
\begin{equation}
\label{eqn:global_distribution}
\begin{split}
    \faktor{\mathfrak{M}}{\sim} &\to \Pcal(\R) \\
    [\Ycal] &\mapsto (\omega_Y)_\# (\mu_Y \otimes \mu_Y).
\end{split}
\end{equation}
Let $\cX$ be a pm-net and let $\overline{\nu}_X \in \Pcal(\R)$ be as in \Cref{thm:wasserstein_estimate} (\Cref{subsec:lower_bound}). The \define{weight distribution of $\cX$} is the measure $\Delta_\cX \in \Pcal\big(\Pcal(\R)\big)$ given by the pushforward of $\overline{\nu}_X$ by the map \Cref{eqn:global_distribution}.
\end{defn}

In~\Cref{defn:weight-distribution}, the weight distribution defines an invariant of pm-nets which takes the form of a probability distribution over the space of probability distributions, denoted as $\Pcal\big(\Pcal(\R)\big)$. Such an invariant may appear to be rather unwieldy, but we show below that it is stable and, moreover, that it is easy to compute in certain circumstances.

\begin{corollary}
\label{cor:global_distribution_stability}
The weight distribution is a stable invariant of pm-nets. That is, let $\cX$ and $\cY$ be pm-nets, let $\overline{\nu}_X$ and $\overline{\nu}_Y$ be as in \Cref{thm:wasserstein_estimate}, and let $\mathsf{C}$ be the cost structure from \Cref{ex:general_parameter_spaces}, for some $p,q \in [1,\infty]$. Then
\begin{equation}
\label{eqn:global_distribution_stability}
\W_q^{\W_p^{\Pcal(\R)}}(\Delta_\Xcal, \Delta_\Ycal) \leq 2 \cdot \mathsf{GW}_\mathsf{C}(\Xcal,\Ycal).
\end{equation}
\end{corollary}
In \Cref{eqn:global_distribution_stability}, the superscript in $\W_q^{\W_p^{\Pcal(\R)}}$ indicates that the Wasserstein distance $\W_q$ is taken over the metric space $(\Pcal(\R), \W_p^{\Pcal(\R)})$; that is, it is a Wasserstein distance on the Wasserstein space.
The proof will use a lemma, whose proof follows directly from (the easy direction of) \Cref{lem:pushforward_lemma}. 

\begin{lemma}\label{lem:induced_Lipschitz}
    Let $f:X \to Y$ be a $k$-Lipschitz map between Polish metric spaces $(X,d_X)$ and $(Y,d_Y)$. Then the pushforward $f_\#:\Pcal(X) \to \Pcal(Y)$ is a $k$-Lipschitz map with respect to $\W_q^{d_X}$ and $\W_q^{d_Y}$.
\end{lemma}

\begin{proof}[Proof of \Cref{cor:global_distribution_stability}]
    It is shown in \cite[Theorem 3.1]{ChowdhuryMemoli2019} that the map \Cref{eqn:global_distribution} is 2-Lipschitz, with respect to $\GW_p$ and $\W_p^{\Pcal(\R)}$. By \Cref{lem:induced_Lipschitz}, the associated pushforward map $\Pcal\big(\faktor{\mathfrak{M}}{\sim}\big) \to \Pcal\big(\Pcal(\R)\big)$ is 2-Lipschitz with respect to $\W_q^{\mathsf{GW}_p}$ and $\W_q^{\mathsf{W}^{\Pcal(\R)}_p}$. From \Cref{thm:wasserstein_estimate}, we have 
    \[
   \W_q^{\W_p^{\Pcal(\R)}}(\Delta_\Xcal, \Delta_\Ycal) \leq 2 \cdot \W_q^{\mathsf{GW}_p}(\overline{\nu}_X,\overline{\nu}_Y) \leq 2 \cdot \mathsf{GW}_\mathsf{C}(\Xcal,\Ycal).
    \]
\end{proof}

\begin{remark}[Computation]\label{rem:computation}
For pm-nets $\cX$ and $\cY$ defined over finite sets $X$ and $Y$ and finite parameter spaces $\Omega_X$ and $\Omega_Y$, respectively, the right-hand-side of \Cref{eqn:global_distribution_stability} involves computing Wasserstein distance between finitely-supported distributions on the Wasserstein space 
$\left(\Pcal\left(\Pcal\left(\R\right)\right), \W_p^{\Pcal(\R)}\right)$.
It is therefore polynomial-time computable (in the magnitudes of $X, Y, \Omega_X, \Omega_Y$). This gives a tractable lower estimate of the parameterized GW distance.
\end{remark}

\subsubsection{Weight Distributions for Random Graph Models}
\label{subsubsec:weight_distributions_random_graph_models}

We now specialize the results of \Cref{sec:weight-distribution} to the setting of random graphs. Consider a random graph model $\Xcal$, in the sense of \Cref{ex:random_graphs}, and suppose that $|X| = n$ and that $\mu_X$ is uniform.  For the sake of simplifying the discussion, choose an ordering of $X$ and consider the kernels $\omega_X^t$ arising in this model as symmetric, binary,  $n\times n$ matrices with zeros on their diagonals (recall that we use adjacency kernels in this example). In this case, the distribution $\overline{\nu}_X \in \Pcal\big(\faktor{\mathfrak{M}}{\sim}\big)$ can be considered as a discrete probability measure on the finite set of such matrices (the set has cardinality $n(n-1)/2$). Given such a matrix $\omega_X^t$, the distribution $(\omega_X^t)_\# (\mu_X \otimes \mu_X)$ arising from \Cref{eqn:global_distribution} is supported on $\{0,1\}$. Indeed, the weight on the point $1$ is exactly $k/n^2$, where $k$ is the number of non-zero entries appearing in the matrix; that is, $k$ is twice the number of edges appearing in the graph $t$. In light of this interpretation, the information contained in $(\omega_X^t)_\# (\mu_X \otimes \mu_X)$ is exactly the total number of edges in the graph $t$, and we refer to the weight distribution $\Delta_\cX$ in this case as the \emph{distribution of total edges}. 

The discussion above immediately leads to the following specialization of \Cref{cor:global_distribution_stability}, stated here somewhat informally.

\begin{corollary}\label{cor:global_distribution_stability_graphs}
The distribution of total edges is a stable invariant of a random graph model.
\end{corollary}

\begin{example}[Erd\H{o}s-R\'{e}nyi Model]\label{rem:erdos_renyi_model}
For a concrete example, consider an Erd\H{o}s-R\'{e}nyi random graph model $\cX$ on $n$ nodes with probability $\rho \in [0,1]$ that any pair of nodes is connected. A straightforward calculation shows that the distribution of total edges is given by 
\[
\Delta_\cX = \sum_{k=0}^N \binom{N}{k} \rho^k (1-\rho)^{N-k} \delta_{\frac{k}{n^2}},
\]
where $N = n(n-1)/2$, and $\delta_{\frac{k}{n^2}}$ is a shorthand for the Dirac mass on $\Pcal(\R)$ located at the distribution supported on $\{0,1\}$ with weight $k/(n^2)$ at $1$ and $1-k/(n^2)$ at $0$. This recovers the well-known fact that the total number of edges in an Erd\H{o}s-R\'{e}nyi graph is binomially distributed. 
\end{example}

\begin{remark}[Computation for Random Graph Models]\label{rem:computation_random_graph_models}
    When $\Xcal$ is a random graph model, in particular, with each $\omega_X^t$ taking values only in $\{0,1\}$, the distribution of total edges $\Delta_\cX$ is especially easy to work with from a computational perspective. In this setting, the cost matrix used in the computation of $\W_q^{\W_1^{\Pcal(\R)}}(\Delta_\cX,\Delta_\cY)$ involves $1$-Wasserstein distances between the measures $(\omega_X^t)_\#(\mu_X \otimes \mu_X)$ and $(\omega_Y^t)_\#(\mu_Y \otimes \mu_Y)$. Each of these measures is supported on $\{0,1\}$, and it is not hard to show (via the semi-explicit formula for Wasserstein distances on the real line~\cite[Remark 2.19]{villani2021topics}) that this is given by 
    \[
    |(\omega_X^t)_\#(\mu_X \otimes \mu_X)(\{1\}) - (\omega_Y^t)_\#(\mu_Y \otimes \mu_Y))(\{1\})|.
    \]
    This makes the lower bound of \Cref{cor:global_distribution_stability} (or \Cref{cor:global_distribution_stability_graphs}) very efficient to work with in the random graph setting; see \Cref{subsec:clustering_random_graphs} for a numerical example.
\end{remark}

\subsection{Sample Approximation}
\label{sec:approximation_by_samples}
In this subsection, we consider random graphs and random metric spaces, as introduced in~\Cref{ex:random_graphs} and~\Cref{ex:random_metric_spaces}. As previously noted, one typically does not have access to the full distribution over the parameter space, but only to i.i.d. samples of measure networks drawn from $\nu_X$. Accordingly, this subsection focuses on convergence results for this sampling process. Throughout, we employ the cost structure $\mathsf{C}$ from~\Cref{ex:general_parameter_spaces}.

\subsubsection{Random Measure Networks}

We begin by focusing on a broadened version of the random metric space model of \Cref{ex:random_metric_spaces}.

\begin{defn}[Random Measure Network Model]\label{def:random_measure_network_model}
A \define{random measure network model} is a pm-net $\cX$ with the property that $\Omega_X \subset L^\infty(X \times X; \mu_X \otimes \mu_X)$ and $\omega_X^t = t$. We view 
$\Omega_X$ as a metric space equipped with the metric $d_{\Omega_X,p}$ induced by the $L^p$-norm $\|\cdot\|_{L^p{(X \times X; \mu_X \otimes \mu_X)}}$.
\end{defn}

Let $\cX$ be a random measure network model. We define the \define{empirical pm-net} 
\[
\cX_{T_N} \coloneqq (X, \mu_X, T_N, \nu_N, \omega_N)
\]
 by sampling $N$ kernels $T_N = \{t_1, \dots, t_N\}$ i.i.d. according to $\nu_X$ and setting $\nu_N \coloneqq \sum_{i=1}^N \frac{1}{N} \delta_{t_i}$ and $\omega_N^{t_i} = t_i$. For any $t \in \Omega_X$, let $\cX_t$ be the trivial pm-net from \Cref{ex:measure_networks} (alternatively, this is equivalent to the associated measure network, utilized in \Cref{subsec:lower_bound}). 
 
 We begin with some basic lemmas. 
 
\begin{lemma}
\label{lemma:one_point_coupling}
Let $(A,\alpha)$ and $(B,\beta)$ be probability spaces such that $|B| = 1$ (i.e.,~$B$ is a one-point set). Then $\coup(\alpha, \beta)$ consists of a single element $\xi = (p_A^{-1})_\# \alpha$ that satisfies $\int_{A \times B} F(a,b) \ \xi(da \otimes db) = \mathbb{E}_\alpha[ F(\bullet, b)] \coloneqq \int_{A} F(a, b) \ \alpha(da)$ for any measurable map $F: A \times B \to \R$. 
\end{lemma}
\begin{proof}
	Let $\xi \in \coup(\alpha, \beta)$. Since $B$ is a one-point set, the projection $p_A:A \times B \to A$ is a bijection, so $(p_A)_\# \xi = \alpha$ forces $\xi = (p_A^{-1})_\# \alpha$. Then $\int_{A \times B} F(a,b) \ \xi(da \otimes db) = \int_{A \times B} F(a,b) \ (p_A^{-1})_\#\alpha(da \otimes db) = \int_{A} F(p_A^{-1}(a)) \ \alpha(da) = \int_{A} F(a, b) \ \alpha(da)$.
\end{proof}

\begin{lemma}
	\label{lemma:MS_one_parameter}
	For $p \in [1, \infty]$ and $q \in [1, \infty)$,
		$
		\displaystyle
		\GW_{\mathsf{C}}(\cX, \cX_t) = \inf_{\pi \in \coup(\mu_X, \mu_X)} \mathbb{E}_{\nu_X}[\dis_p(\pi, \bullet, t)^q]^{1/q}.
		$
\end{lemma}
\begin{proof}
	This follows by setting $A = \Omega_X$, $\alpha = \nu_X$, $B = \{t\}$ and $F_\pi(s, t) = \dis_p(\pi, s, t)^q$ in \Cref{lemma:one_point_coupling}.
\end{proof}

This leads to our first sampling convergence result.

\begin{prop}
	\label{prop:approx_one_point_pm_net}
	For $p \in [1, \infty]$ and $q \in [1, \infty)$, $\GW_{\mathsf{C}}(\cX_{T_N}, \cX_t) \to \GW_{\mathsf{C}}(\cX, \cX_t)$ almost surely as $N \to \infty$.
\end{prop}
\begin{proof}
	Given $\pi \in \coup(\mu_X, \mu_X)$, define $F_\pi:T_N \to \R_{\geq 0}$ by $F_\pi(t_i) \coloneqq \dis_p(\pi, t_i, t)$. By the Strong Law of Large Numbers, $\mathbb{E}_{\nu_N}[F_\pi^q] \to \mathbb{E}_{\nu_X}[F_\pi^q]$ almost surely. Taking infimum over $\coup(\mu_X, \mu_X)$ and using \Cref{lemma:one_point_coupling} yields
	\begin{equation*}
		\GW_{\mathsf{C}}(\cX_{T_N}, \cX_t) = \inf_{\pi \in \coup(\mu_X, \mu_Y)} \mathbb{E}_{\nu_N}[F_\pi^q(s)]^{1/q} \to \inf_{\pi \in \coup(\mu_X, \mu_Y)} \mathbb{E}_{\nu_X}[F_\pi^q(s)]^{1/q} = \GW_{\mathsf{C}}(\cX, \cX_t)
	\end{equation*}
	almost surely.
\end{proof}

\begin{remark}[$\W_q^{d_{\Omega_X,p}}$ metrizes weak convergence on $\Pcal(\Omega_X)$]
	\label{rmk:finite_moments}
	Let $p \in [1, \infty]$ and $q \in [1, \infty)$. \cite[Theorem 6.9]{Villani2009} states that $\W_q^{d_{\Omega_X,p}}$ parametrizes weak convergence on the set of measures $\nu_X \in \Pcal(\Omega_X)$ that have finite $q$ moment. However, any $\nu_X \in \Pcal(\Omega_X)$ has finite moments of all orders by compactness of $\Omega_X$ because if $D_p$ is the diameter of $(\Omega_X, d_{\Omega_X,p})$ and $q < \infty$,
	\begin{equation*}
		\int_{\Omega_X} \| t - t_0 \|_{\Omega_X, p}^q \ \nu_X(dt) \leq \int_{\Omega_X} D_p^q \ \nu_X(dt) = D_p^q < \infty.
	\end{equation*}
	Hence, $\W_q^{d_{\Omega_X,p}}$ parametrizes weak convergence on $\Pcal(\Omega_X)$ for any $1 \leq q < \infty$.
\end{remark}

Hence, we can state our next sampling convergence result for all measures $\nu_X \in \Pcal(\Omega_X)$.

\begin{prop}
	\label{prop:approx_random_metrics}
	Let $p \in [1, \infty]$ and $q \in [1, \infty)$. For any $\nu_X \in \Pcal(\Omega_X)$, $\GW_{\mathsf{C}}(\cX_{T_N}, \cX) \leq \W_q^{d_{\Omega_X,p}}(\nu_N, \nu_X)$ and, thus, $\GW_{\mathsf{C}}(\cX_{T_N}, \cX) \to 0$ almost surely as $N \to \infty$.
\end{prop}
\begin{proof}
	Let $\Delta \coloneqq (\mathrm{id}_X \times \mathrm{id}_X)_\# \mu_X$ be the diagonal coupling in $\coup(\mu_X, \mu_X)$. For $p < \infty$, we have
	\begin{align*}
		\dis_p(\Delta, t_i, t)
		&= \left( \int_{(X \times X)^2} |t_i(x,x') - t(y,y')|^p \ \Delta(dx \otimes dy) \Delta(dx' \otimes dy') \right)^{1/p}\\
		&= \left( \int_{X \times X} |t_i(x,x') - t(x,x')|^p \ \mu_X(dx) \mu_X(dx') \right)^{1/p}
		= d_{\Omega_X,p}(t_i,t),
	\end{align*}
	while 
	\[
	\displaystyle \dis_\infty(\Delta, t_i, t) = \sup_{(x,x'),(y,y') \in \supp(\Delta)} |t_i(x,y) - t(x',y')| = \sup_{x,x' \in X} |t_i(x,x') - t(x,x')| = d_{\Omega_X, \infty}(t_i,t).
	\]
	 Then
	\begin{align*}
		\GW_{\mathsf{C}}(\cX_{T_N}, \cX)^q
		&= \inf_{\xi \in \coup(\nu_{N}, \nu_X)} \inf_{\pi \in \coup(\mu_X, \mu_X)} \int_{T_N \times \Omega_X} \dis_p(\pi, t_i, t)^q \ \xi(dt_i \otimes dt) \\
		&\leq \inf_{\xi \in \coup(\nu_{N}, \nu_X)} \int_{T_N \times \Omega_X} \dis_p(\Delta, t_i, t)^q \ \xi(dt_i \otimes dt) \\
		&= \inf_{\xi \in \coup(\nu_{N}, \nu_X)} \int_{T_N \times \Omega_X} \|t_i - t\|_{\Omega_X, p}^q \ \xi(dt_i \otimes dt) \\
		&= \W_q^{d_{\Omega_X,p}}(\nu_N, \nu_X)^q.
	\end{align*}
	This yields the first claim. Since $\W_q^{d_{\Omega_X,p}}$ metrizes weak convergence on $\Pcal(\Omega_X)$ by \Cref{rmk:finite_moments}, and the empirical measures $\nu_N$ converge (set-wise, hence weakly) to $\nu_X$ almost surely, $\W_q^{d_{\Omega_X,p}}(\nu_N, \nu_X) \to 0$ almost surely as $N \to \infty$.
\end{proof}

This quickly leads to our main sampling convergence result.

\begin{theorem}
	\label{thm:approx_random_metrics}
	Let $p \in [1, \infty]$ and $q \in [1, \infty)$. Let $\cX$ and $\cY$ be random measure network models, as in \Cref{def:random_measure_network_model}. Let $\cX_{T_N}$ and $\cY_{S_N}$ be their respective empirical pm-nets. Then
	\begin{equation*}
		|\GW_{\mathsf{C}}(\cX_{T_N}, \cY_{S_N}) - \GW_{\mathsf{C}}(\cX, \cY)|
		\leq
		\GW_{\mathsf{C}}(\cX_{T_N}, \cX) + \GW_{\mathsf{C}}(\cY_{S_N}, \cY),
	\end{equation*}
	and thus, $\GW_{\mathsf{C}}(\cX_{T_N}, \cY_{S_N}) \to \GW_{\mathsf{C}}(\cX, \cY)$ almost surely as $N \to \infty$.
\end{theorem}

\begin{proof}
	By the triangle inequality,
	\begin{align*}
		\GW_{\mathsf{C}}(\cX_{T_N}, \cY_{S_N})
		&\leq
		\GW_{\mathsf{C}}(\cX_{T_N}, \cX)
		+ \GW_{\mathsf{C}}(\cX, \cY)
		+ \GW_{\mathsf{C}}(\cY, \cY_{S_N}), \text{ and} \\
		\GW_{\mathsf{C}}(\cX, \cY)
		& \leq
		\GW_{\mathsf{C}}(\cX, \cX_{T_N})
		+ \GW_{\mathsf{C}}(\cX_{T_N}, \cY_{S_N})
		+ \GW_{\mathsf{C}}(\cY_{S_N}, \cY).
	\end{align*}
	Hence,
	\[
		|\GW_{\mathsf{C}}(\cX_{T_N}, \cY_{S_N}) - \GW_{\mathsf{C}}(\cX, \cY)|
		\leq
		\GW_{\mathsf{C}}(\cX_{T_N}, \cX) + \GW_{\mathsf{C}}(\cY_{S_N}, \cY),
	\]
	and the right-hand-side converges almost surely to zero by  \Cref{prop:approx_random_metrics}.
\end{proof}

From \Cref{prop:approx_random_metrics}, the triangle inequality for Wasserstein distances, the argument in the proof of \Cref{thm:approx_random_metrics}, and \Cref{cor:global_distribution_stability}, we deduce the following corollary.

\begin{corollary}
	With the same notation and setup as \Cref{thm:approx_random_metrics}, the weight distributions satisfy
	\[
	\W_q^{\W_p^{\Pcal(\R)}}(\Delta_{\cX_{T_N}},\Delta_{\cX}) \to 0 \quad \mbox{and} \quad 
	\W_q^{\W_p^{\Pcal(\R)}}(\Delta_{\cX_{T_N}},\Delta_{\cY_{S_N}}) \to \W_q^{\W_p^{\Pcal(\R)}}(\Delta_{\cX},\Delta_{\cY})
	\]
	almost surely as $N \to \infty$.
\end{corollary}

\indent \Cref{prop:approx_random_metrics} implies that $\GW_{\mathsf{C}}(\cX_{T_N}, \cX)$ converges towards 0 no faster than $\W_q^{d_{\Omega_X,p}}(\nu_N, \nu_X)$ does. Likewise, the convergence of $\GW_{\mathsf{C}}(\cX_{T_N}, \cY_{S_N})$ towards $\GW_{\mathsf{C}}(\cX, \cY)$ in  \Cref{thm:approx_random_metrics} is dominated by terms of the form $\GW_{\mathsf{C}}(\cX_{T_N}, \cX)$ and thus, by Wasserstein distances. Thanks to existing results on rates of convergence of Wasserstein distances, we can quantify the rate of convergence in~\Cref{prop:approx_random_metrics} and \Cref{thm:approx_random_metrics}. Although the results we cite hold for measures on $\R^d$ and $\nu_X$ is defined on $\Omega_X$, the elements of $\Omega_X$ are network functions which can be embedded into $\R^d$ for some fixed $d$. We just need to restrict to finite pm-nets in order for $d < \infty$. Following \cite{Fournier2023}, we use the notation $M_{\R^d,r}(\mu) \coloneqq \int_{\R^d} \|v\|^r \ \mu(dv)$ where $\mu$ is a measure and $\| \cdot \|$ is a norm, both defined on $\R^d$. Below, we use $\|\cdot\|_{\Omega_X,p}$ as short hand for the $L^p$-norm, restricted to $\Omega_X$. 

\begin{prop}
	\label{prop:convergence_rates_empirical}
	Let $p \in [1, \infty]$ and $q \in [1, \infty)$. Let $\cX$ and $\cY$ be random measure network models with $n \coloneqq |X|$, $m \coloneqq |Y|$ and $m \leq n < \infty$. Let $\cX_{T_N}$ and $\cY_{S_N}$ be their respective empirical pm-nets for some $T_N \subset \Omega_X$ and $S_N \subset \Omega_Y$ with $|T_N| = |S_N| = N$. Define $D_X \coloneqq \sup_{t \in \Omega_X} \|t\|_{\Omega_X, p}$ and $D_Y \coloneqq \sup_{t \in \Omega_Y} \|t\|_{\Omega_Y, p}$. Let $d \coloneqq n^2$. Then there exist constants $C = C(p,q,n)$ and $C' = C'(p,q,n,m)$ such that
	\begin{equation*}
		\mathbb{E} \left[ \GW_{\mathsf{C}}(\cX_{T_N}, \cX)^q \right]
		\leq C D_X^q \cdot
		\begin{cases}
			N^{-1/2} & \text{ if } q > d/2,\\
			N^{-1/2} \log(1+N) & \text{ if } q = d/2,\\
			N^{-q/d} & \text{ if } 0 < q < d/2.
		\end{cases}
	\end{equation*}
	and
	\begin{align*}
		\mathbb{E} \left[\left| \GW_{\mathsf{C}}(\cX_{T_N}, \cY_{S_N}) - \GW_{\mathsf{C}}(\cX, \cY) \right|\right]
		&\leq C' (D_X + D_Y) \cdot 
		\begin{cases}
			N^{-1/{2q}} & \text{ if } q > d/2,\\
			N^{-1/{2q}} \log(1+N)^{1/q} & \text{ if } q = d/2,\\
			N^{-1/d} & \text{ if } 0 < q < d/2.
		\end{cases}
	\end{align*}
	If $\cX$ and $\cY$ are random metric space models instead (see  \Cref{ex:random_metric_spaces}), then the same bounds hold with $d = n(n-1)/2$ instead of $d = n^2$.
\end{prop}
\begin{proof}
	Given a labelling $X = \{x_1, x_2, \dots, x_n\}$, any function $f \in L^p(X \times X, \mu_X \otimes \mu_X)$ is represented by a matrix $M_f \in \R^{n \times n}$ defined by $(M_f)_{ij} = f(x_i, x_j)$. We define the linear function $\Phi_{n \times n}: L^p(X \times X, \mu_X \otimes \mu_X) \to \R^{n \times n}$ by sending $f$ to $M_f$. Since $X$ is finite, any $f:X \times X \to \R$ belongs to $L^p(X \times X, \mu_X \otimes \mu_X)$, so $\Phi_{n \times n}$ has an inverse given by $\Phi_{n \times n}^{-1}(M)(x_i, x_j) = M_{ij}$ for all $M \in \R^{n \times n}$. This makes $\Phi_{n \times n}$ into an isomorphism of vector spaces.\\
	\indent If $\cX$ and $\cY$ are random metric space models, let $L^p_{\text{Sym}}$ be the linear subspace of $L^p(X \times X, \mu_X \otimes \mu_X)$ of symmetric functions with 0 diagonal. As the dimension of $L^p_{\text{Sym}}$ is $n(n-1)/2$,  define $\Phi_{\text{Sym}}: L^p_{\text{Sym}} \to \R^{n(n-1)/2}$ to be the coordinate map of $L^p_{\text{Sym}}$ with respect to the standard basis of $\R^{n(n-1)/2}$. As before, $\Phi_{\text{Sym}}$ is an isomorphism of vector spaces.\\
	\indent The rest of the proof proceeds analogously for both random measure network and metric space models, so we will fix the notation $L^p_X$ and $\Phi$ to mean $L^p(X \times X, \mu_X \otimes \mu_X)$ and $\Phi_{n \times n}$ if $\cX$ and $\cY$ are random measure network models and $L^p_{\text{Sym}}$ and $\Phi_{\text{Sym}}$ if $\cX$ and $\cY$ are random metric space models instead. Let $\| \cdot \|_{\Omega_X, p}$ be the $L^p$ norm on $L^p_X$. Since $\Phi$ is an isomorphism, the function $\| v \|_{\R^d,p} \coloneqq \| \Phi^{-1}(v) \|_{\Omega_X, p}$ defines a norm on $\R^d$. Recall that $\Omega_X \subset L^p_X$ by definition, so any measure $\nu_X \in \Pcal(\Omega_X)$ extends to a measure on $L^p_X$. Hence, $\Phi_\#\nu_X$ is a measure on $\R^d$ with support $\Phi(\Omega_X)$ such that
	\begin{equation*}
		M_{\R^d,r}(\Phi_\# \nu_X)
		= \int_{\R^d} \|v\|_{\R^d,p}^r \ \Phi_\#\nu_X(dv)
		= \int_{\Omega_X} \| \Phi(t) \|_{\R^d,p}^r \ \nu_X(dt)
		= \int_{\Omega_X} \| t \|_{\Omega_X,p}^r \ \nu_X(dt).
	\end{equation*}
	Moreover, $\sup_{r>0} M_{\R^d,r}(\Phi_\# \nu_X)^{1/r} = \sup_{t \in \Omega_X} \|t\|_{\Omega_X,p} = D_X$. For any other measure $\nu_X' \in \Pcal(\Omega_X) \subset \Pcal(L^p_X)$, we have, for $\W_q$ denoting the Wasserstein distance over the appropriate Euclidean space,
	\begin{align*}
		\W_q(\Phi_\# \nu_X, \Phi_\# \nu_X')^q
		 &= \inf_{\xi \in \coup(\Phi_\# \nu_X, \Phi_\# \nu_X')} \int_{\R^d \times \R^d} \| v - v' \|_{\R^d,p}^q \ \xi(dv \otimes dv') \\
		 &= \inf_{\xi' \in \coup(\nu_X, \nu_X')} \int_{\R^d \times \R^d} \| v - v' \|_{\R^d,p}^q \ (\Phi \times \Phi)_\# \xi'(dv \otimes dv') \\
		 &= \inf_{\xi' \in \coup(\nu_X, \nu_X')} \int_{L^p_X \times L^p_X} \| \Phi(t) - \Phi(t') \|_{\R^d,p}^q \ \xi'(dt \otimes dt') \\
		 &= \inf_{\xi' \in \coup(\nu_X, \nu_X')} \int_{L^p_X \times L^p_X} \| t - t' \|_{\Omega_X,p}^q \ \xi'(dt \otimes dt') \\
		 &= \W_q^{d_{\Omega_X,p}}(\nu_X, \nu_X')^q.
	\end{align*}
	We used \Cref{lem:pushforward_lemma} in the second line and the fact that the support of $\xi'$ is $\Omega_X \times \Omega_X$ in the last.\\
	\indent Now that we have pushed our measures into $\R^d$, we can use the bounds of Fournier~\cite{Fournier2023}. Let $\nu_N$ be the empirical measure defined by $T_N$. By  \Cref{prop:approx_random_metrics},
	\begin{equation*}
		\GW_{\mathsf{C}}(\cX_{T_N}, \cX)^q \leq \W_q^{d_{\Omega_X,p}}(\nu_N, \nu_X)^q = \W_q(\Phi_\# \nu_N, \Phi_\# \nu_X)^q,
	\end{equation*}
	and thus, $\mathbb{E} \left[ \GW_{\mathsf{C}}(\cX_{T_N}, \cX)^q \right] \leq \mathbb{E}\left[ \W_q(\Phi_\# \nu_N, \Phi_\# \nu_X)^q \right]$. If $q \neq d/2$, we apply \cite[Theorem 2.1]{Fournier2023} (replacing their $p,q,m$ with $q,r,p$ respectively) to get
	\begin{equation*}
		\mathbb{E}\left[ \W_q(\Phi_\# \nu_N, \Phi_\# \nu_X)^q \right]
		\leq 2^q \kappa_{d,q}^{(p)} [M_r^{(p)}(\Phi_\# \nu_X)]^{q/r} \theta_{d,q,r}^{(p)} \cdot N^{-e}
	\end{equation*}
	for some functions $\kappa_{d,q}^{(p)}$ and $\theta_{d,q,r}^{(p)}$; see~\cite[Theorem 2.1]{Fournier2023} for their closed-form expressions. The value of $e$ depends on $q$ and $d$: $e = 1/2$ if $q > d/2$ and $e=q/d$ otherwise. Fournier noted in \cite[Section 2.4]{Fournier2023} that $\theta_{d,q,r}^{(p)}$ is a decreasing function with $\lim_{r \to \infty} \theta_{d,q,r}^{(p)} = 1$, so together with $M_{\R^d,r}(\Phi_\# \nu_X)^{1/r} \leq D_X$, we remove the dependency on $r$ from the inequality above:
	\begin{align*}
		\mathbb{E}\left[ \W_q(\Phi_\# \nu_N, \Phi_\# \nu_X)^q \right]
		&\leq \inf_{r > 0} 2^q \kappa_{d,q}^{(p)} [M_r^{(p)}(\Phi_\# \nu_X)]^{q/r} \theta_{d,q,r}^{(p)} \cdot N^{-e} \\
		& \leq \inf_{r > 0} 2^q \kappa_{d,q}^{(p)} D_X^q \theta_{d,q,r}^{(p)} \cdot N^{-e} \\
		&= 2^q \kappa_{d,q}^{(p)} D_X^q \cdot N^{-e}.
	\end{align*}
	Hence, we obtain the first claim with $C(p,q,n) \coloneqq 2^q \kappa_{d,q}^{(p)} D_X^q$ if $q \neq d/2$. If $q = d/2$, we get an analogous bound with $e=1/2$, but the functions $\kappa_{d,p,N}^{(p)}$ and $\theta_{d,q,r,N}^{(p)}$ depend on $N$. However, $\theta_{d,q,r,N}^{(p)}$ is still decreasing in $r$ and has $\lim_{r \to \infty} \theta_{d,q,r,N}^{(p)} = 1$, while $\kappa_{d,p,N}^{(p)} = O(\ln(1+N))$. The result follows as above.\\
	\indent For the second claim, we embed $\Omega_Y$ in $\R^d$ instead of the smaller space $\R^{m \times m}$ (or $\R^{m(m-1)/2}$ for random metric space models) as the rates are dominated by the convergence in $\R^d$ anyways. Then by \Cref{thm:approx_random_metrics} and Jensen's inequality $\mathbb{E}[X]^q \leq \mathbb{E}[X^q]$ for $q \geq 1$, we get
	\begin{align*}
		\mathbb{E} \left| \GW_{\mathsf{C}}(\cX_{T_N}, \cY_{S_N}) - \GW_{\mathsf{C}}(\cX, \cY) \right|
		&\leq \mathbb{E} \left[ \GW_{\mathsf{C}}(\cX_{T_N}, \cX) \right] + \mathbb{E} \left[ \GW_{\mathsf{C}}(\cY_{S_N}, \cY) \right] \\
		&\leq \mathbb{E} \left[ \GW_{\mathsf{C}}(\cX_{T_N}, \cX)^q \right]^{1/q} + \mathbb{E} \left[ \GW_{\mathsf{C}}(\cY_{S_N}, \cY)^q \right]^{1/q}.
	\end{align*}
	The bound on $\mathbb{E} \left| \GW_{\mathsf{C}}(\cX_{T_N}, \cY_{S_N}) - \GW_{\mathsf{C}}(\cX, \cY) \right|$ follows by applying the first claim to each term.
\end{proof}

Since $p=q=2$ is a common choice in the upcoming experiments, we specialize the result above to these values.

\begin{corollary}
	\label{cor:convergence_rates_empirical_networks}
    With the same notation as \Cref{prop:approx_random_metrics}, but with the specialization $p=q=2$ and $n > 2$, we have
    \begin{align*}
        &\mathbb{E} \left[ \GW_{\mathsf{C}}(\cX_{T_N}, \cX)^2 \right] \leq C D_X^2 \cdot N^{-2/n^2}, \text{ and}\\
        &\mathbb{E} \left| \GW_{\mathsf{C}}(\cX_{T_N}, \cY_{S_N}) - \GW_{\mathsf{C}}(\cX, \cY) \right| \leq C' (D_X + D_Y) N^{-1/n^2}
    \end{align*}
\end{corollary}
\begin{proof}
	These bounds are obtained by simplifying the conclusion of  \Cref{prop:convergence_rates_empirical} with $d=n^2$ and $q=2$. For example, the conditions $q > d/2$, $q = d/2$ and $0 < q < d/2$ become $n^2 < 4$, $n^2 = 4$ and $4 < n^2$, which in turn simplify to $n < 2$, $n = 2$ and $n > 2$, respectively. We only retain the $n>2$ case.
\end{proof}

Once again, random metric space models have slightly better convergence rates.
\begin{corollary}
	\label{cor:convergence_rates_empirical_metrics}
	Under the same assumptions as \Cref{prop:approx_random_metrics}, except that $\cX$ and $\cY$ are random metric space models, $d = n(n-1)/2$, $p=q=2$ and $n \geq 4$, we have
    \begin{align*}
        &\mathbb{E} \left[ \GW_{\mathsf{C}}(\cX_{T_N}, \cX)^2 \right] \leq C D_X^2 \cdot N^{-2/d}, \text{ and}\\
        &\mathbb{E} \left| \GW_{\mathsf{C}}(\cX_{T_N}, \cY_{S_N}) - \GW_{\mathsf{C}}(\cX, \cY) \right| \leq C' (D_X + D_Y) \cdot N^{-1/d}.
    \end{align*}
\end{corollary}
\begin{proof}
	Note that with $q=2$, the conditions $q > d/2$ and $0 < q < d/2$ become $n(n-1)/2 < 4$ and $4 < n(n-1)/2$, which are equivalent to $n \leq 3$ and $n \geq 4$. Once again, the result follows by simplifying \Cref{prop:approx_random_metrics}.
\end{proof}

\begin{remark}
	\label{rmk:convergence_empirical}
	The convergence rates for the quantity $\mathbb{E} \left| \GW_{\mathsf{C}}(\cX_{T_N}, \cY_{S_N}) - \GW_{\mathsf{C}}(\cX, \cY) \right|$ in the previous propositions originate from the rates for $\W_q^{d_{\Omega_X,p}}(\nu_X, \nu_N)$. One may wonder if there exists another estimator of $\GW_{\mathsf{C}}(\cX, \cY)$ that has better convergence rates. However, \cite[Chapter 3]{convergence_empirical_thesis} proves that there exists no estimator of $\W_q(\mu_X, \mu_Y)$ that improves the convergence by more than a logarithmic factor. We have no reason to believe that we can improve the situation for $\GW_{\mathsf{C}}(\cX, \cY)$.
\end{remark}

\begin{remark}
	The convergence results of this section would not hold if we replaced parametrized GW distances with the average GW distance. Let $\hat{\omega}_{N} \coloneqq \frac{1}{N} \sum_{i=1}^N t_i$ and let $\hat{\Xcal}_N$ be the measure network $(X, \mu_X, \hat{\omega}_{N})$. To condense notation, we use $\| \cdot \|_p$ to denote the $L^p$ norm in $L^p\left( (X \times X)^2; \mu_X \otimes \mu_X \right)$ and denote the coordinate projections as $p_1, p_2:X \times X \to X$. By the triangle inequality,
	\begin{align*}
		\dis_p(\pi, \hat{\omega}_N, t)
		&= \| \hat{\omega}_N \circ (p_1, p_1) - t \circ (p_2, p_2) \|_p
		= \left\| \sum_{i=1}^{N} \frac{1}{N} [ t_i \circ (p_1, p_1) - t \circ (p_2, p_2) ] \right\|_p \\
		&\leq \sum_{i=1}^{N} \frac{1}{N} \| t_i \circ (p_1, p_1) - t \circ (p_2, p_2) \|_p
		= \frac{1}{N} \sum_{i=1}^{N} \dis_p(\pi, t_i, t).
	\end{align*}
	Hence, infimizing over $\coup(\mu_X, \mu_Y)$ yields $\displaystyle \GW_p(\hat{\cX}_N, \cX_t) \leq \inf_{\pi \in \coup(\mu_X, \mu_X)} \frac{1}{N} \sum_{i=1}^{N} \dis_p(\pi, t_i, t)$.\\
	\indent However, $\inf f + \inf g \leq \inf (f + g)$, so
	\begin{equation*}
		\frac{1}{N} \sum_{i=1}^{N} \GW_p(X_{t_i}, X_{t}) \leq \inf_{\pi \in \coup(\mu_X, \mu_X)} \frac{1}{N} \sum_{i=1}^{N} \dis_p(\pi, t_i, t).
	\end{equation*}
	Note that $\GW_p(\hat{\cX}_N, \cX_t)$ and $\frac{1}{N} \sum_{i=1}^{N} \GW_p(X_{t_i}, X_{t})$ are not comparable in these inequalities, so even if the average GW distance converges, we can say nothing about $\GW_p(\hat{\cX}_N, \cX_t)$.
\end{remark}

\subsubsection{Extending Beyond Random Measure Network Models}
The results of the previous subsection apply specifically to random measure network models. In particular, the convergence result does not apply directly to random graph models, as formulated in \Cref{ex:random_graphs}. Clearly, the set of graphs over a node set $X$ is in bijective equivalence with the set of adjacency kernels $X \times X \to \{0,1\}$, so that one can trivially reformulate any random graph model as a  random measure network model. Our convergence theorem therefore easily translates to this setting. We record this, somewhat informally, as a corollary. 

\begin{corollary}
	Let $\cX$ and $\cY$ be random graph models and let $\cX_{T_N}$ and $\cY_{S_N}$ be their respective empirical pm-nets. Then $\GW_{\mathsf{C}}(\cX_{T_N}, \cY_{S_N}) \to \GW_{\mathsf{C}}(\cX, \cY)$ almost surely as $N \to \infty$. 
\end{corollary}

Moreover, we make the observation that, when working with the cost structure $\mathsf{C}$, the convergence results described in this subsection apply broadly when considering pm-nets up to isomorphism. This is formalized as follows.

\begin{prop}
\label{prop:iso}
Any pm-net is isomorphic to a random measure network model.
\end{prop}

\begin{proof}
Let $\cX$ be an arbitrary pm-net. We define an associated pm-net $\widetilde{\Xcal} = (X,\mu_X,\widetilde{\Omega}_X,\widetilde{\nu}_X,\widetilde{\omega}_X)$ with:
\begin{itemize}[leftmargin=*]
\item $\widetilde{\Omega}_X = L^\infty(X \times X; \mu_X \otimes \mu_X)$ and  $\widetilde{\nu}_X = (\widetilde{m}_\cX)_\# \mu_X$, where $\widetilde{m}_\cX: \Omega_X \to \widetilde{\Omega}_X$ is a map which is closely related to the maps  used in the proof of \Cref{lem:mX_continuous}, namely, 
\[
\widetilde{m}_\cX(t) = \omega_X^t. 
\]
Here, the map $\widetilde{m}_\cX$ is continuous, by the same arguments used in \Cref{lem:mX_continuous}, so that the measure $\widetilde{\nu}_X$ is well-defined;
\item $\widetilde{\omega}_X$ is defined in the obvious way: given a point $\omega_X^t$ in the support of $\widetilde{\nu}_X$, we define $\widetilde{\omega}_X^{\omega_X^t} = \omega_X^t$. 
\end{itemize}
Then $\widetilde{\Xcal}$ is a random measure network model. 

We claim that $\cX$ is a stabilization of $\widetilde{\Xcal}$, hence that $\cX$ is isomorphic to a random measure network model. Indeed, the structure-preserving maps $\Phi: \Omega_X \to \widetilde{\Omega}_X$ and $\varphi:X \to X$ in the definition of stabilization are given by $\Phi = \widetilde{m}_\cX$ and $\varphi = \mathrm{id}_X$. Clearly, these are both measure-preserving maps. The second condition in the definition of structure-preserving maps reads in this case as 
\[
\omega_X^t(x,x') = \widetilde{\omega}_X^{\omega_t}(x,x'),
\]
which is also obvious from the definition. This verifies that $\cX$ is a stabilization of $\widetilde{\Xcal}$ and completes the proof.
\end{proof}

\Cref{prop:iso} says that, when working with the cost structure $\mathsf{C}$, we can replace an arbitrary measure network with a random measure network model at $\GW_\mathsf{C}$-distance zero. Employing these replacements, the sampling result \Cref{thm:approx_random_metrics} then applies to general measure networks.

\section{Numerical Experiments}
\label{sec:experiments}

\subsection{Implementation}
\label{sec:implementation}

We provide Python implementations of the distances described in \Cref{ex:fixed_parameter_space,ex:general_parameter_spaces} with $p = q = 2$. Our implementation builds on the \texttt{ot.gromov.gromov\_wasserstein} function from the Python Optimal Transport (POT) library~\cite{python-ot}, which in turn implements the algorithms of~\cite{gw_averaging,ot-structured-data}. We briefly review the key results from these works before presenting our own algorithms in detail.

Let $\cX = (X, \mu_X, \omega_X)$ and $\cY = (Y, \mu_Y, \omega_Y)$ be measure networks (recall~\Cref{sec:GW}) and let $N := |X|$ and $M := |Y|$. Following the notation of \cite{gw_averaging}, define $C \in \R^{N \times N}$ and $\overline{C} \in \R^{M \times M}$ by $C_{ik} = \omega_X(x_i, x_k)$ and $\overline{C}_{jl} = \omega_Y(y_j, y_l)$. Recall that a coupling $\pi \in \coup(\mu_X, \mu_Y)$ is represented by a matrix $\pi \in \R^{N \times M}$ that satisfies $\pi \cdot \1_M = \mu_X$ and $\pi^\intercal \cdot \1_N = \mu_Y$ where $\1_N \in \R^N$ and $\1_M$ are all-one vectors. Given a function $L:\R^2 \to \R$, define the 4-way tensor
\begin{equation*}
	\mathcal{L}(C, \overline{C}) := \left( L(C_{ik}, \overline{C}_{jl}) \right)_{ijkl} \in \R^{N \times M \times N \times M}
\end{equation*}
and the tensor-matrix multiplication
\begin{equation}
	\label{eq:tensor-matrix-multiplication}
	\mathcal{L}(C, \overline{C}) \otimes \pi := \left( \sum_{kl} L(C_{ik}, \overline{C}_{jl}) \pi_{ik} \right)_{ij} \in \R^{N \times M}.
\end{equation}
\indent Let $\mathcal{L}_p$ be the 4-way tensor induced by $L_p(x,y) := |x-y|^p$ and let $\langle \bullet, \bullet \rangle$ be the Frobenius inner product. The distortion functional satisfies
\begin{equation}
	\label{eq:distortion_tensor}
	\dis_p(\pi, \omega_X, \omega_Y)^p = \sum_{ijkl} |\omega_X(x_i,x_k) - \omega_Y(y_j,y_l)|^p \ \pi_{ij}\pi_{kl} = \sum_{ijkl} L(C_{ik}, \overline{C}_{jl}) \pi_{ij} \pi_{kl} = \langle \mathcal{L}_p(C, \overline{C}) \otimes \pi, \pi \rangle.
\end{equation}
$\mathcal{L}_2(C, \overline{C}) \otimes \pi$ has a simplified form that is an order of magnitude faster to compute than \Cref{eq:tensor-matrix-multiplication}; see \cite[Remark 1]{gw_averaging}.
\begin{lemma}[\!\!{\cite[Proposition 1]{gw_averaging}}]
	Let $f_1(a) = a^2$, $f_2(b) = b^2$, $h_1(a) = a$ and $h_2(b) = 2b$. Then:
	\begin{equation*}
		\mathcal{L}_2(C, \overline{C}) \otimes \pi = c_{C,\overline{C}} - h_1(C) \cdot \pi \cdot h_2(\overline{C})^\intercal,
	\end{equation*}
	where $c_{C, \overline{C}} = f_1(C) \cdot \mu_X \cdot \1_N^\intercal + \1_M \cdot \mu_Y \cdot f_2(\overline{C})^\intercal$.
\end{lemma}

To find $\displaystyle \GW_2(\cX, \cY)$, Peyr\'{e} et al.~\cite{gw_averaging} used projected gradient descent to minimize the function $\mathcal{E}_{C, \overline{C}}(\pi) := \langle \mathcal{L}_2(C, \overline{C}) \otimes \pi, \pi \rangle$; note that \Cref{eq:distortion_tensor} implies $\displaystyle \GW_2(\cX, \cY) = \tfrac{1}{2} \inf_{\pi \in \coup(\mu_X,\mu_Y)} \mathcal{E}_{C, \overline{C}}(\pi)^{1/2}$. A useful observation is that the line-search step, i.e. minimizing the objective function in the direction of the projected gradient, has an explicit solution \cite[Algorithm 2]{ot-structured-data} that we specify in  \Cref{lemma:optim_pot} \Cref{it:optim_pot_line_search}. This Lemma collects other results from \cite{gw_averaging, ot-structured-data} that we use to implement parameterized GW distances. Since these previous works solve more general problems, we also specify the parameter values that yield \Cref{lemma:optim_pot}.

\begin{lemma}
	\label{lemma:optim_pot}
	\hfill
	\begin{enumerate}
		\item\label{it:optim_pot_cost} $\GW_{\mathsf{C}}(\cX, \cY) = \frac{1}{2} \inf_{\pi \in \coup(\mu_X, \mu_Y)} \mathcal{E}_{ C, \overline{C}}(\pi)^{1/2}$.
		\item\label{it:optim_pot_gradient} $\nabla \mathcal{E}_{C, \overline{C}}(\pi) = 2 \mathcal{L}(C, \overline{C}) \otimes \pi$.
		\item\label{it:optim_pot_line_search} Given $\displaystyle \tau = \operatorname*{argmin}_{\tau \in \coup(\mu_X, \mu_Y)} \langle \tau, \nabla \mathcal{E}_{C, \overline{C}}(\pi) \rangle$ and $\tau_\gamma = (1 - \gamma) \pi + \gamma \tau = \pi + \gamma(\tau - \pi)$ for $0 \leq \gamma \leq 1$, the function $f(\gamma) := \mathcal{E}_{C, \overline{C}}(\tau_\gamma)$ expands as a second degree polynomial $f(\gamma) = a\gamma^2 + b\gamma + c$ with coefficients
		\begin{align*}
			a &= - \langle h_1(C) \cdot (\tau-\pi) \cdot h_2(\overline{C})^\intercal, \tau-\pi \rangle \\
			b &= - \langle h_1(C) \cdot \pi \cdot h_2(\overline{C})^\intercal, \tau-\pi \rangle - \langle h_1(C) \cdot \pi \cdot h_2(\overline{C})^\intercal, \pi \rangle.
		\end{align*}
		If $a > 0$, $f$ is minimized when $\gamma$ is either $0$, $1$ or $-b/2a$. Otherwise, $f$ is minimized at $\gamma = 0$ or $\gamma = 1$.
	\end{enumerate}
\end{lemma}
\begin{proof}
\Cref{it:optim_pot_cost} follows from \Cref{eq:distortion_tensor} and the definition of $\mathcal{E}_{C, \overline{C}}(\pi)$. Contrary to the above, Peyr\'{e} et al.~\cite{gw_averaging} defined the GW distance in terms of $\mathcal{E}_{C, \overline{C}}$.
\Cref{it:optim_pot_gradient} originally appears in a formula in \cite[Proposition 2]{ot-structured-data} after setting $\varepsilon = 0$. However, \cite{ot-structured-data} does not contain the derivation and their final formula is missing a factor of 2. The detailed and corrected calculations are found in \cite[Section 1.2]{pot-implementation}.
Finally, considering item \Cref{it:optim_pot_line_search}: the original solution of the line-search step appears in \cite[Algorithm 2]{ot-structured-data} when setting $\alpha = 1$. Instead, we use the formulas from \cite[Section 1.3]{pot-implementation}, which also come with a detailed derivation.
\end{proof}

\subsubsection{Gradient descent for \Cref{ex:fixed_parameter_space}}
\label{sec:optim_fixed}
Let $(\Omega, \nu)$ be a fixed parameter space. When $\Omega$ is finite, the cost structure in \Cref{ex:fixed_parameter_space} becomes a sum of terms of the form $\dis_p(\pi, \omega_X^t, \omega_Y^t)$, and the formulas above generalize accordingly. Consequently, we compute $\GW_{\mathsf{C}}$ with projected gradient descent, using the formulas in \Cref{lemma:optim_fixed_parameters} to find the gradient and the explicit solution of the line-search step.

Fix $p=q=2$ and suppose, for simplicity, that $\Omega = \{1, \dots, T\}$. Let $\cX = (X, \mu_X, \Omega, \nu, \omega_X)$ and $\cY = (Y, \mu_Y, \Omega, \nu, \omega_Y)$ be pm-nets in $\mathfrak{N}_{\nu}$ with $N := |X|$ and $M := |Y|$. Define $C \in \R^{T \times N \times N}$ and $\overline{C} \in \R^{T \times M \times M}$ by $C_{t,i,k} = \omega_X^t(x_i, x_k)$ and $\overline{C}_{t,j,l} = \omega_Y^t(y_j, y_l)$, and let
\begin{equation*}
	\mathcal{E}_{\mathsf{C}, C, \overline{C}}(\pi, \nu) := \sum_t \langle \mathcal{L}_2(C_{t,*,*}, \overline{C}_{t,*,*}) \otimes \pi, \pi \rangle \cdot \nu_t.
\end{equation*}

\begin{lemma}
	\label{lemma:optim_fixed_parameters}
	\hfill
	\begin{enumerate}
		\item\label{it:optim_fixed_cost} $\displaystyle \GW_{\mathsf{C}}(\cX, \cY) = \frac{1}{2} \inf_{\pi \in \coup(\mu_X, \mu_Y)} \mathcal{E}_{\mathsf{C}, C, \overline{C}}(\pi, \nu)^{1/2}$.
		\item\label{it:optim_fixed_gradient} $\displaystyle \nabla_\pi \mathcal{E}_{\mathsf{C}, C, \overline{C}}(\pi, \nu) = \sum_t 2 \mathcal{L}_2(C_{t,*,*}, \overline{C}_{t,*,*}) \otimes \pi \cdot \nu_t$.
		\item\label{it:optim_fixed_line_search} Given $\displaystyle \tau = \operatorname*{argmin}_{\tau \in \coup(\mu_X, \mu_Y)} \langle \tau, \nabla_\pi \mathcal{E}_{\mathsf{C}, C, \overline{C}}(\pi, \nu) \rangle$ and $\tau_\gamma = (1 - \gamma) \pi + \gamma \tau = \pi + \gamma(\tau - \pi)$ for $0 \leq \gamma \leq 1$, the function $f(\gamma) := \mathcal{E}_{\mathsf{C}, C, \overline{C}}(\tau_\gamma, \nu)$ expands as a second degree polynomial $f(\gamma) = a\gamma^2 + b\gamma + c$ with coefficients
		\begin{align*}
			a &= - \sum_t \langle h_1(C_{t,*,*}) \cdot (\tau-\pi) \cdot h_2(\overline{C}_{t,*,*})^\intercal, \tau-\pi \rangle \cdot \nu_t \\
			b &= - \sum_t \left[ \langle h_1(C_{t,*,*}) \cdot \pi \cdot h_2(\overline{C}_{t,*,*})^\intercal, \tau-\pi \rangle + \langle h_1(C_{t,*,*}) \cdot \pi \cdot h_2(\overline{C}_{t,*,*})^\intercal, \pi \rangle \right] \cdot \nu_t.
		\end{align*}
		If $a > 0$, $f$ is minimized when $\gamma$ is either $0$, $1$ or $-b/2a$. Otherwise, $f$ is minimized at $\gamma = 0$ or $\gamma = 1$.
	\end{enumerate}
\end{lemma}
\begin{proof}
Recall from \Cref{ex:fixed_parameter_space} (setting $p=2$) that
\begin{equation*}
		\GW_{\mathsf{C}}(\cX, \cY)
		= \frac{1}{2} \inf_{\pi \in \coup(\mu_X, \mu_Y)} \|\mathrm{dis}_2(\pi,\omega_X,\omega_Y)\|_{L^2(\Omega;\nu)}
		= \frac{1}{2} \inf_{\pi \in \coup(\mu_X, \mu_Y)} \left( \sum_t \dis_2(\pi, \omega_X^t, \omega_Y^t)^2 \cdot \nu_t \right)^{1/2}.
\end{equation*} 
\Cref{it:optim_fixed_cost} follows by applying \Cref{eq:distortion_tensor} to each term above. Since the gradient is linear, using \Cref{lemma:optim_pot}~\Cref{it:optim_pot_gradient} on each summand of $\nabla_\pi \mathcal{E}_{\mathsf{C}, C, \overline{C}}(\pi, \nu)$ yields \Cref{it:optim_fixed_gradient}. Likewise, $f(\gamma) = \mathcal{E}_{\mathsf{C}, C, \overline{C}}(\tau_\gamma, \nu)$ is a sum of second degree polynomials in $\gamma$ with coefficients given by \Cref{lemma:optim_pot}  \Cref{it:optim_pot_line_search}, so \Cref{it:optim_fixed_line_search} follows.
\end{proof}

\subsubsection{Alternating optimization for \Cref{ex:general_parameter_spaces}}
\label{sec:optim_general}
Similar to \Cref{lemma:optim_fixed_parameters}, it is straightforward to generalize the objective function and the formulas in \Cref{lemma:optim_pot} to compute the parametrized $\GW_{\mathsf{C}}$ from \Cref{ex:general_parameter_spaces}. However, this time we need to find two couplings $\xi \in \coup(\nu_X, \nu_Y)$ and $\pi \in \coup(\mu_X, \mu_Y)$ that jointly minimize the objective function, so we have to update the optimization procedure. We set up notation and generalize \Cref{lemma:optim_pot} before explaining the algorithm.

Once again, fix $p=q=2$. Let $(\Omega_X, \nu_X)$ and $(\Omega_Y, \nu_Y)$ be parameter spaces with $\Omega_X = \{1, \dots, T\}$ and $\Omega_Y = \{1, \dots, S\}$, and let $\cX = (X, \mu_X, \Omega_X, \nu_X, \omega_X), \cY = (Y, \mu_Y, \Omega_Y, \nu_Y, \omega_Y) \in \mathfrak{N}_{\mathrm{all}}$. Define $C \in \R^{T \times N \times N}$ and $\overline{C} \in \R^{S \times M \times M}$ by $C_{t,i,k} = \omega_X^t(x_i, x_k)$ and $\overline{C}_{s,j,l} = \omega_Y^s(y_j, y_l)$. Given $\pi \in \coup(\mu_X, \mu_Y)$ and $\xi \in \coup(\nu_X, \nu_Y)$, define
\begin{equation*}
	\mathcal{E}_{\mathsf{C}, C, \overline{C}}(\pi, \xi) := \sum_{t,s} \langle \mathcal{L}_2(C_{t,*,*}, \overline{C}_{s,*,*}) \otimes \pi, \pi \rangle \cdot \xi_{ts}.
\end{equation*}

\begin{lemma}
	\label{lemma:optim_general_parameters}
	\hfill
	\begin{enumerate}
		\item\label{it:optim_general_cost} $\displaystyle \GW_{\mathsf{C}}(\cX, \cY) = \frac{1}{2} \inf_{\pi, \xi} \mathcal{E}_{\mathsf{C}, C, \overline{C}}(\pi, \xi)^{1/2}$ where the inf runs over $\pi  \in \coup(\mu_X, \mu_Y)$ and $\xi \in \coup(\nu_X, \nu_Y)$.
		\item\label{it:optim_general_gradient_pi} $\displaystyle \nabla_\pi \mathcal{E}_{\mathsf{C}, C, \overline{C}}(\pi, \xi) = \sum_{t,s} 2 \mathcal{L}_2(C_{t,*,*}, \overline{C}_{s,*,*}) \otimes \pi \cdot \xi_{ts}$.
		\item\label{it:optim_general_line_search} Given $\displaystyle \tau = \operatorname*{argmin}_{\tau \in \coup(\mu_X, \mu_Y)} \langle \tau, \nabla_\pi \ \mathcal{E}_{\mathsf{C}, C, \overline{C}}(\pi, \xi) \rangle$ and $\tau_\gamma = (1 - \gamma) \pi + \gamma \tau = \pi + \gamma(\tau - \pi)$ for $0 \leq \gamma \leq 1$, the function $f(\gamma) := \mathcal{E}_{\mathsf{C}, C, \overline{C}}(\tau_\gamma, \xi)$ expands as a second degree polynomial $f(\gamma) = a\gamma^2 + b\gamma + c$ with coefficients
		\begin{align*}
			a &= - \sum_{t,s} \langle h_1(C_{t,*,*}) \cdot (\tau-\pi) \cdot h_2(\overline{C}_{s,*,*})^\intercal, \tau-\pi \rangle \cdot \xi_{ts} \\
			b &= - \sum_{t,s} \left[ \langle h_1(C_{t,*,*}) \cdot \pi \cdot h_2(\overline{C}_{s,*,*})^\intercal, \tau-\pi \rangle + \langle h_1(C_{t,*,*}) \cdot \pi \cdot h_2(\overline{C}_{s,*,*})^\intercal, \pi \rangle \right] \cdot \xi_{ts}.
		\end{align*}
		If $a > 0$, $f$ is minimized when $\gamma$ is either $0$, $1$ or $-b/2a$. Otherwise, $f$ is minimized at $\gamma = 0$ or $\gamma = 1$.
	\end{enumerate}
\end{lemma}

We minimize the two-variable objective $\mathcal{E}_{\mathsf{C}, C, \overline{C}}(\pi, \xi)$ with an alternating optimization procedure. The minimization with respect to $\pi \in \coup(\mu_X, \mu_Y)$ is solved using projected gradient descent with the updated formulas in  \Cref{lemma:optim_general_parameters}. The minimization $\inf_{\xi \in \coup(\nu_X, \nu_Y)} \mathcal{E}_{\mathsf{C}, C, \overline{C}}(\pi, \xi)$ is a standard optimal transport problem with cost matrix $M_{\pi, C, \overline{C}} \in \R^{T \times S}$ given by
\begin{equation*}
	(M_{\pi, C, \overline{C}})_{ts} = \langle \mathcal{L}_2(C_{t,*,*}, \overline{C}_{s,*,*}) \otimes \pi, \pi \rangle.
\end{equation*}
In other words, we solve
\begin{equation*}
	\begin{split}
		\text{Minimize: } & \sum_{ts} (M_{\pi, C, \overline{C}})_{ts} \cdot \xi_{ts} \\
		\text{Subject to: } & \xi \in \coup(\nu_X, \nu_Y).
	\end{split}
\end{equation*}

\subsection{Pandas}
\label{sec:pandas}

We begin with a proof of concept that the parametrized GW distance incorporates information that is spread across multiple scales. By ``spreading information'' we mean that given a metric space $(X, d_X)$ and an expression $X = X_1 \cup \cdots \cup X_\ell$, we define the pseudo-metrics $\omega_X^i: X \times X \to \R_{\geq 0}$ by $\omega_X^i(x,x') = d_X(x,x')$ if $x, x' \in X_i$ and $0$ otherwise. Each $\omega_X^i$ only remembers the distances between points in $X_i$, so if we have another metric space $(Y, d_Y)$ with an analogous expression $Y = Y_1 \cup \cdots \cup Y_\ell$ and pseudo-metrics $\omega_Y^i$, the GW coupling between $\omega_X^i$ and $\omega_Y^i$ only has information on $X_i$ and $Y_i$. We use the parametrized GW distance to incorporate the information of all $\omega_X^i$ and $\omega_Y^i$ in one coupling. 

For this experiment, we use a graph that we call a \emph{panda}. Let $P_1$ be a cycle graph of size $N$ and select two vertices of $C$ at distance $e \leq \lfloor N/2 \rfloor$. Given integers $N, n, e$ with $n < N$ and $e \leq \lfloor N/2 \rfloor$, an $(N,n,e)$-panda graph $P$ is formed by gluing two cycle graphs $P_2$ and $P_3$ of size $n$ to $P_1$, one at each distinguished vertex. We say that the $N$-cycle $P_1$ is the head of the panda, and that each of the smaller $n$-cycles $P_2$ and $P_3$ is an ear. Consequently, $P = P_1 \cup P_2 \cup P_3$ and $|P_1 \cap P_i| = 1$ for $i=2,3$. The pseudo-metrics $\omega_P^t$ are defined as above. We define a number of pm-nets from this setup. Let $\Omega := \{1, 2, 3\}$. Let $\mu_P$ and $\nu$ be the uniform measures on $P$ and $\Omega$, respectively, and let $d_P$ be the shortest path distance on $P$. We define the metric measure spaces (mm-spaces) $\mathcal{P}_0 := (P, \mu_P, d_P)$ and $\mathcal{P}_t := (P, \mu_P, \omega_P^t)$ for $t = 1, 2, 3$. We also define the pm-net $\mathcal{P}_{MS} := (P, \mu_P, \Omega, \nu, (\omega_P^t)_{t \in \Omega})$. 

We perform our experiments on a $(25,10,6)$-panda graph $X$ and a $(30,12,6)$-panda $Y$. The pm-nets $\cX_{MS}$, $\cY_{MS}$ and the mm-spaces $\cX_t$ and $\cY_t$ are defined as above.  \Cref{fig:pandas} shows the graph, distance matrix, and the pseudo-metrics $\omega_t$ of $X$ in the top row and those of $Y$, in the bottom.

We compute $\GW_2(\cX_t, \cY_t)$ for $t=1, 2, 3$ and $\GW_{\mathsf{C}}(\cX_{MS}, \cY_{MS})$ where $\mathsf{C}$ is the cost structure of \Cref{ex:fixed_parameter_space} with $p=q=2$. \Cref{fig:panda_couplings} has the optimal couplings $\pi_{\mathsf{C}}$ for $\GW_{\mathsf{C}}(\cX_{MS}, \cY_{MS})$ and $\pi_t$ for $\GW_2(\cX_t, \cY_t)$, $0 \leq t \leq 3$. 
We observe that a single $\pi_t$ with $1 \leq t \leq 3$ only sees the points from $X_t$ and $Y_t$, so every $\pi_t$ is a random coupling outside of a single block. The coupling $\pi_{\mathsf{C}}$ combines the information from these couplings into one. We remark that $\pi_{\mathsf{C}}$ is still not an optimal coupling for $\GW_2(\cX_t, \cY_t)$ because the computation of $\GW_{\mathsf{C}}$ still has no access to interactions between $X_i$ and $Y_j$ for $i \neq j$.

\begin{figure}[!ht]
	\centering
	\begin{tabularx}{0.90\textwidth}{cc|Y}
		Graph & Distance Matrix & $(\omega^t)_{1 \leq t \leq 3}$ \\
		\includegraphics[height=2.5cm]{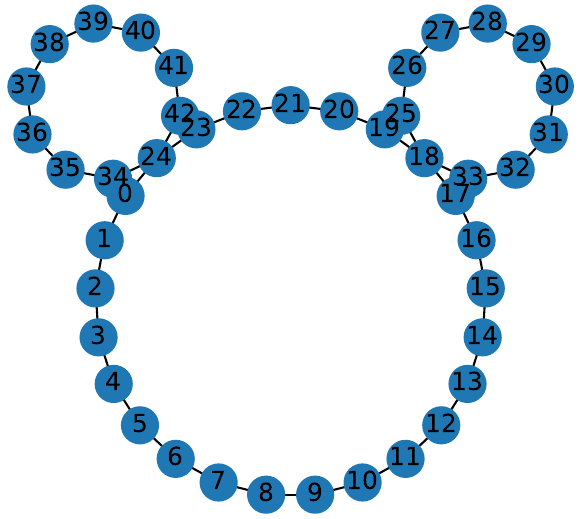}
		&
		\includegraphics[height=2.5cm]{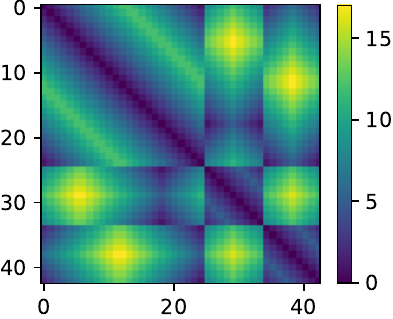}
		&
		\includegraphics[height=2.5cm]{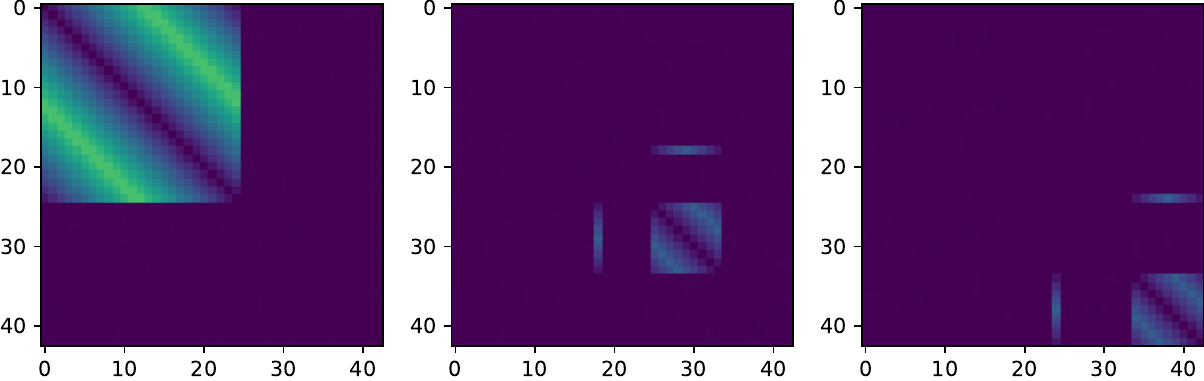}
		\\
		\includegraphics[height=2.5cm]{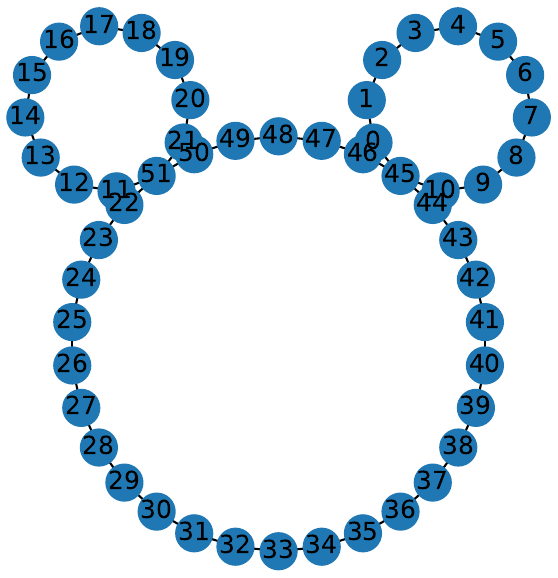}
		&
		\includegraphics[height=2.5cm]{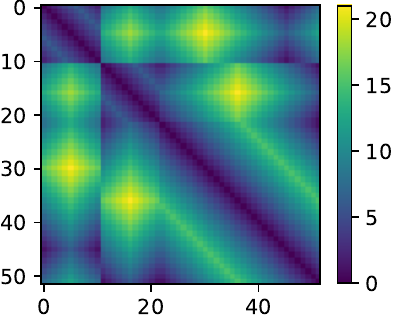}
		&
		\includegraphics[height=2.5cm]{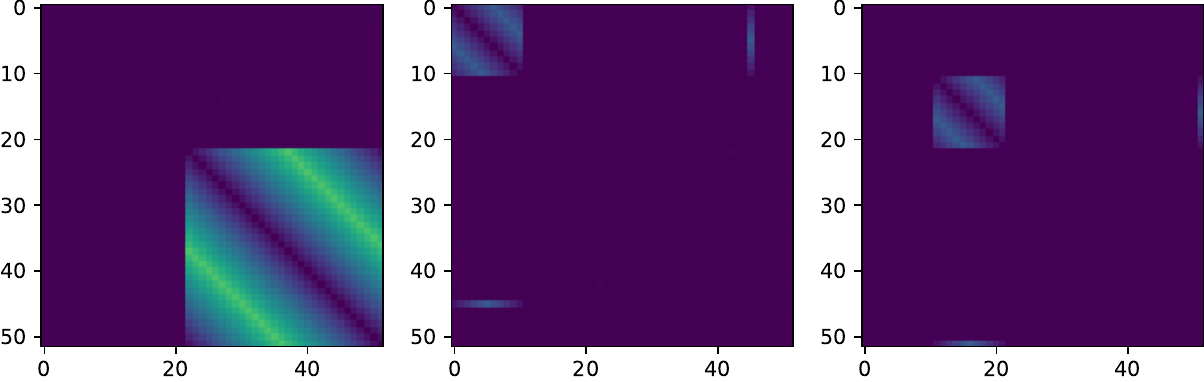}
	\end{tabularx}
	\caption{From left to right, each row shows a panda graph, its distance matrix, and the network functions $\omega^1$, $\omega^2$, and $\omega^3$. The function $\omega^1$ is the restriction of the distance matrix to the vertices of the head, while $\omega^2$ and $\omega^3$ are the restrictions to the ears. In the top row, the head consists of 25 vertices and each ear of 10 vertices; in the bottom row, the head has 30 vertices and each ear 12. The ears are formed by the vertex sets $\{18\} \cup \{25,26,\dots,33\}$ and $\{24\} \cup \{34,\dots,42\}$ in the top panda, and by $\{0,\dots,10\} \cup \{45\}$ and $\{11,\dots,21\} \cup \{51\}$ in the bottom panda. The corresponding heads are given by $\{0,\dots,24\}$ in the top panda and $\{22,\dots,51\}$ in the bottom panda.}
	\label{fig:pandas}
\end{figure}

\begin{figure}[!ht]
	\begin{tabularx}{0.90\textwidth}{cc|Y}
		\includegraphics[height=2.75cm]{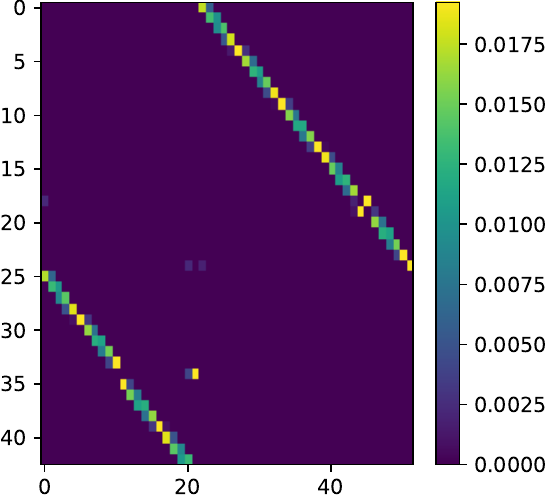}
		&
		\includegraphics[height=2.75cm]{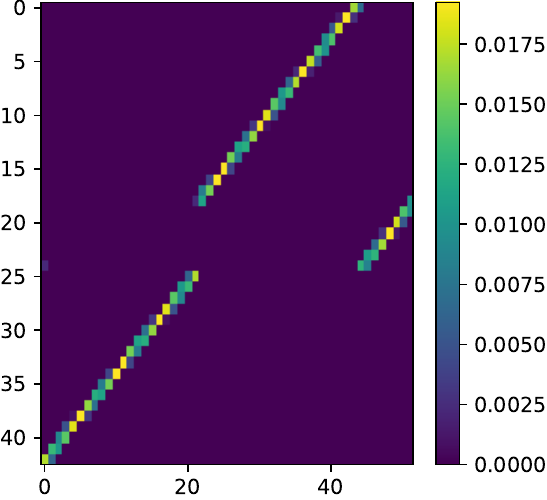}
		&
		\includegraphics[height=2.75cm]{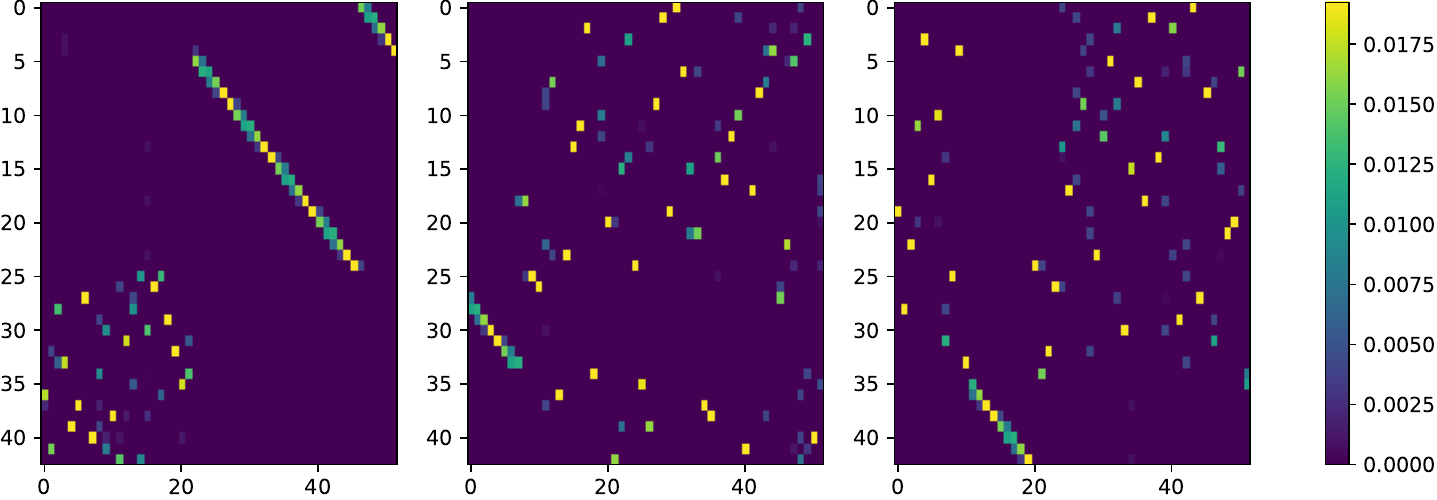}
	\end{tabularx}
	\caption{From left to right, the optimal couplings for $\GW_{\mathsf{C}}(\cX_{MS}, \cY_{MS})$ and $\GW_2(\cX_t, \cY_t)$ for $t = 0, 1, 2, 3$.}
	\label{fig:panda_couplings}
\end{figure}

\subsection{Random Graph Models} 
\label{sec:random-graphs}

The next experiments illustrate the behavior of parameterized GW distances on random graph models. 

\subsubsection{Perturbation Model}\label{subsubsec:perturbation_model} For the first experiment, we study the behavior of the parameterized GW distance on a perturbative random graph model. Given a graph $G = (V,E)$, we construct a  generative random graph model as follows. For a positive integer $k < |E|$, a random sample is generated by  deleting $k$ existing edges and adding $k$ new edges, both uniformly at random. Each graph is represented by a binary adjacency kernel, yielding a pm-net $\mathcal{V}_k = (V,\mu_V,\Omega_V,\nu_V,\omega_V)$, where:
\begin{itemize}
    \item $\mu_V$ is the uniform distribution on $V$;
    \item $\Omega_V$ is the set of all adjacency kernels on $V$;
    \item $\nu_V$ is the (unknown) distribution from which the graph kernels are being sampled under the perturbation model with parameter $k$;
    \item $\omega_V^t = t$ for all $t \in \Omega_V$, i.e., an adjacency kernel.
\end{itemize}

The goal of this experiment is to understand the behavior of the parameterized GW distance on empirical estimates of this pm-net, as in \Cref{sec:approximation_by_samples}. We begin with the well-known Karate Club graph $G=(V,E)$, from~\cite{zachary1977information}. For a fixed $k \in \{1,2,3,4,5\}$ and $n \in \{10,20,50,100,150\}$, we construct an empirical estimate $\mathcal{V}_{k,n}$ of the pm-net $\mathcal{V}_k$ by drawing $n$ samples of the perturbation model with addition/deletion parameter $k$. We then construct an additional empirical estimate $\mathcal{V}_{k,n}'$ via the same procedure (i.e., $\mathcal{V}_{k,n}$ and $\mathcal{V}_{k,n}'$ are both estimates of the same pm-net $\mathcal{V}_k$) and compute $\mathsf{GW}_\mathsf{C}(\mathcal{V}_{k,n},\mathcal{V}_{k,n}')$, where $\mathsf{C}$ is the cost structure defined in \Cref{ex:general_parameter_spaces}, with $p=q=2$. This calculation is repeated 10 times for each choice of parameters $(k,n)$, and results are reported in \Cref{fig:KarateClub}. As a baseline, we compare the empirical estimates using the standard $p=2$ GW distance: for each pair of sampled graphs in $\mathcal{V}_{k,n}$ and $\mathcal{V}_{k,n}'$ (in the arbitrary order they were sampled), we compute the GW distance between their adjacency kernels and then average the results. These results are also recorded in \Cref{fig:KarateClub}. 

The results of this experiment are rather intuitive. The standard GW distance (denoted as \narrowbf{GW}) is essentially constant as the number of samples increases, with the only difference being a tightening of the standard deviations over trials for larger numbers of samples. On the other hand, the parameterized GW distance (denoted as \narrowbf{PGW}) decreases as the number of samples increases---indeed, in theory, this should converge to zero as the number of samples goes to infinity. The difference between the standard GW and parameterized GW distances is more pronounced as $k$ increases, i.e., as the underlying distribution becomes more complicated. This experiment illustrates the benefit of incorporating global information in the distance computation via the parameterized GW framework.

\begin{figure}[!ht]
    \centering
    \includegraphics[width=\linewidth]{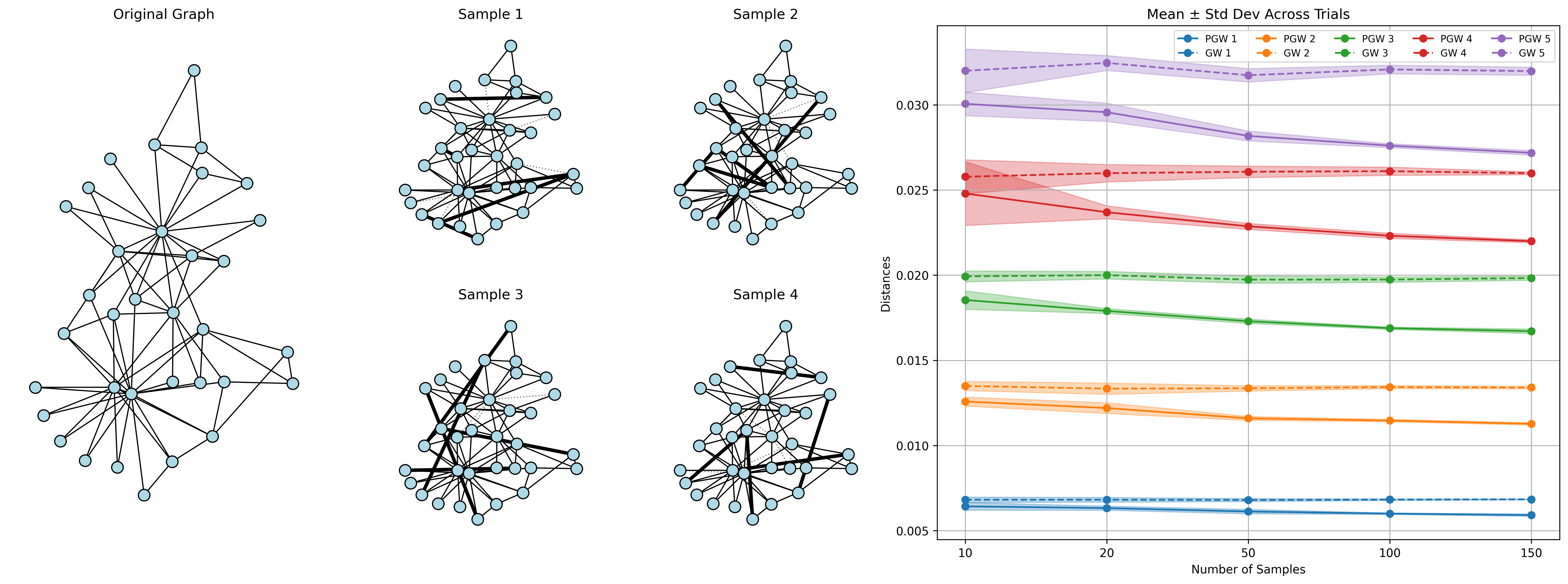}
    \caption{Illustration of the perturbation model from \Cref{subsubsec:perturbation_model}. The Karate Club graph $G$ appears on the left. The center panel shows samples from the perturbation model, where $k=5$ edges are deleted from and added to $G$ in each instance; added edges are drawn in bold, and deleted edges in dotted style. The right panel plots the standard GW distance (dashed, denoted as {\bf GW}) and the parameterized GW distance (solid, denoted as {\bf PGW}) between empirical estimates of the random graph model, as a function of the number of samples ($x$-axis) and the number of edge additions/deletions (encoded by color).}\label{fig:KarateClub}
\end{figure}

\subsubsection{Clustering Random Graphs via Distributions of Total Edges}\label{subsec:clustering_random_graphs} We use the distribution of total edges invariants described in \Cref{subsubsec:weight_distributions_random_graph_models} to cluster random graph models by their parameters. By \Cref{cor:global_distribution_stability} and \Cref{cor:global_distribution_stability_graphs}, this serves a proxy for the parameterized GW distance, and by \Cref{rem:computation_random_graph_models}, it is efficiently computable. Here, we use parameters $q=2$ and $p=1$ when computing Wasserstein distances between distributions of total edges. 

In the first version of the experiment, we use the Erd\H{o}s-R\'{e}nyi random graph model. We consider four instances of this model: in each instance, the underlying graphs have 50 nodes, with the probability of connecting any two nodes given by $\rho \in \{0.44,0.46,0.48,0.5\}$. A single trial of the experiment is described as follows. For each $k \in \{1,5,10,15,20\}$, we draw $k$ graphs from each model (i.e., each choice of $\rho$), and then repeat this a total of 10 times. This gives a total of 40 empirical random graph models, but these really come from only 4 classes---we expect that empirical models with the same $\rho$ should cluster tightly together, and that this clustering should become more pronounced for larger values of $k$. We compute pairwise Wasserstein distances between the distributions of total edges across the dataset of 40 random graph models. Clustering is measured by \emph{Leave One Out Nearest Neighbor (LOONN) score} (higher is better): for a fixed random model, we determine which model among the remaining 39 is closest to the fixed one (using Wasserstein distance between distributions of total edges); if the closest model has the same $\rho$-value as the fixed one, this is treated as a success, and total success percentage across the dataset is the reported score. We repeat the full experiment 10 times, and the results are provided in  \Cref{fig:graph_clustering}. Observe that the results agree with intuition. Indeed, increasing $k$ yields higher clustering scores, and the distribution of total edges appears to capture the dependence of the models on $\rho$ quite well (this is unsurprising, given \Cref{rem:erdos_renyi_model}).

We next run the same experiment on a different random graph model. In the second version of the experiment, we use \emph{stochastic block models}. In each instance, we have a graph on 50 nodes which have been partitioned into even groups of 25. An instance of the model depends on parameters $\rho_1,\rho_2 \in [0,1]$, where $\rho_1$ is the probability of connecting any two nodes within a partition block, and $\rho_2$ is the probability of connecting two nodes lying in distinct blocks. Here, we also use four classes, with 
\[
(\rho_1,\rho_2) \in \{(0.5,0.28), (0.5,0.3), (0.6,0.28), (0.6,0.3)\}.
\]
The experimental setup is then identical to the above. The results  (also reported in \Cref{fig:graph_clustering}) are qualitatively similar to the Erd\H{o}s-R\'{e}nyi case, but quantiatively indicate that this classification task is slightly more difficult. 

\begin{figure}
    \centering
    \includegraphics[width=0.45\linewidth]{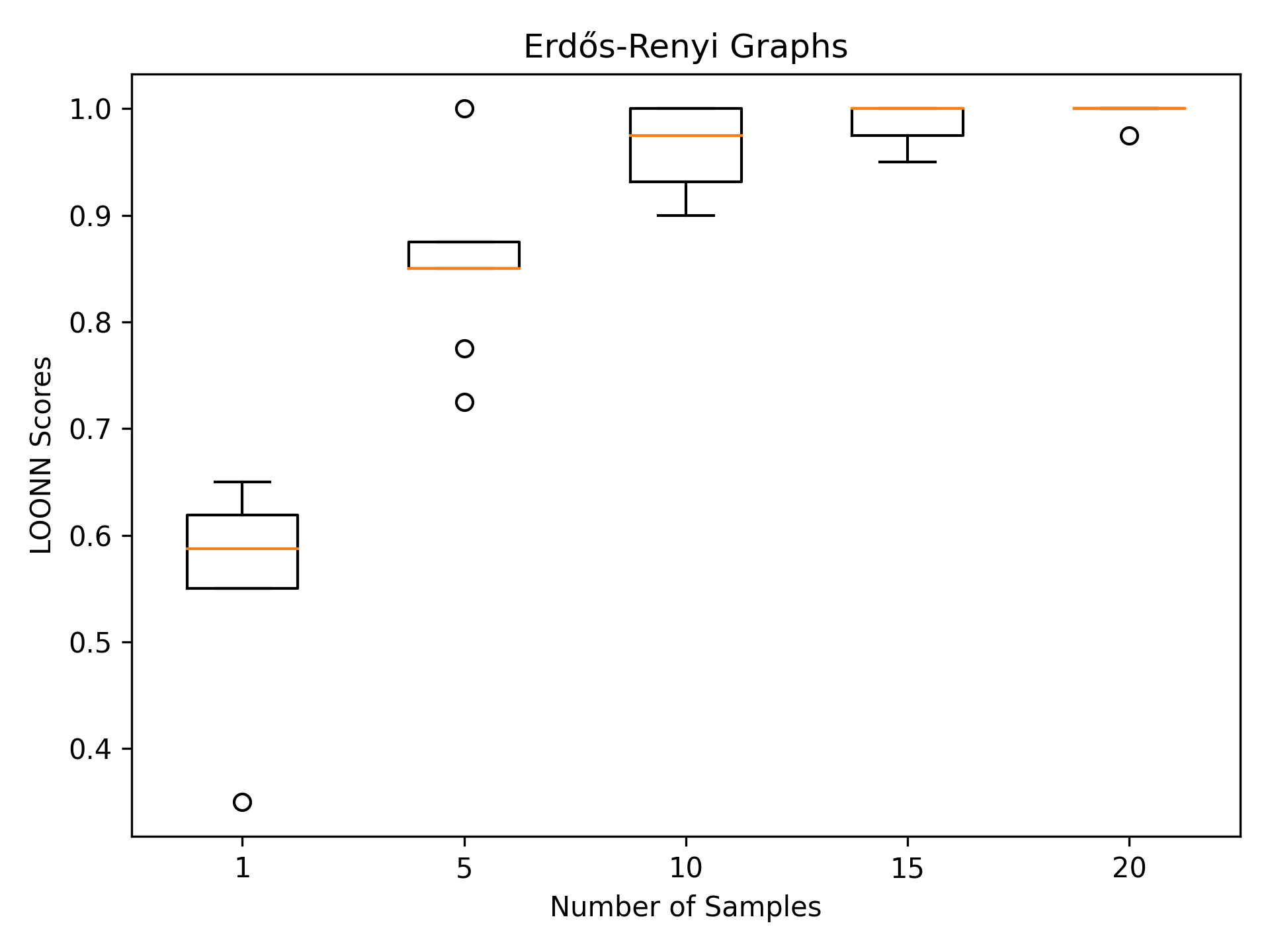}
    \includegraphics[width=0.45\linewidth]{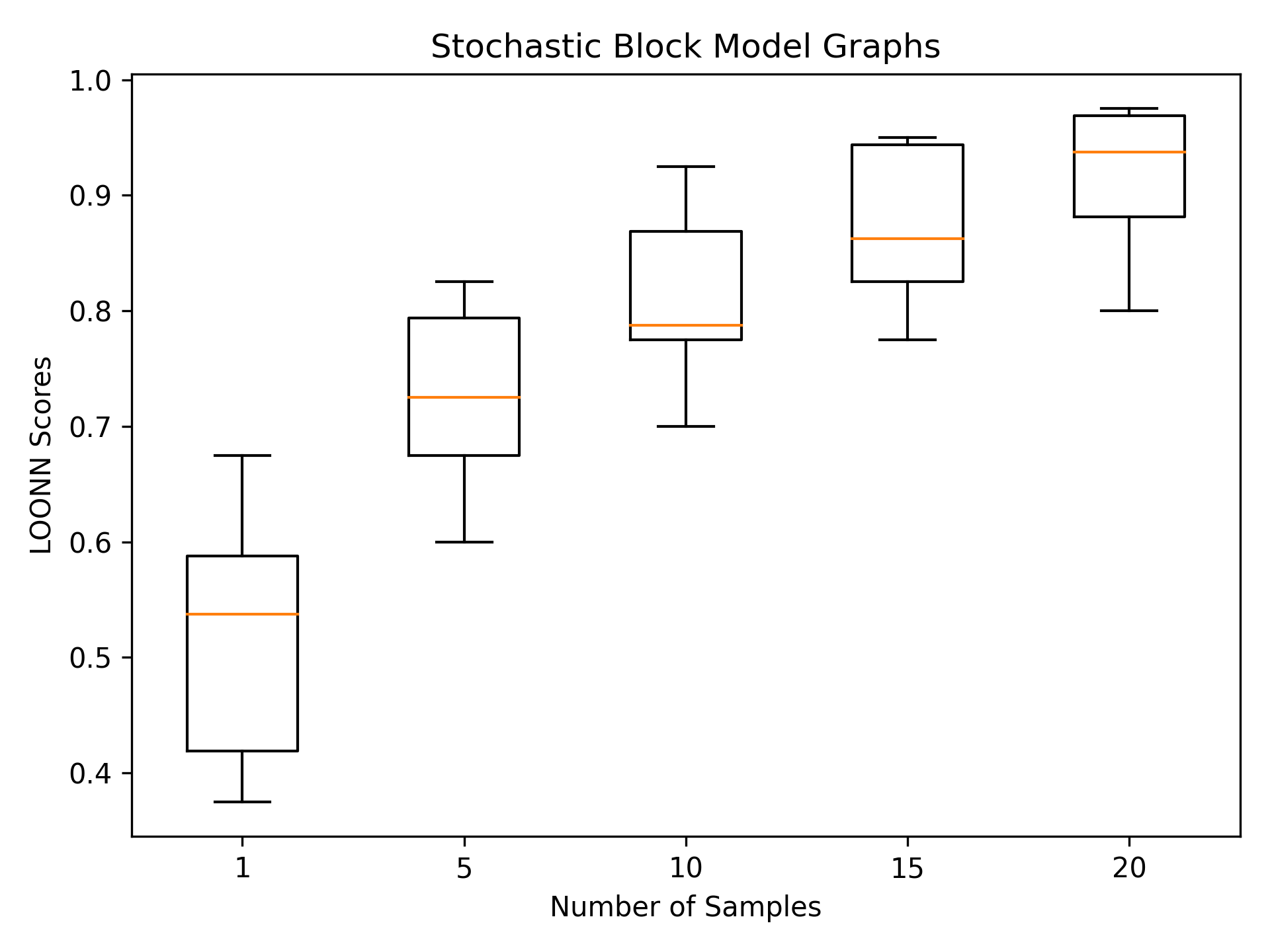}
    \caption{LOONN clustering scores (higher is better) versus number of samples in the empirical approximations for the Erd\H{o}s-R\'{e}nyi random graph model {\bf (Left)} and the stochastic block model {\bf (Right)}.}
    \label{fig:graph_clustering}
\end{figure}

\subsection{Nested Cycles}
\label{sec:nested-cycles}

The graph heat kernel represents the diffusion of heat in a graph across time, and it captures graph features at increasing scales as time advances. In many applications, people perform tests on graphs using the heat kernel at a single time, which raises the question of how to compare graphs that have features at multiple scales that the heat kernel cannot capture simultaneously.

We study this question with the following family of graphs. Given a sequence of graphs $G_1, \dots, G_n$ and basepoints $v_1 \in G_1, \dots, v_n \in G_n$, we define an \emph{$n$-cycle of graphs} as the result of attaching each $G_i$ to an $n$-cycle $C_n$ by gluing $v_i \in G_i$ to the $i$-th vertex of $C_n$. We set $v_1$ as the basepoint of the resulting graph. For a fixed set of positive integers $n_1, \dots, n_\ell$ and $m$, we define a \emph{1-nested cycle of cliques} of type $(n_1, m)$ as an $n_1$-cycle of $m$-cliques and an \emph{$\ell$-nested cycle of cliques} of type $(n_1, \dots, n_\ell, m)$ as an $n_1$-cycle of cliques of type $(n_2, \dots, n_\ell, m)$. Note that this construction is independent of the choice of basepoint in the $m$-cliques. See \Cref{fig:nested_graphs}. We refer to the $n_i$-cycles and $m$-cliques as \emph{features at scale $i$} and $\ell+1$, respectively. The heat kernel of an $\ell$-nested cycle of cliques has features at $\ell$ different times because, in order for heat to diffuse through each $n_i$-cycle, it first has to diffuse through $(\ell-i)$-nested cycles of cliques, and this process takes longer for smaller $i$.

\begin{figure}
    \centering
    \includegraphics[width=0.40\linewidth]{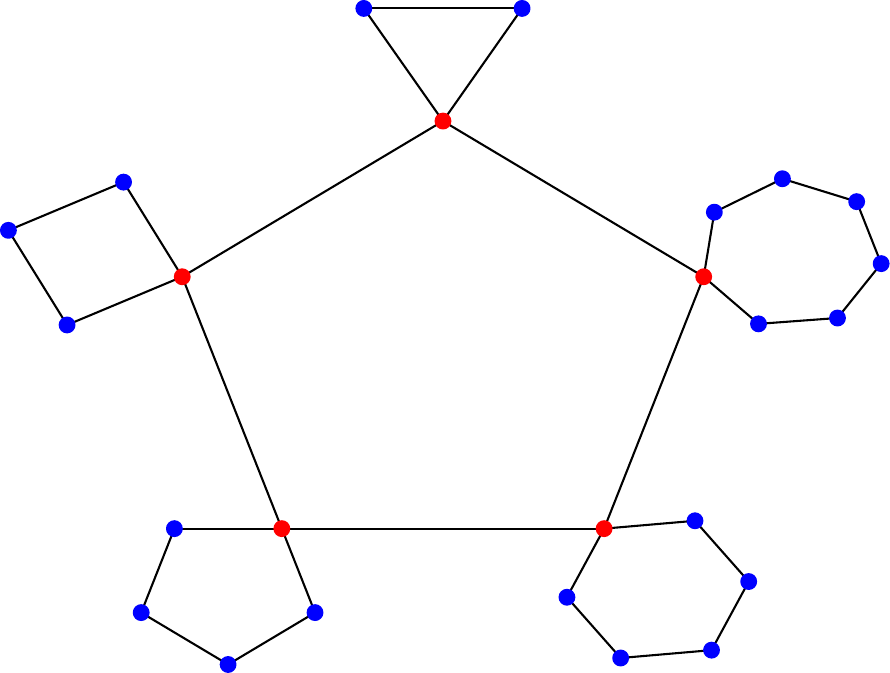}
    \quad\quad
    \includegraphics[width=0.40\linewidth]{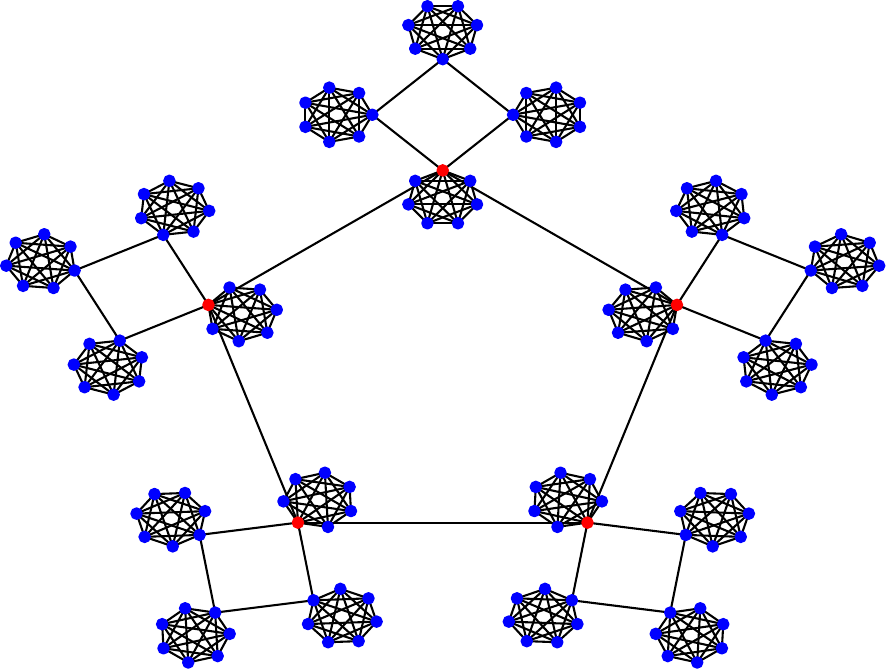}
    \caption{\textbf{Left:} A $5$-cycle of graphs formed by the cycle graphs $C_3, C_4, C_5, C_6, C_7$. The basepoint of each cycle graph is marked in red. \textbf{Right:} A 2-nested cycle of cliques of type $(5,4,7)$. This graph is a 5-cycle of 1-nested cycles of type $(4,7)$, each of which is a 4-cycle of 7-cliques. The basepoint of each 1-nested cycle is marked in red.}
    \label{fig:nested_graphs}
\end{figure}

To set a benchmark for the upcoming experiments, suppose we compute the GW distance between the distance matrices of two $\ell$-nested cycles of cliques of types $(n_1, \dots, n_\ell, m)$ and $(n_1, \dots, n_\ell, m')$ with $m \neq m'$ and equipped with the uniform measure. The optimal coupling $\pi$ captures the multiscale structure of these graphs through a hierarchy of nested blocks. At the coarsest level, these graphs are $n_1$-cycles of subgraphs, so $\pi$ should be an $n_1$-by-$n_1$ grid of blocks of size $(N_1 m)$-by-$(N_1 m')$, where $N_i=n_{i+1} \cdots n_\ell$. These outer blocks form a coupling\footnote{More precisely, the $n_1$-by-$n_1$ matrix of block sums is a coupling between two $n_1$-cycles.} between two $n_1$-cycles. After normalizing, each non-zero sub-block of size $(N_2m)$-by-$(N_2m')$ is a coupling between $(\ell-1)$-nested cycles of cliques, and we can iterate this description on each block. Thus, $\pi$ has a nested block structure where the outer blocks form a coupling of $n_1$-cycles, each non-zero block thereof is a coupling of $n_2$-cycles, and so on. At the smallest scale, the blocks that form a coupling of $n_\ell$-cycles are themselves couplings between cliques of sizes $m$ and $m'$. Note that an optimal coupling between $n$-cycles with uniform measure is a cyclic permutation of $\{1, \dots, n\}$, while an optimal coupling between cliques is random. We say that the outer blocks of size $(N_1m)$-by-$(N_1m')$ occur at \emph{scale 1}, their sub-blocks of size $(N_2m)$-by-$(N_2m')$ occur at \emph{scale 2}, and so on.

To test how well the GW distance captures multiscale features, we compare heat kernels at multiple values of $t$. Concretely, let $G_1$ and $G_2$ be 2-nested cycles of cliques of types $(10, 5, 5)$ and $(10, 5, 20)$, and let $H_{i,t}$ denote the heat kernel of $G_i$ at time $t$. Based on the discussion above, we expect the optimal coupling $\pi$ to have 10 blocks, each with 5 sub-blocks of random noise of size 5-by-20, and the blocks at scales 1 and 2 should form cyclic permutations of 10- and 5-cycles, respectively. We summarize this structure using the binary vector \texttt{cyclic} above each coupling in \Cref{fig:HK_GW_2} and \Cref{fig:HK_MS_2}. The $i$-th entry of \texttt{cyclic} indicates whether all non-zero blocks of $\pi$ at scale $i$ form cyclic permutations of $n_i$-cycles. In line with our intuition, the GW couplings in \Cref{fig:HK_GW_2} capture small and large scale features at different times. Specifically, only the blocks at scale 2 are cyclic permutations when $30 \leq t \leq 50$ (as indicated by $\texttt{cyclic = [0, 1]}$), while the opposite is true when $t \geq 170$ ($\texttt{cyclic = [1, 0]}$).

\begin{figure}[!ht]
	\includegraphics[width=\textwidth]{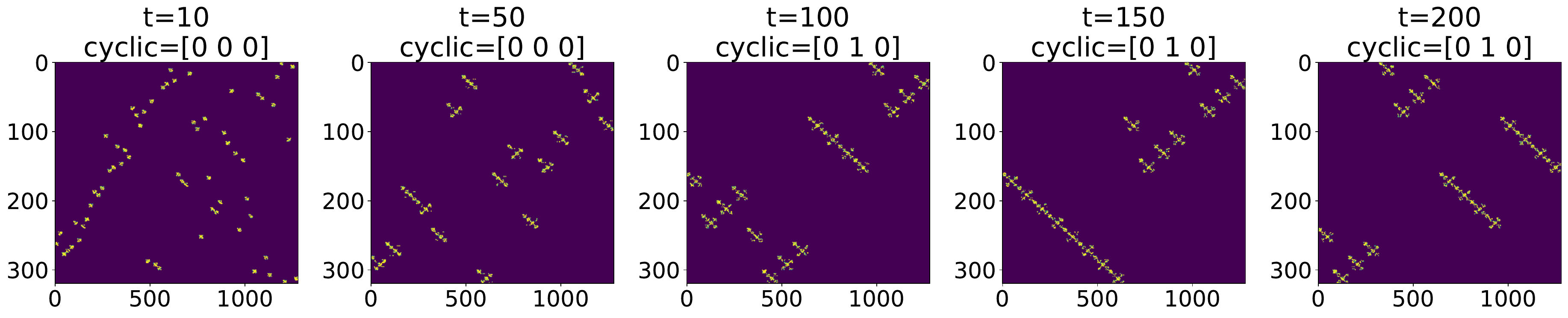}
	\caption{Optimal couplings for $\GW_2(H_{1,t}, H_{2,t})$, where $H_{i,t}$ is the heat kernel of the 2-nested cycle of cliques $G_i$, and $G_1$ and $G_2$ have types $(10,5,5)$ and $(10,5,20)$. Each panel has the time parameter of $H_{i,t}$, and the vector \texttt{cyclic} indicates whether the coupling is a cyclic permutation at each scale.}
	\label{fig:HK_GW_2}
\end{figure}

\begin{figure}[!ht]
	\includegraphics[width=\textwidth]{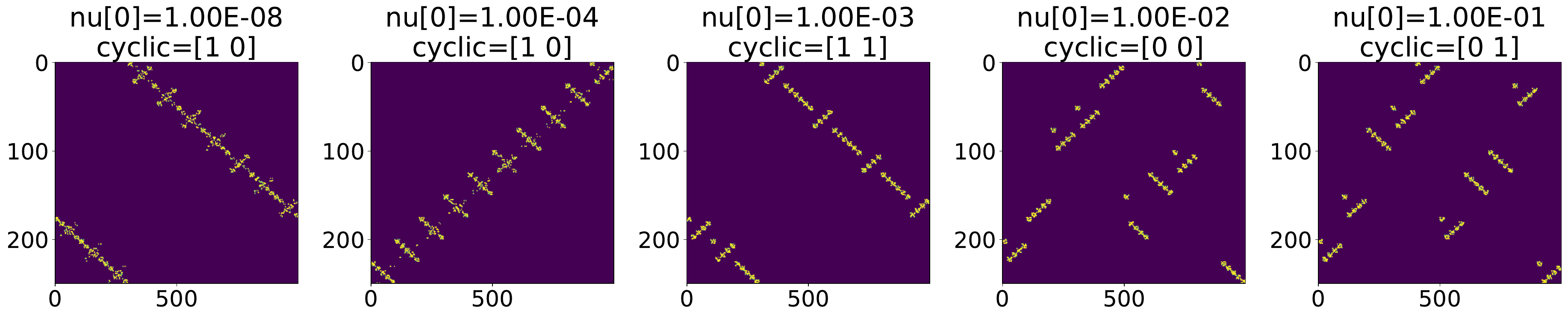}
	\caption{Optimal couplings for $\GW_{\mathsf{C}}(\mathcal{H}_{1}, \mathcal{H}_{2})$. $\mathcal{H}_i$ is a sequence of two heat kernels $H_{i,t}$ with $t_1 = 50$ and $t_2 = 200$. The title of each panel contains the value of $\nu_1$ ($\texttt{nu[0]}$ in 0-indexing) and the vector \texttt{cyclic} that indicates whether the coupling is a cyclic permutation at each scale. The coupling induced by $\nu_1 = 10^{-3}$ (and $\nu_2 = 0.999$) has the desired block structure at all scales (i.e. $\texttt{cyclic=[1,1]}$).}
	\label{fig:HK_MS_2}
\end{figure}

We now attempt to capture features at both scales simultaneously with the parametrized GW distance on a fixed parameter space. Let $\mathsf{C}$ be the cost structure from  \Cref{ex:fixed_parameter_space} with $p=q=2$. We set $t_1=50$, $t_2=200$ and $\Omega = \{t_1, t_2\}$, but select $\nu$ later. We manually chose $t_1$ and $t_2$ because the GW couplings satisfy $\texttt{cyclic = [0,1]}$ when $t=50$ and $\texttt{cyclic = [1,0]}$ when $t = 200$. We construct pm-nets $\mathcal{H}_{i} = (G_i, \mu_i, (H_{i, t_j})_{j=1,2}, \Omega, \nu)$ with the uniform measure $\mu_i$ for each $i=1,2$.

We have to carefully choose $\nu$ because of a numerical issue in the computation of $\GW_{\mathsf{C}}(\mathcal{H}_1, \mathcal{H}_2)$. The extreme values of $H_{1,t_1}$ and $H_{2,t_1}$ are several orders of magnitude larger than those of $H_{1,t_2}$ and $H_{2,t_2}$, even after normalization (e.g. with the Frobenius norm). When we set $\nu$ as the uniform measure, the values at $t_1$ dominate the optimization and the resulting coupling resemble the GW coupling at $t_1$. We resolve this issue by doing a grid search on $\nu$. \Cref{fig:HK_MS_2} has the optimal couplings for $\GW_{\mathsf{C}}(\mathcal{H}_1, \mathcal{H}_2)$ for several choices of $\nu$. In particular, the coupling with $\nu_1 = 10^{-3}$ and $\nu_2 = 1 -\nu_1$ has the expected block structure ($\texttt{cyclic = [1, 1]}$). 

Therefore, after some parameter tuning, the parametrized GW distance with a fixed parameter space (\Cref{ex:fixed_parameter_space}) captures information that is spread across multiple heat kernels with a single coupling.

\subsubsection{3-nested cycles of cliques}
We repeat the experiments above with 3-nested cycles of cliques of types $(4, 4, 4, 5)$ and $(4, 4, 4, 20)$ to see if the parametrized GW distance needs 3 levels to capture features at three scales. 

For the standard GW framework, we construct the heat kernels $H_{1,t}$ and $H_{2,t}$ as above with $10 \leq t \leq 500$ and compute $\GW_2(H_{1,t}, H_{2,t})$; see \Cref{fig:HK_GW_3}. The blocks of the GW couplings form cyclic permutations of 4-cycles at scale 3 when $t=20, 30$ ($\texttt{cyclic=[0,0,1]}$), at scale 2 when $100 \leq t \leq 280$ ($\texttt{cyclic=[0,1,0]}$) and at scale 1 when $t=410, 430$ ($\texttt{cyclic=[1,0,0]}$). 

For the parametrized GW framework, we manually select $t_1=30$, $t_2=100$, and $t_3=410$, and set $\Omega = \{t_1, t_2, t_3\}$ and $\mathcal{H}_{i} = (G_i, \mu_i, (H_{i, t_j})_{1 \leq j \leq 3}, \Omega, \nu)$. We perform a grid search over $\nu$ and found that $\nu_1 = 1.29 \times 10^{-4}$, $\nu_2 = 3.39 \times 10^{-5}$ and $\nu_3 = 1-\nu_1-\nu_2 \approx 9.998 \times 10^{-1}$ produces a coupling with the correct block structure at all scales; see \Cref{fig:HK_MS_3_lev_3}.

\begin{figure}[!ht]
	\includegraphics[width=\textwidth]{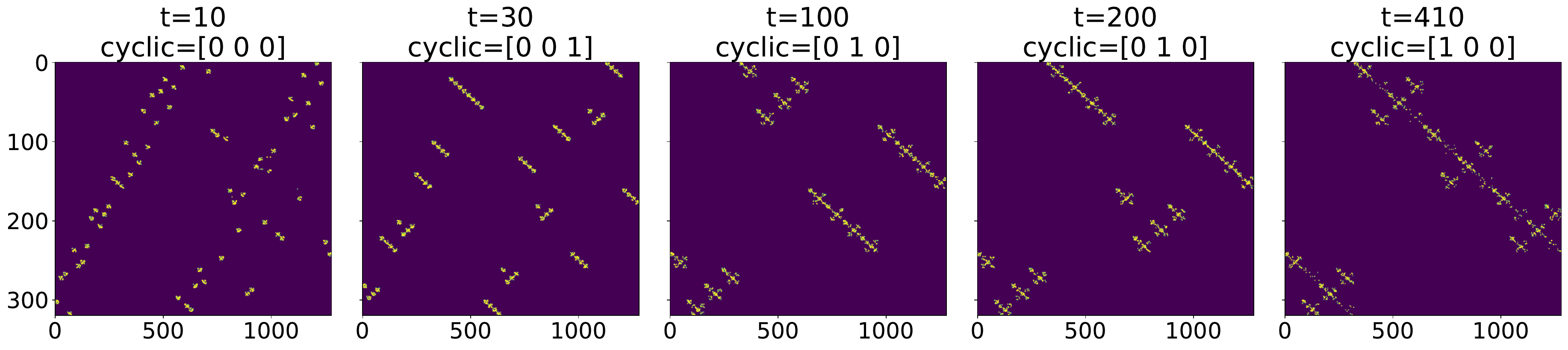}
	\caption{Optimal couplings for $\GW_2(H_{1,t}, H_{2,t})$, where $H_{i,t}$ is the heat kernel of the 3-nested cycle of cliques $G_i$. $G_1$ and $G_2$ have types $(4,4,4,5)$ and $(4,4,4,20)$. Each panel has the time parameter of $H_{i,t}$, and the vector \texttt{cyclic} indicates whether the coupling is a cyclic permutation at each scale.}
	\label{fig:HK_GW_3}
\end{figure}

\begin{figure}
	\includegraphics[width=0.8\textwidth]{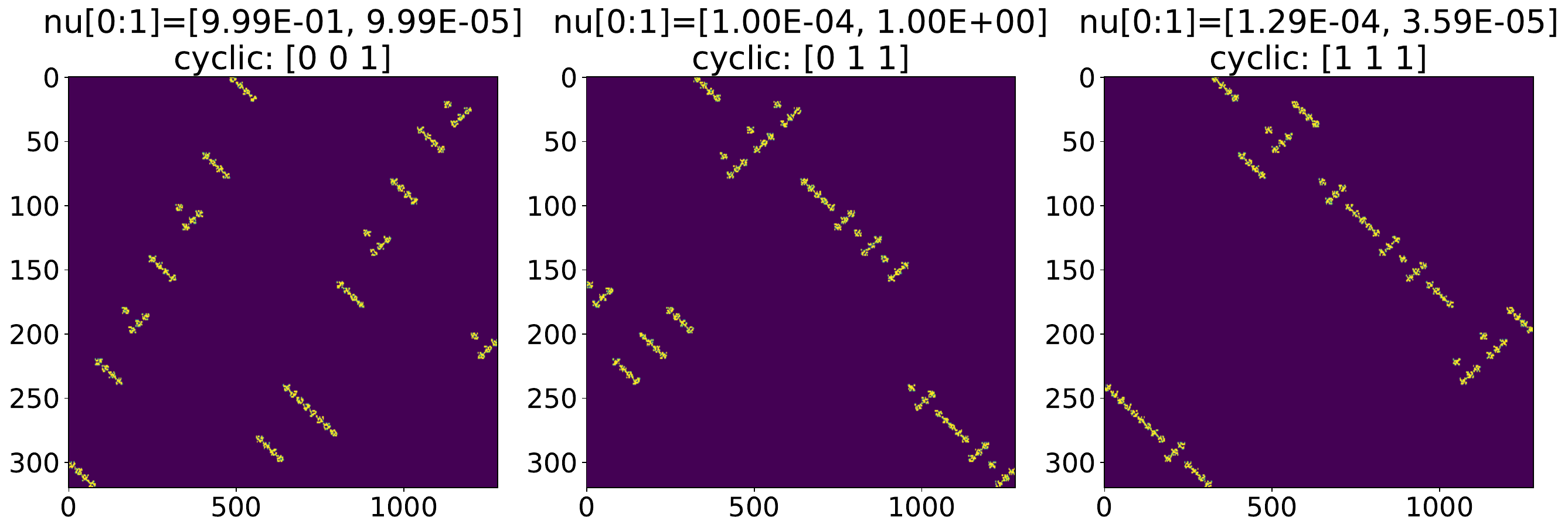}
	\caption{Optimal couplings for $\GW_{\mathsf{C}}(\mathcal{H}_{1}, \mathcal{H}_{2})$. $\mathcal{H}_i$ is a sequence of three heat kernels $H_{i,t}$ with $t_1 = 30$, $t_2 = 100$, and $t_3 = 410$. The title of each panel contains the value of $\nu_1$ and $\nu_2$ (\texttt{nu[0:1]} in 0-indexing, and rounded to 2 decimal places) and the vector \texttt{cyclic} that indicates whether the coupling is a cyclic permutation at each scale. The coupling induced by $\nu_1 = 1.29 \times 10^{-4}$, $\nu_2 = 3.39 \times 10^{-5}$ and $\nu_3 = 1-\nu_1-\nu_2 \approx 9.998 \times 10^{-1}$ has the desired block structure at all scales ($\texttt{cyclic=[1,1,1]}$).}
	\label{fig:HK_MS_3_lev_3}
\end{figure}

\subsection{Feature Selection}
\label{sec:feature-selection}

In this subsection, we propose the parameterized GW distance as a cost function for feature selection. We interpret $\Omega$ as the set of feature labels (invariants) and $\nu$ as their relative importance. For instance, in the case of graph data one may take $\Omega = \{\text{adj}, \text{Lap}, \text{dist}\}$, where $\omega_G^{\text{adj}}$ denotes the adjacency matrix, $\omega_G^{\text{Lap}}$ the Laplacian, and $\omega_G^{\text{dist}}$ the shortest-path distance matrix of a given graph $G$. 

Our objective is to find the weights $\nu$ that optimize a downstream task. Invariants that hinder performance should get small or zero weights, while those that contribute positively should get larger weights. We work on the class $\mathfrak{N}_{\nu}$ of pm-nets parameterized by $(\Omega, \nu)$ (recall \Cref{ex:classes_of_pmnets}) and use the cost structure from  \Cref{ex:fixed_parameter_space} with $p=q=2$. For any $\cX = (X, \mu_X, \Omega, \nu, \omega_X) \in \mathfrak{N}_{\nu}$, the set $\{\omega_X^t : t \in \Omega\}$ contains all invariants of $(X, \mu_X)$, and the parametrized GW distance $\GW_{\mathsf{C}}(\cX, \cY)$ measures the difference between the corresponding invariants of $\cX, \cY \in \mathfrak{N}_{\nu}$.

To choose a concrete task, suppose we have pm-nets $\cX_1, \dots, \cX_n$ with class labels $y_1, \dots, y_n \in \{1, \dots, m\}$, and we want to determine which invariants from $\Omega$ correctly classify them. Let $M_\nu$ be the $n$-by-$n$ matrix given by $(M_\nu)_{ij} = \GW_{\mathsf{C}}(\mathcal{X}_i, \mathcal{X}_j)$. Suppose that $y_1 \leq \dots \leq y_n$ so that $M_\nu$ has the block structure
\begin{equation}
	\label{eq:classification_matrix}
	M_\nu =
	\begin{pmatrix}
		B_{11} & \cdots  & B_{1m}\\
		\vdots & \ddots & \vdots \\
		B_{m1} & \cdots & B_{mm}
	\end{pmatrix}
\end{equation}
where $B_{ij}$ is the matrix of distances between elements of classes $i$ and $j$. The best clustering is achieved when the intra-cluster distances are small relative to the inter-cluster distances, so we want to find the $\nu$ that minimizes
\begin{equation*}
	\operatorname{cost}_p(\nu) := \frac{\|B_{11}\|_p^p + \cdots + \|B_{mm}\|_p^p}{\|M_\nu\|_p^p}
\end{equation*}
where $\| \bullet \|_2$ is the Frobenius norm. Depending on the context, we may add a regularization term and minimize
\begin{equation}
	\label{eq:classification_cost_lambda}
	\operatorname{cost}_{p,\lambda}(\nu) := \frac{\|B_{11}\|_p^p + \cdots + \|B_{mm}\|_p^p}{\|M_\nu\|_p^p} + \lambda \cdot KL(\nu | q)
\end{equation}
instead, where $\lambda \geq 0$, $q$ is the uniform measure on $\Omega$, and $KL(\nu | q)$ is the KL divergence.

\subsubsection{Dynamic Metric Spaces}
Recall that $\cX$ is a dynamic metric space if $\Omega$ is a compact subset of $\R_{\geq 0}$ and every $\omega_X^t$ is a (pseudo)-metric. Suppose that a set of drones flies through one of two corridors that have the same shape, except that one corridor has an obstacle. If the drones maintain roughly the same speed and direction, their behavior only changes significantly when they dodge the obstruction. We use the pipeline above to identify the times when the drones find the obstacle.

We define two types of pm-nets that represent the flight of the drones through one of the two corridors; see \Cref{fig:drone_flight}. Fix the indexing sets $X = \{0,\cdots,4\} \times \{0, \cdots, 4\}$ and $\Omega = \{0, \dots, 4\}$, and let $\mu_X$ and $\nu_0$ be their uniform probability measures. The corridor is the rectangle $[-1,4] \times [-2,2]$ in $\R^2$ and the obstacle is the ellipse $\left(\frac{x-1.5}{0.5}\right)^2 + \left(\frac{y}{0.7}\right)^2 = 1$. The initial drone configuration is the $5 \times 5$ grid $P(X) \subset [0,1] \times [-1,1]$ where $P(x,y) = \left(\frac{x}{4}, -1+\frac{y}{2} \right)$. Each point represents a drone. At each time step, all drones move $\Delta x = 0.6$ units to the right resulting in 5 grids $X_0, \dots, X_4$ of 25 points each. In the corridor without obstacles, we apply Gaussian noise $\mathcal{N}(0,0.05)$ independently to every point of $X_i$, and define $\omega_X^i$ as the distance matrix of $X_i$. This process defines a random variable $\cX = (X, \mu_X, \Omega, \nu_0, \omega_X)$ valued in pm-nets that we call \emph{clear flight}.

In the presence of an obstacle, a drone avoids collision by moving above the ellipse if its $y$-coordinate is positive and below the ellipse otherwise. Within each vertical column of drones (i.e. drones with the same $x$ coordinate), those that go above remain evenly spaced between the top of the obstacle and the line $y=1$, while those that go below remain evenly spaced between the line $y=-1$ and the bottom of the obstacle. As before, this process produces 5 grids $Y_0, \dots, Y_4$ of drones to which we apply independent Gaussian noise $\mathcal{N}(0,0.05)$, except when it would cause a collision. Let $\overline{\omega}_X^i$ be the distance matrix of $Y_i$; the tuple $\overline{\cX}= (X, \mu_X, \Omega, \nu_0, \overline{\omega}_X^i)$ defines another random variable valued in pm-nets called \emph{obstructed flight}.

In our experiment, we sample 10 instances of pm-nets: 5 clear flights $\cX_1, \dots, \cX_5$ and 5 obstructed flights $\cX_6, \dots, \cX_{10}$. We then apply alternating optimization to minimize $\operatorname{cost}_\lambda(\nu)$ for several values of $\lambda$, with results summarized in \Cref{table:drone_feature_selection}. As $\lambda$ increases from $0.01$ to $10$, the optimal measure $\nu$ assigns different weights to the features. For $\lambda = 0.01$, the regularization is too weak, and $\nu$ concentrates on the time steps where the drone configurations differ the most (see \Cref{fig:drone_flight}). In contrast, when $\lambda = 10$, the strong regularization drives $\nu$ close to the uniform measure, failing to distinguish between time steps. The most informative range is between $0.1$ and $1$, where the weights in $\nu$ capture the differences between the two flight classes. Specifically, in obstructed flights the drones begin crossing the obstacle at time 1 and nearly clear it by time 3, making times 1 and 2 the most distinctive, time 3 moderately distinctive, and times 0 and 4 indistinguishable. This pattern is reflected in the weights, with $\nu_1$ and $\nu_2$ largest, $\nu_3$ intermediate, and $\nu_0$ and $\nu_4$ smallest.

\begin{table}
	\begin{equation*}
		\begin{array}{c|ccccc|c}
			\lambda & \nu_0 & \nu_1 & \nu_2 & \nu_3 & \nu_4 & \operatorname{cost}_\lambda(\nu)\\
			\hline
			0.01 & 0.007& 0.001& \textbf{0.968}& 0.022& 0.001 & 0.095149 \\
			0.1 & 0.011& \textbf{0.423}& \textbf{0.525}& 0.027& 0.014 & 0.187433 \\
			1 & 0.146& \textbf{0.258}& \textbf{0.298}& 0.151& 0.146 & 0.283082 \\
			10 & 0.195& 0.206& 0.210&  0.195& 0.195 & 0.296712
		\end{array}
	\end{equation*}
	\caption{Minimizer of $\operatorname{cost}_\lambda(\nu)$ for several values of $\lambda$. If $\lambda $ is small, the measure $\nu$ assigns the most weight to the single time that distinguishes classes 1 and 2 the best. Conversely, if $\lambda$ is too big, there is no time that has significantly larger weight than the others. In the intermediate range ($\lambda = 0.1, 1$), the times when the drones avoid the obstacle ($t=1,2$) have the largest weights.}
	\label{table:drone_feature_selection}
\end{table}

\begin{figure}
	\centering
	\includegraphics[width=0.8\linewidth]{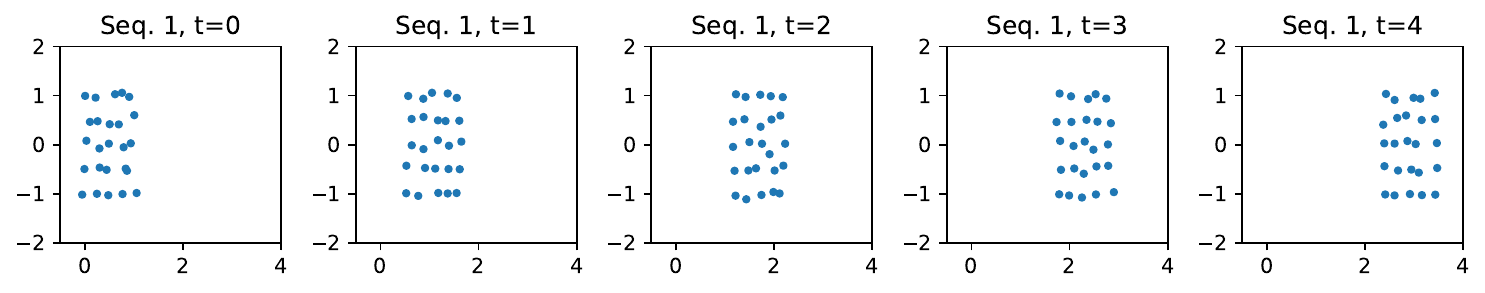}\\
	\includegraphics[width=0.8\linewidth]{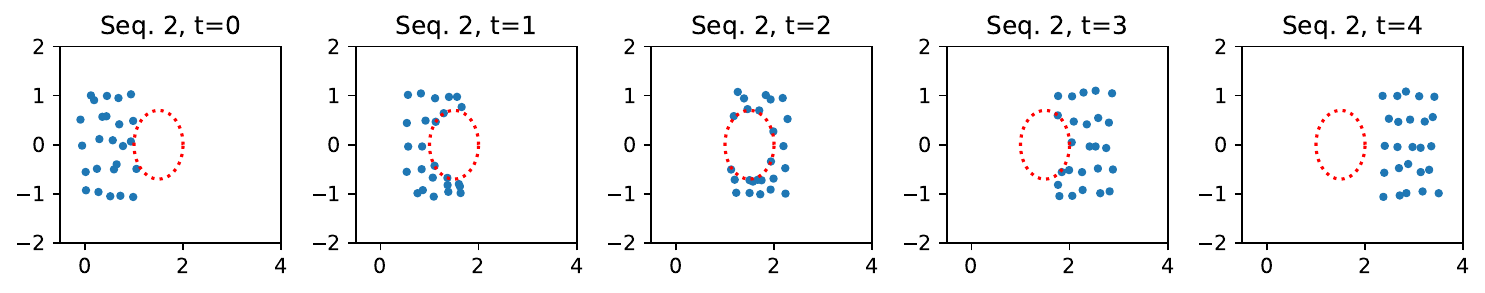}
	
	\caption{Examples of obstructed and unobstructed drone flights. Each row shows 5 snapshots of a $5 \times 5$ grid of drones flying through the square $[0,4] \times [-2,2]$. The first row shows unobstructed flight, while the second row shows the drones avoiding an obstacle at the ellipse $\left(\frac{x-1.5}{0.5}\right)^2 + \left(\frac{y}{0.7}\right)^2 = 1$.}
	\label{fig:drone_flight}
\end{figure}

\subsubsection{Supervised Classification}
We use the above dataset in a supervised classification experiment. We sample 15 instances of each flight pattern and reserve 5 of each as test set. Using the rest as training set, we obtain the $\nu_\text{opt}$ that minimizes $\operatorname{cost}_\lambda(\nu)$ with $\lambda = 10^{-1}$. 

We then classify each entry of the test set by its nearest neighbor under the parametrized GW distance with parameter space $(\Omega, \nu_\text{opt})$. For comparison, we build the following ensemble classifier. For every $\cX = (X, \mu_X, \Omega, \nu_\text{opt}, \omega_X)$ in the training set and every $\cY = (X, \mu_X, \Omega, \nu_\text{opt}, \omega_Y)$ in the test set, we can compute the standard GW distances $\GW_2(\omega_X^t, \omega_Y^t)$ for every $t \in \Omega$. This results in one set of distances between the training and test sets for every $t \in \Omega$, and thus, one nearest neighbor classifier for each $t \in \Omega$ that we call the \emph{$t$-th GW classifier}. The ensemble classifier then labels each entry of the test set with the most frequent label among the labels given by the GW classifiers. 

After repeating the above experiment 10 times, the parametrized GW distance classifier reaches an average of  $95\%$ classification accuracy with a $7\%$ standard deviation, while the ensemble classifier (based on standard GW) manages only $87\%$ accuracy with a $8\%$ standard deviation.

\subsubsection{Implementation}
\label{sec:classification_implementation}
Minimizing $\operatorname{cost}_\lambda(\nu)$ requires solving several optimizations that we describe now. Let $\nu \in \mathcal{P}(\Omega)$. For every $i = 1, \dots, n$, let $\cX_i = (X_i, \mu_i, \Omega, \nu, \omega_i)$ be a pm-net in $\mathfrak{N}_{\nu}$ with class label $y_i$ such that $1 \leq y_1 \leq \cdots \leq y_n \leq m$. We begin by finding couplings $\pi_{ij} \in \coup(\mu_i, \mu_j)$ that realize $\GW_{\mathsf{C}}(\cX_i, \cX_j)$ for $1 \leq i, j \leq n$ using gradient descent as detailed in \Cref{sec:optim_fixed}. Then we assemble the matrix $M_\nu$ as in \Cref{eq:classification_matrix} and minimize $\operatorname{cost}_\lambda(\nu)$ subject to $\nu \in \mathcal{P}(\Omega)$ using gradient descent. We alternate these optimizations until a convergence criterion is satisfied. Each gradient descent starts from the previous optimal coupling (resp. measure).

To complement \Cref{sec:implementation}, we record the gradient of $\operatorname{cost}_\lambda(\nu)$ with respect to $\nu$. To simplify the presentation below, we assume $\Omega = \{1, \dots, T\}$. The results below hold for arbitrary $1 \leq p, q < \infty$, but our code only implements the version for $p = q = 2$.

\begin{lemma}
	\label{lemma:gradient_classification_M_ij}
	Let $M \in \R^{n \times n}$ be a matrix with block structure
	\begin{equation*}
		M =
		\begin{pmatrix}
			B_{11} & \cdots  & B_{1m}\\
			\vdots & \ddots & \vdots \\
			B_{m1} & \cdots & B_{mm}
		\end{pmatrix}
	\end{equation*}
	where $B_{ij} \in \R^{n_i \times n_j}$ and $n_1 + \cdots + n_m = n$. Let $\displaystyle S(M) := \frac{\|B_{11}\|_1 + \cdots + \|B_{mm}\|_1}{\|M\|_1}$. Then:
	\begin{itemize}
		\item If $M_{ij}$ belongs to a block $B_{kk}$, $\displaystyle \frac{\partial S}{\partial M_{ij}} = \frac{1 - S(M)}{\|M\|_1}$.
		\item Otherwise, $\displaystyle \frac{\partial S}{\partial M_{ij}} = -\frac{S(M)}{\|M\|_1}$.
	\end{itemize}
\end{lemma}
\begin{proof}
	Suppose that the entry $M_{ij}$ belongs to the block $B_{kk}$ for some $1 \leq k \leq m$. Note that $\partial \|B_{kk}\|_1 /\partial M_{ij} = 1$ and $\partial \|B_{hh}\|_1 /\partial M_{ij} = 0$ for any $h \neq k$. Likewise, $\partial \|M\|_1 /\partial M_{ij} = 1$. Then
	\begin{align*}
		\frac{\partial S}{\partial M_{ij}}
		&= \left[ \frac{\partial \|B_{kk}\|_1}{\partial M_{ij}} \cdot \|M\|_1 - \left(\sum_{h=1}^m \|B_{hh}\|_1 \right) \cdot \frac{\partial \|M\|_1}{\partial M_{ij}} \right] / \|M\|_1^{2} \\
		&= \left[ 1 - \left(\sum_{h=1}^m \|B_{hh}\|_1 \right)/\|M\|_1 \right] \frac{1}{\|M\|_1} \\
		&= \frac{1 - S(M)}{\|M\|_1}.
	\end{align*}
	If $M_{ij}$ does not belong to any block $B_{kk}$, then $\partial \|B_{hh}\|_1 /\partial M_{ij} = 0$ for all $1 \leq h \leq m$ and $\partial \|M\|_1 /\partial M_{ij} = 1$. Hence
	\begin{align*}
		\frac{\partial S}{\partial M_{ij}}
		&= \left[ 0 \cdot \|M\|_1 - \left(\sum_{h=1}^m \|B_{hh}\|_1 \right) \cdot \frac{\partial \|M\|_1}{\partial M_{ij}} \right] / \|M\|_1^{2} \\
		&= - \frac{\sum_{h=1}^m \|B_{hh}\|_1}{\|M\|_1} \cdot \frac{1}{\|M\|_1} \\
		&= -\frac{S(M)}{\|M\|_1}.
	\end{align*}
\end{proof}

\begin{lemma}
	\label{lemma:gradient_GW_nu}
	Let $\nu \in \mathcal{P}(\Omega)$, and let $\cX, \cY \in \mathfrak{N}_{\nu}$. Let $\pi \in \coup(\mu_X, \mu_Y)$ be the coupling that realizes $\GW_{\mathsf{C}}(\cX_i, \cX_j)$. Then for any $1 \leq t \leq T$,
	\begin{equation*}
		\frac{\partial}{\partial \nu_t} \GW_{\mathsf{C}}(\cX_i, \cX_j)^p = \dis_p(\pi, \omega_X^t, \omega_Y^t)^p.
	\end{equation*}
\end{lemma}
\begin{proof}
	When $\Omega$ is finite, the equation in \Cref{ex:fixed_parameter_space} becomes $\displaystyle \GW_{\mathsf{C}}(\cX, \cY)^p = \sum_{s=1}^T \dis_p(\pi, \omega_X^s, \omega_Y^s)^p \cdot \nu_s$. The result is immediate from here.
\end{proof}

\begin{lemma}
	\label{lemma:gradient_classification_nu}
	Let $\nu \in \mathcal{P}(\Omega)$ and fix $1 \leq p < \infty$. For every $i = 1, \dots, n$, let $\cX_i = (X_i, \mu_i, \Omega, \nu, \omega_i)$ be a pm-net in $\mathfrak{N}_{\nu}$ with class label $y_i$ such that $1 \leq y_1 \leq \cdots \leq y_n \leq m$. Let $\pi_{ij} \in \coup(\mu_i, \mu_j)$ be the coupling that realizes $\GW_{\mathsf{C}}(\cX_i, \cX_j)$ and define $M \in \R^{n \times n}$ by $M_{ij} = \GW_{\mathsf{C}}(\cX_i, \cX_j)^p$. Then for any $1 \leq t \leq T$,
	\begin{equation*}
		\frac{\partial}{\partial \nu_t} \operatorname{cost}_p(\nu)
		= \sum_{\substack{M_{ij} \in B_{kk}\\\text{for some } k}} \frac{1 - S(M)}{\|M\|_1} \cdot \dis_p(\pi_{ij}, \omega_i^t, \omega_j^t)^p
		+ \sum_{\substack{M_{ij} \notin B_{kk}\\\text{for any } k}} -\frac{S(M)}{\|M\|_1} \cdot \dis_p(\pi_{ij}, \omega_i^t, \omega_j^t)^p
	\end{equation*}
\end{lemma}
\begin{proof}
Since we define $M_{ij} = \GW_{\mathsf{C}}(\cX_i, \cX_j)^p$, $\operatorname{cost}_p(\nu) = S(M)$, where $S$ is the function defined in \Cref{lemma:gradient_classification_M_ij}. Hence, using the chain rule and  \Cref{lemma:gradient_classification_M_ij} and \Cref{lemma:gradient_GW_nu} yields
	\begin{align*}
		\frac{\partial}{\partial \nu_t} \operatorname{cost}_p(\nu)
		&= \sum_{i,j=1}^{n} \frac{\partial S}{\partial M_{ij}} \cdot \frac{\partial M_{ij}}{\partial \nu_t}
		= \sum_{i,j=1}^{n} \frac{\partial S}{\partial M_{ij}} \cdot \frac{\partial}{\partial \nu_t} \GW_{\mathsf{C}}(\cX_i, \cX_j)^p \\
		&= \sum_{\substack{M_{ij} \in B_{kk}\\\text{for some } k}} \frac{1 - S(M)}{\|M\|_1} \cdot \dis_p(\pi_{ij}, \omega_i^t, \omega_j^t)^p
		+ \sum_{\substack{M_{ij} \notin B_{kk}\\\text{for any } k}} -\frac{S(M)}{\|M\|_1} \cdot \dis_p(\pi_{ij}, \omega_i^t, \omega_j^t)^p
	\end{align*}
\end{proof}

\subsection{Discussions, Limitations, and Future Work}

In this section, we demonstrated four applications of the parameterized GW distance. The utility of the parameterized GW framework in each case is summarized as follows.
\begin{enumerate}[leftmargin=*]
\item \narrowbf{Block matrices}: It defines distances by decomposing each block matrix into submatrices and combining information from all parts (\cref{sec:pandas}).
\item \narrowbf{Random graphs}: It serves as a meaningful invariant for comparing and clustering random graph samples, tested on Erd\H{o}s--R\'enyi and stochastic block models (\cref{sec:random-graphs}).
\item \narrowbf{Heat kernels}: By aligning sets of heat kernels, it captures features across multiple timescales (e.g., cycles of different lengths or nesting structures), with $\nu$ tuned for effective couplings (\cref{sec:nested-cycles}).
\item \narrowbf{Feature selection}:  Allowing $\nu$ to vary enables identification of discriminative time intervals in dynamic data, illustrated by classifying obstructed vs. unobstructed drone flights (\cref{sec:feature-selection}).
\end{enumerate}
In the block matrix experiments, the computation of $\GW_{\mathsf{C}}$ ignores interactions between blocks $X_i$ and $Y_j$ for $i \neq j$, and consequently does not yield the expected optimal alignment. Further work is needed to determine the \emph{minimal} interaction information required to approach the global optimum. In the heat kernel and feature selection experiments, manual tuning of the parameter $\nu$ is necessary both to capture features across scales and to identify discriminative time intervals. 

We remark that the focus of this paper is on fundamental theory and establishing core methods for its application. In particular, the experiments above are qualitative in nature and only deal with synthentic data. An important direction for future research is the development of more efficient learning frameworks that support automatic parameter tuning, so that these methods are more applicable to real-world data. Applications to time-varying metric space and network data, e.g., in the form of longitudinal fMRI data, will be the goal of a followup project.


\section*{Acknowledgments}
This work was partially supported by grants from National Science Foundation (NSF) projects IIS-2145499, DMS-2324962, and CIF-2526630, and Department of Energy (DOE) projects DE-SC0023157 and DE-SC0021015. 
We thank C\'{e}dric Vincent-Cuaz for clarifying the Python Optimal Transport library and for identifying errors in our initial implementation.

\bibliographystyle{plain}
\bibliography{refs-reeb}


\end{document}